\documentclass[aoas,preprint]{imsart}

\RequirePackage[OT1]{fontenc}
\RequirePackage{amsthm,amsmath}
\RequirePackage[square,authoryear]{natbib}
\RequirePackage[colorlinks,citecolor=blue,urlcolor=blue]{hyperref}
\RequirePackage{hypernat}

\usepackage{url}

\usepackage[T1]{fontenc}
\usepackage[latin1]{inputenc}
\usepackage{amssymb}
\usepackage{amsfonts}
\usepackage{graphicx}
\usepackage[algo2e,boxed]{algorithm2e}
\usepackage{xspace}
\usepackage{dsfont}
\usepackage{graphicx}  
\usepackage{verbatim}
\usepackage{appendix}
\usepackage{pstricks}
\usepackage{ifthen}
\usepackage{framed}
\usepackage{qtree}
\usepackage{epic}
\usepackage{stmaryrd}
\usepackage{subfigure}
\usepackage{algorithmic}
\usepackage{algorithm}
\usepackage{wrapfig}

\newtheorem{theorem}{Theorem}
\newtheorem{corollary}[theorem]{Corollary}
\newtheorem{lemma}[theorem]{Lemma}
\newtheorem{proposition}[theorem]{Proposition}

\newtheorem{example}[theorem]{Example}
\theoremstyle{definition}
\newtheorem{definition}[theorem]{Definition}
\theoremstyle{remark}

\newcommand{\denselist}{
\itemsep -2pt\topsep-8pt\partopsep-8pt
}

\newcommand{\blockcomment}[1]{ }

\def\reals{\mathbb{R}}

\newbox\subfigbox


\newcommand{\Xedge}{\Omega^{cross}_{A',B'}}
\newcommand{\ABedge}{\Omega^{cross}_{A,B}}
\newcommand{\BAedge}{\Omega^{cross}_{B,A}}
\newcommand{\Aedge}{\Omega^{int}_{A}}
\newcommand{\Bedge}{\Omega^{int}_{B}}
\newcommand{\FF}{\tilde{\mathcal{F}}}
\newcommand{\FFF}{\hat{\mathcal{F}}}

\arxiv{math.PR/0000000}

\startlocaldefs
\numberwithin{equation}{section}
\theoremstyle{plain}

\endlocaldefs

\begin{document}

\begin{frontmatter}
\title{Uncovering the Riffled Independence Structure of Rankings}
\runtitle{Uncovering the Riffled Independence Structure of Rankings}

\begin{aug}
\author{\fnms{Jonathan} \snm{Huang}\ead[label=e1]{jch1@cs.cmu.edu}}
\and
\author{\fnms{Carlos} \snm{Guestrin}\ead[label=e2]{guestrin@cs.cmu.edu}}

\runauthor{J. Huang et al.}

\affiliation{Carnegie Mellon University}

\address{200 Smith Hall \\
Carnegie Mellon University, \\
5000 Forbes Avenue,\\
Pittsburgh Pennsylvania, 15213. \\
\printead{e1}\\
\phantom{E-mail:\ }\printead*{e2}}
\end{aug}

\begin{abstract}
\looseness -1 Representing distributions over permutations
can be a daunting task due to the fact that
the number of permutations of $n$ objects
scales factorially in $n$. One recent way that 
has been used to reduce storage complexity
has been to exploit probabilistic independence, 
but as we argue, full independence assumptions
impose strong sparsity constraints on 
distributions and are unsuitable for modeling 
rankings. We identify a novel class of 
independence structures, called \emph{riffled 
independence}, encompassing a more expressive 
family of distributions while retaining many of 
the properties necessary for performing 
efficient inference and reducing sample 
complexity. In riffled independence, one draws 
two permutations independently, then performs 
the \emph{riffle shuffle}, common in card games,
to combine the two permutations to form a single 
permutation. Within the context of ranking, 
riffled independence corresponds to ranking 
disjoint sets of objects independently, then 
interleaving those rankings. In this paper, we 
provide a formal introduction to riffled independence
and present algorithms for using riffled independence
within Fourier-theoretic frameworks which have been 
explored by a number of recent papers. Additionally, 
we propose an automated method for discovering sets 
of items which are riffle independent from a training 
set of rankings. We show that our clustering-like 
algorithms can be used to discover meaningful latent 
coalitions from real preference ranking datasets and
to learn the structure of hierarchically 
decomposable models based on riffled independence.
\end{abstract}



\end{frontmatter}

\section{Introduction}\label{sec:introduction}
\looseness -1 Ranked data appears ubiquitously in various statistics and machine learning
application domains.
Rankings are useful, for example, in reasoning about
preference lists in surveys~\citep{toshihiro03}, search results
in information retrieval applications~\citep{sunetal10}, 
and ballots in certain elections~\citep{diaconis89}
and even the ordering of topics and paragraphs within 
a document~\citep{chenetal09}.
The problem of building statistical models on rankings has
thus been an important research topic in the learning community.
As with many challenging learning problems, one must contend
with an intractably large state space when dealing with rankings
since there are $n!$ ways to rank $n$ objects.
In building a statistical model over rankings, simple (yet flexible)
models are therefore preferable because they are typically
more computationally tractable and less prone to overfitting.

\looseness -1 A popular and highly successful approach for
achieving such simplicity for distributions
involving large collections of interdependent variables
has been to exploit conditional independence structures
(e.g., naive Bayes, tree, Markov models).
With ranking problems, however, independence-based
relations are typically trickier to exploit due to the so-called
\emph{mutual exclusivity} constraints which constrain any
two items to map to different ranks in any given ranking.

In this paper, we present a novel, relaxed notion of independence, called
\emph{riffled independence}, in which one ranks disjoint subsets of
items independently, then interleaves the subset rankings
to form a joint ranking of the item set.  
For example, if one ranks a set of food items containing fruits and
vegetables by preference, then one might first rank the vegetable and fruit sets
separately, then interleave the two rankings to form a ranking for the full item set.
Riffled independence appears naturally in many ranked datasets --- as we show, political
coalitions in elections which use the STV (single transferable vote)
voting mechanism typically lead to pronounced riffled independence
constraints in the vote histograms.

Chaining the interleaving operations recursively leads to a 
simple, interpretable class of models over rankings, not unlike
graphical models.  We present methods for learning the parameters
of such models and for estimating their structure.

The following is an outline of our main contributions 
as well as a roadmap for the sections ahead.\footnote{
This paper is an extended presentation of our previous papers
~\citep{huangetal09c}, which was the first introduction
of riffled independence, and ~\citep{huangetal10a}, which studied
hierarchical models based on riffle independent decompositions.
}
\begin{itemize}\denselist
\item Section~\ref{sec:distributions} gives a broad overview of several
approaches for modeling probability distributions over permutations.
In particular, we summarize the results of~\cite{huangetal09a}, 
which studied probabilistic independence relations
in distributions on permutations.
\item In Section~\ref{sec:riffledindependence}, we introduce our main
contribution: an intuitive, novel generalization of the notion of
independence for permutations,
\emph{riffled independence}, based on interleaving independent rankings
of subsets of items. We show riffled independence to be a more
appropriate notion of independence for ranked data
and exhibit evidence that riffle independence relations can 
approximately hold in real ranked datasets.  
We also discuss ideas for exploiting riffled
independence relations in a distribution to reduce sample
complexity and to perform efficient inference.
\item Within the same section, we introduce a novel family of distributions
over the set of interleavings of two item sets,
called \emph{biased riffle shuffles}, that are useful in the context of
riffled independence. We propose an efficient recursive 
procedure for computing the
Fourier transform of these biased riffle shuffle distributions, 
\item 
In Section~\ref{sec:algorithms}, we discuss the problem of estimating
model parameters of a riffle independent model from ranking data,
and computing various statistics from model parameters.
To perform such computations in a scalable way, we
develop algorithms that can be used in the Fourier-theoretic
framework of~\cite{kondor07},~\cite{huangetal09a}, and \cite{huangetal09b}
for joining riffle independent factors (\emph{RiffleJoin}),
and for teasing apart the riffle independent factors
from a joint (\emph{RiffleSplit}), and provide theoretical
and empirical evidence that our algorithms perform well.
\item 
We use Section~\ref{sec:hierarchical}
to define a family of simple and interpretable, yet flexible
distributions over rankings,
called hierarchical riffle independent models,
in which subsets of items are iteratively interleaved into larger 
and larger subsets
in a recursive stagewise fashion.
\item  
Sections~\ref{sec:structure1},~\ref{sec:structure2}, and~\ref{sec:structure3}
tackle the problem of structure learning for our 
riffle independent models.
In Section~\ref{sec:structure1}, we propose a method for finding the
partitioning of the item set such that the subsets
of the partition are as close to riffle independent as possible.
In particular, we propose a novel objective for quantifying the
degree to which two subsets are riffle independent to each
other.
In Section~\ref{sec:structure2} and ~\ref{sec:structure3}
we apply our partitioning algorithm to perform model selection
from training data in polynomial time, without
having to exhaustively search over
the exponentially large space of hierarchical structures.
\item 
Finally in Section~\ref{sec:experiments},
we apply our algorithms to a number of datasets both
simulated and real in order to validate our methods and assumptions.
We show that our methods are indeed effective, and apply them
in particular to various voting and preference ranking datasets.
\end{itemize}

\section{Distributions on rankings}\label{sec:distributions}
In this paper, we will be concerned with distributions
over rankings.  A \emph{ranking}
$\sigma=(\sigma(1),\dots,\sigma(n))$ is a one-to-one association
between $n$ items and ranks, where $\sigma(j)=i$ means that the $j^{th}$
item is assigned rank $i$ under $\sigma$.
By convention, we will think of low ranked items as being \emph{preferred}
over higher ranked items (thus, ranking an item in first place
means that it is the most preferred out of all items).
We will also refer to a ranking $\sigma$ by its
inverse, $\llbracket\sigma^{-1}(1),\dots,\sigma^{-1}(n)\rrbracket$
(called an \emph{ordering} and denoted with double brackets instead
of parentheses),
where $\sigma^{-1}(i)=j$ also means that the
$j^{th}$ item is assigned rank $i$ under $\sigma$.
The reason for using both notations is due to the fact that
certain concepts will be more intuitive to express using either
the ranking or ordering notation.

\newcommand{\aC}{{\bf{C}}}
\newcommand{\aP}{{\bf{P}}}
\newcommand{\aL}{{\bf{L}}}
\newcommand{\aO}{{\bf{O}}}
\newcommand{\aF}{{\bf{F}}}
\newcommand{\aG}{{\bf{G}}}
\begin{example}
As a running example in this paper, we will consider ranking a small
list of 6 items consisting of fruits and vegetables enumerated below:
\vspace{-5mm}

{\footnotesize
\begin{align*}
1.&\; Corn \;(\aC) &
2.&\; Peas \;(\aP) &
3.&\; Lemon \;(\aL) \\ 
4.&\; Orange \;(\aO) & 
5.&\; Fig \;(\aF) & 
6.&\; Grapes \;(\aG) 
\end{align*}\vspace{-5mm}
}

\noindent The ranking $\sigma=(3,1,5,6,2,4)$ means, for example, that
Corn is ranked third, Peas is ranked first, Lemon is ranked fifth, and so on.
In ordering notation, the same ranking is expressed as:
$\sigma=\llbracket\aP,\aF,\aC,\aG,\aL,\aO\rrbracket$.
Finally we will use 
$\sigma(3)=\sigma(\aL)=5$ to denote the rank of the third
item, Lemon.
\end{example}

\paragraph{Permutations and the symmetric group}
Rankings are similar to \emph{permutations}, which are 1-1 mappings from the
set $\{1,\dots,n\}$ into itself, the subtle difference being that
rankings map between two \emph{different} sets of size $n$.
In this paper, we will use the same notation for permutations and
rankings, but use permutations to refer to (1-1) functions which rearrange
the ordering of the item set or the ranks.
If $\tau$ is a permutation of the set of ranks, then
then given a ranking $\sigma$, one can rearrange the ranks by left-composing
with $\tau$.  Thus, the ranking $\tau\sigma$
maps item $i$ to rank $\tau(\sigma(i))$.  On the other hand, if $\tau$ is
a permutation of the item set, one can 
rearrange the item set by right-composing with $\tau^{-1}$.
Thus, if item $j$ was relabeled as item $i=\tau(j)$,
then $\sigma(\tau^{-1}(i))$ returns the rank of item $j$ with respect
to the original item ordering. 
Finally, we note that the composition of any two permutations is
itself a permutation, and the collection of all $n!$ permutations
forms a group, commonly known as the \emph{symmetric group}, 
or $S_n$.\footnote{We will sometimes abusively denote the set of rankings 
by $S_n$. Strictly speaking, however, the rankings are not a group, 
and instead one says that $S_n$ acts faithfully on rankings.}

A distribution $h(\sigma)$, defined over the set of rankings or permutations
can be viewed as a joint distribution over the $n$ variables
$(\sigma(1),\dots,\sigma(n))$ (where $\sigma(j)\in\{1,\dots,n\}$),
subject to \emph{mutual exclusivity constraints}
which stipulate that two objects cannot simultaneously map to the same rank,
or alternatively, that two ranks cannot simultaneously be occupied by
the same object
($h(\sigma(i)=\sigma(j))=0$ whenever $i\neq j$).

\begin{example}[APA election data]\label{ex:apadata}
As a running example throughout the paper (in addition to 
the fruits and vegetables), we will analyze the well known 
APA election dataset that was first used by~\cite{diaconis88} and 
has since been
analyzed in a number of ranking studies.  The APA dataset is a 
collection of 5738 ballots from a 1980 presidential election 
of the American Psychological Association
where members rank ordered five candidates from favorite to least favorite.
The names of the five candidates that year were 
(1) William Bevan, (2) Ira Iscoe, (3) Charles Kiesler, (4) Max Siegle,
and (5) Logan Wright~\citep{marden95}.

Since there are five candidates, there are $5!=120$ possible rankings.
In Figure~\ref{fig:apavotedistribution} we plot the proportion of votes
that each ranking received.  
Interestingly, instead of concentrating at just a small set of rankings, the vote
distribution in the APA dataset is fairly diffuse with every ranking receiving
some number of votes.
The mode of the vote distribution occurs
at the ranking $\sigma=(2,3,1,5,4)=\llbracket\mbox{C. Kiesler},\mbox{W. Bevan},
\mbox{I. Iscoe},\mbox{L. Wright},\mbox{M. Siegle}\rrbracket$ with 186 votes.  

For interpretability, we also visualize the \emph{matrix of first-order marginals} 
in which the $(i,j)$ entry represents the number of voters who assigned rank
$i$ to candidate $j$.
Figure~\ref{fig:apafirstorder} represents the first-order matrix using grayscale
levels to represent numbers of voters.  What can be seen is that overall, candidate
3 (C. Kiesler) received the highest number of votes for rank 1 (and incidentally,
won the election).  The vote distribution gives us a story that goes far
deeper than simply telling us who the winner was, however.  
\cite{diaconis88}, for example, noticed 
that candidate 3 also had a significant ``hate'' 
vote --- a good number of voters placed him in the last rank.
Throughout this paper, we will let this story unfold 
via a series of examples based on the APA dataset.

\begin{figure*}[t!]
\begin{center}
\subfigure[]{
\includegraphics[width=.55\textwidth]{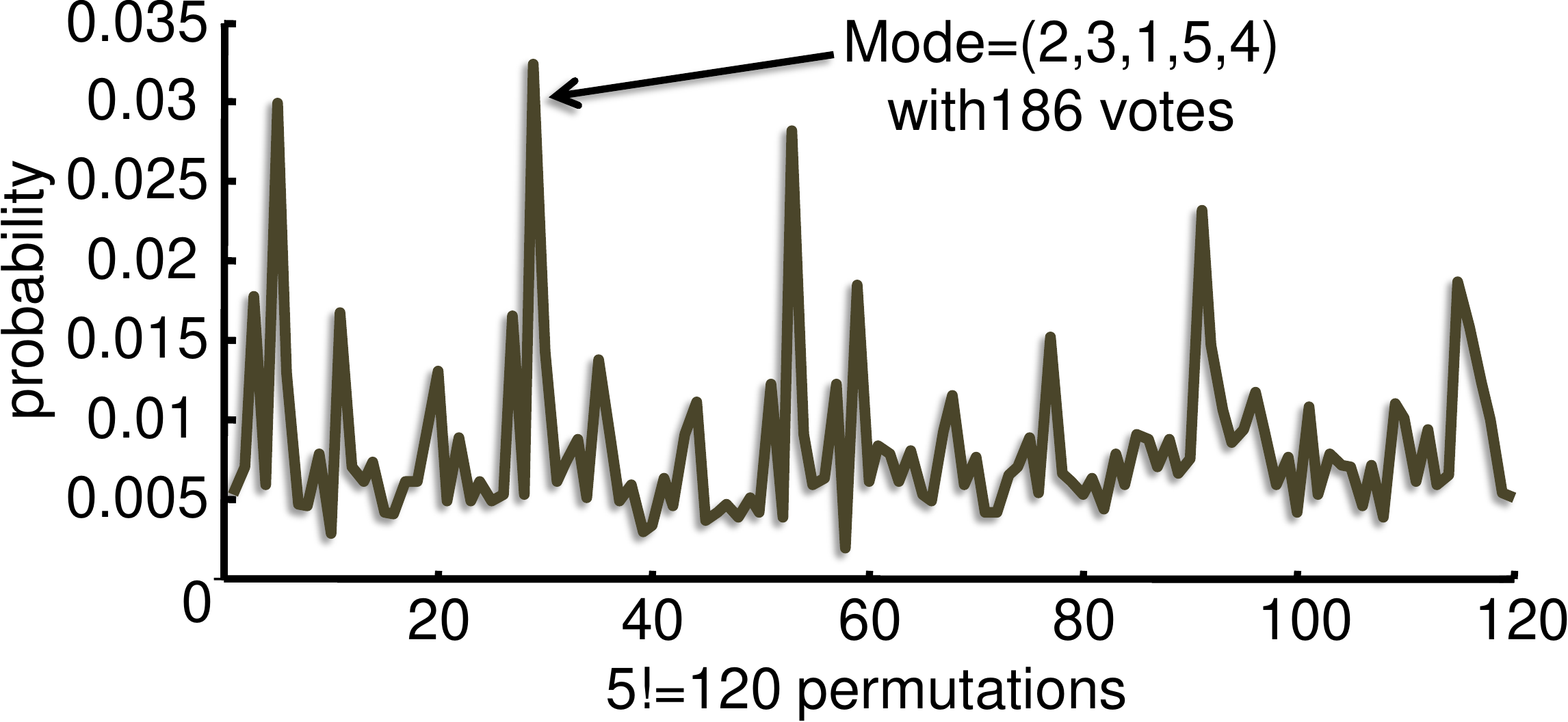}
   \label{fig:apavotedistribution}
}
\subfigure[]{
\raisebox{3pt}{
\includegraphics[width=.35\textwidth]{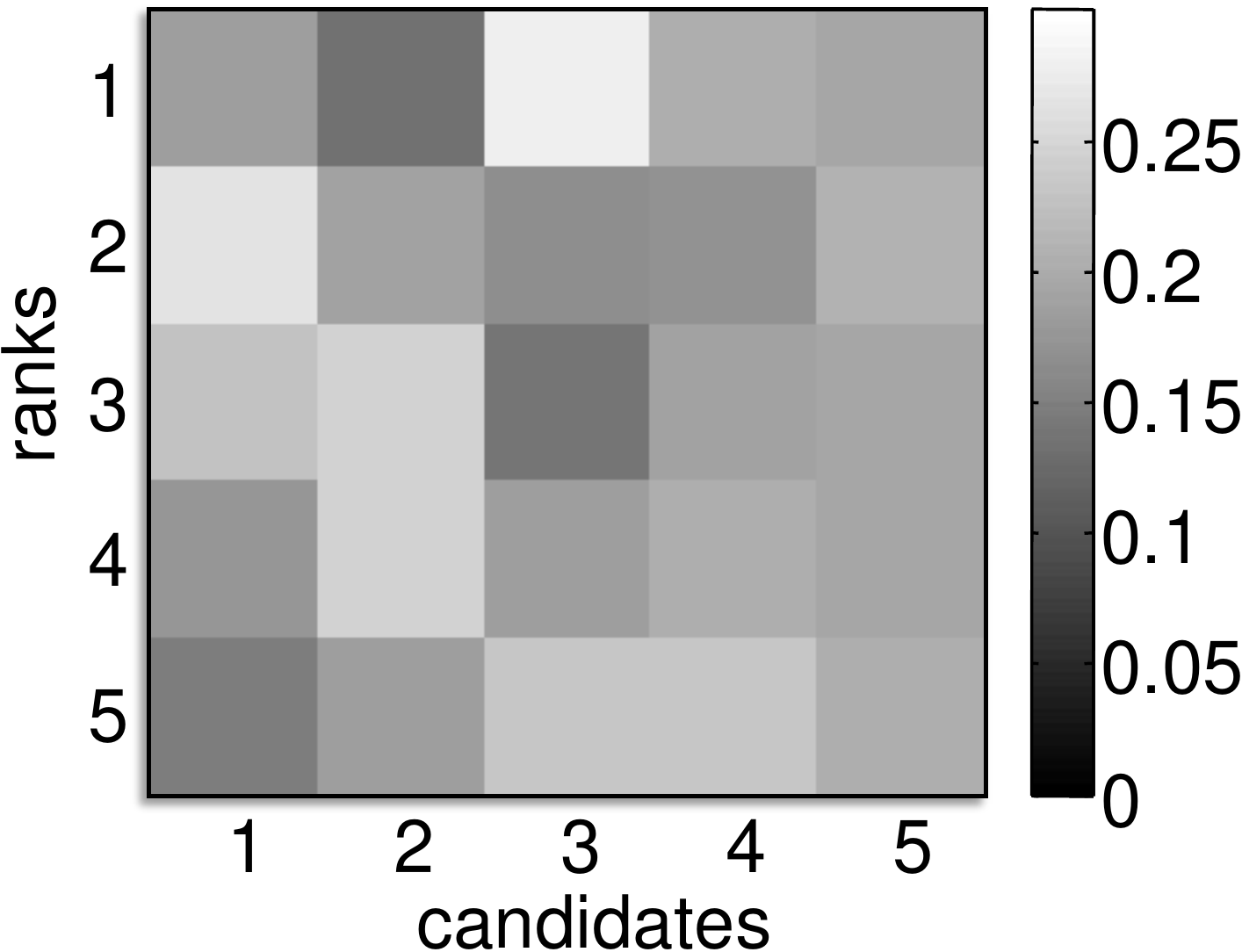}
   \label{fig:apafirstorder}
}}
\caption{APA (American Psyochological Association) election data. 
\subref{fig:apavotedistribution} vote distribution: percentage of votes for
each of $5!=120$ possible rankings --- the mode of the distribution
is $\sigma=(2,3,1,5,4)$.  
\subref{fig:apafirstorder} Matrix of first order marginals: the $(i,j)^{th}$
entry reflects the number of voters who ranked candidate $j$ in the $i^{th}$ rank.
}
\label{fig:apa}
\end{center}
\end{figure*}
\end{example}

\subsection{Dealing with factorial possibilities}
The fact that there are factorially many possible rankings
poses a number of significant challenges for learning and inference.  
First, there is no way to tractably represent arbitrary distributions over rankings
for large $n$.  Storing an array of $12!$ doubles, for example, requires
roughly 14 gigabytes of storage, which is beyond the 
RAM capacity of a typical modern PC. 
Second, the naive algorithmic complexity of common probabilistic operations
is also intractable for such distributions.  Computing the marginal probability,
$h(\sigma(i)<\sigma(j))$, that item $i$ is preferred to item $j$, for example,
requires a summation over $O((n-2))!)$ elements.
Finally, even if storage and computation issues were resolved, one would still have 
sample complexity issues to contend with --- for nontrivial $n$, it is 
impractical to hope that each of the $n!$ possible rankings would appear
even once in a training set of rankings.  The only existing 
datasets in which every possible ranking is realized are those for which
$n\leq 5$, and in fact, the APA dataset (Example~\ref{ex:apadata}) is the only
such dataset for $n=5$ that we are aware of.

The quest for exploitable problem structure has led
researchers in machine learning and related fields
to consider a number of possibilities including
distribution sparsity~\citep{reid79,jagabathula08,farias09},
exponential family parameterizations~\citep{meila07,helmbold07,
lebanon08,patterson09}, algebraic/Fourier structure~\citep{kondor07,
kondor08,huangetal07,huangetal09b},
and probabilistic independence~\citep{huangetal09a}.  We briefly summarize
several of these approaches in the following.

\paragraph{Parametric models}
We will not be able to do justice to the sheer volume of previous work
on parametric ranking models. Parametric probabilistic models over 
the space of rankings have a rich tradition in statistics, 
\citep{thurstone27,mallows57,plackett75,marden95,fligner86,fligner88,meila07,guiver09}, 
and to this day, researchers continue to expand upon this body of work.
For example, the well known \emph{Mallows model} (which we will discuss in 
more detail in Section~\ref{sec:hierarchical}), which is
often thought of as an analogy of the normal distribution for permutations,
parameterizes a distribution with a ``mean'' permutation and a 
precision/spread parameter.

The models proposed in this paper generalize some of the classical models
from the statistical ranking literature, allowing for more expressive 
distributions to be captured. At the same time, our methods 
form a conceptual bridge to popular models (i.e., graphical models) 
used in machine learning
which, rather than relying on a prespecified parametric form, 
simply work within a family of distributions that are consistent with 
some set of conditional independence assumptions~\citep{koller09}.

\paragraph{Sparse methods}
Sparse methods for summarizing distributions range from older ad-hoc 
approaches such as maintaining $k$-best hypotheses~\citep{reid79} to the more
updated compressed sensing inspired 
approaches discussed in~\citep{jagabathula08,farias09}.
Such approaches assume that there are at most $k$ permutations which 
own all (or almost all) of the probability mass, where $k$ scales
either sublinearly or as a low degree polynomial in $n$.
While sparse distributions have been successfully applied
in certain tracking domains, we argue that they are often
less suitable in ranking problems where it might be necessary
to model indifference over a large subset of objects.\footnote{
In some situations, particularly when one is interested primarily
in accurately capturing a loss or payoff function instead of 
raw ranking probabilities, it can suffice to use a sparse proxy
distribution even if the true underlying distribution is not itself
sparse.  See, for example,~\citep{helmbold07,farias09} for details.} 
If one is approximately indifferent among 
a subset of $k$ objects, then there are at least $k!$
rankings with nonzero probability mass.  As an
example, one can see that 
the APA vote distribution (Figure~\ref{fig:apavotedistribution}) 
is clearly not a sparse distribution, with each ranking having
received some nonzero number of votes.

\paragraph{Fourier-based (low-order) methods}
Another recent thread of research has centered around
\emph{Fourier-based methods} which maintain a set of low-order summary
statistics~\citep{shin05,diaconis88,kondor08b,huangetal09b}.
The \emph{first-order summary}, for example, stores a marginal
probability of the form $h(\sigma\,:\,\sigma(j)=i)$
for every pair $(i,j)$ and thus requires
storing a matrix of only $O(n^2)$ numbers.  In our fruits/vegetables example,
 we might store
the probability that Figs are ranked first, or the probability that
Peas is ranked last.  
\begin{example}[APA election data (continued)]
In the following matrix, we record the first order
matrix computed from the histogram of votes in the APA election example
(also visualized using grayscale levels in Figure~\ref{fig:apafirstorder}).
Dividing each number by the total number of votes would yield a matrix
of first order marginal probabilities.\vspace{-3mm}

{\scriptsize
\[
\left[\sum_{\sigma:\sigma(j)=i} h(\sigma)\right]_{i,j} = 
\left[
\begin{array}{c|ccccc}
 & \mbox{W. Bevan} & \mbox{I. Iscoe} & \mbox{C. Kiesler} & \mbox{M. Siegle} & \mbox{L. Wright} \\
\hline
\mbox{Rank 1} & 1053 &  775 & 1609 & 1172 & 1129 \\
\mbox{Rank 2} & 1519 & 1077 &  960 &  972 & 1210 \\
\mbox{Rank 3} & 1313 & 1415 &  793 & 1089 & 1128 \\
\mbox{Rank 4} & 1002 & 1416 & 1050 & 1164 & 1106 \\
\mbox{Rank 5} &  851 & 1055 & 1326 & 1341 & 1165 
\end{array}\right].
\]\vspace{-5mm}
}
\end{example}
More generally, one might store
\emph{$s^{th}$-order marginals}, which are
marginal probabilities of $s$-tuples.
The second-order marginals, for example, take the form
$h(\sigma:\sigma(k,\ell)=(i,j))$, (perhaps encoding
the joint probability that Grapes are ranked first,
and Peas second) and require $O(n^4)$ storage.

Low-order marginals turn out to be intimately related to 
a generalized form of Fourier analysis.
Generalized Fourier transforms for functions on permutations have been
studied for several decades now primarily by Persi Diaconis and
his collaborators~\citep{diaconis88,clausen93,maslen98,terras99,rockmore00}.
Low-order marginals correspond, in a certain sense, to the 
low-frequency Fourier coefficients of a distribution over permutations.  
For example, the first-order matrix of $h(\sigma)$ can be 
reconstructed exactly from $O(n^2)$
of the lowest frequency Fourier coefficients of $h(\sigma)$, 
and the second-order matrix from $O(n^4)$ of the lowest frequency 
Fourier coefficients.  From a Fourier
theoretic perspective, one sees that low order marginals
are not just a reasonable way of summarizing a distribution, but can 
actually be viewed as a principled ``low frequency'' approximation 
thereof.
In contrast with sparse methods, 
Fourier-based methods handle diffuse distributions well
but are not easily scalable without making aggressive independence
assumptions~\citep{huangetal09a} since, in general,
one requires $O(n^{2s})$ coefficients
to exactly reconstruct $s^{th}$-order marginals, which
quickly becomes intractable for moderately large $n$.

\subsection{Fully independent subsets of items}\label{sec:fullyindep}
To scale to larger problems,~\cite{huangetal09a}
demonstrated that, by exploiting \emph{probabilistic independence},
one could dramatically improve the scalability of Fourier-based
methods, e.g., for tracking problems, since confusion in data association
only occurs over small independent subgroups of objects in many problems.
Probabilistic independence assumptions on the symmetric group
can simply be stated as follows.
Consider a distribution $h$ defined over $S_n$.
Let $A$ be a $p$-subset of $\{1,\dots,n\}$, say, $\{1,\dots,p\}$
and let $B$ be its complement ($\{p+1,\dots,n\}$) with size $q=n-p$.
We say that $\sigma(A)=(\sigma(1),\sigma(2),\dots,\sigma(p))$ and
$\sigma(B)=(\sigma(p+1),\dots,\sigma(n))$ are \emph{independent} if
\begin{equation}\label{eqn:fullyfactorized}
h(\sigma) = f(\sigma(1),\sigma(2),\dots,\sigma(p))\cdot 
        g(\sigma(p+1),\dots,\sigma(n)).
\end{equation}
Storing the parameters for the above distribution requires
keeping $O(p!+q!)$ probabilities instead of the much larger $O(n!)$
size required for general distributions.  Of course, $O(p!+q!)$
can still be quite large.  Typically, one decomposes the distribution
recursively and stores factors exactly for small enough factors,
or compresses factors using Fourier coefficients (but using higher
frequency terms than what would be possible without the independence
assumption).
In order to exploit probabilistic independence in the Fourier domain,
~\cite{huangetal09a} proposed algorithms for joining
factors and splitting distributions into independent components in
the Fourier domain.

\begin{figure*}[t!]
\begin{center}
\subfigure[]{
\includegraphics[width=.22\textwidth]{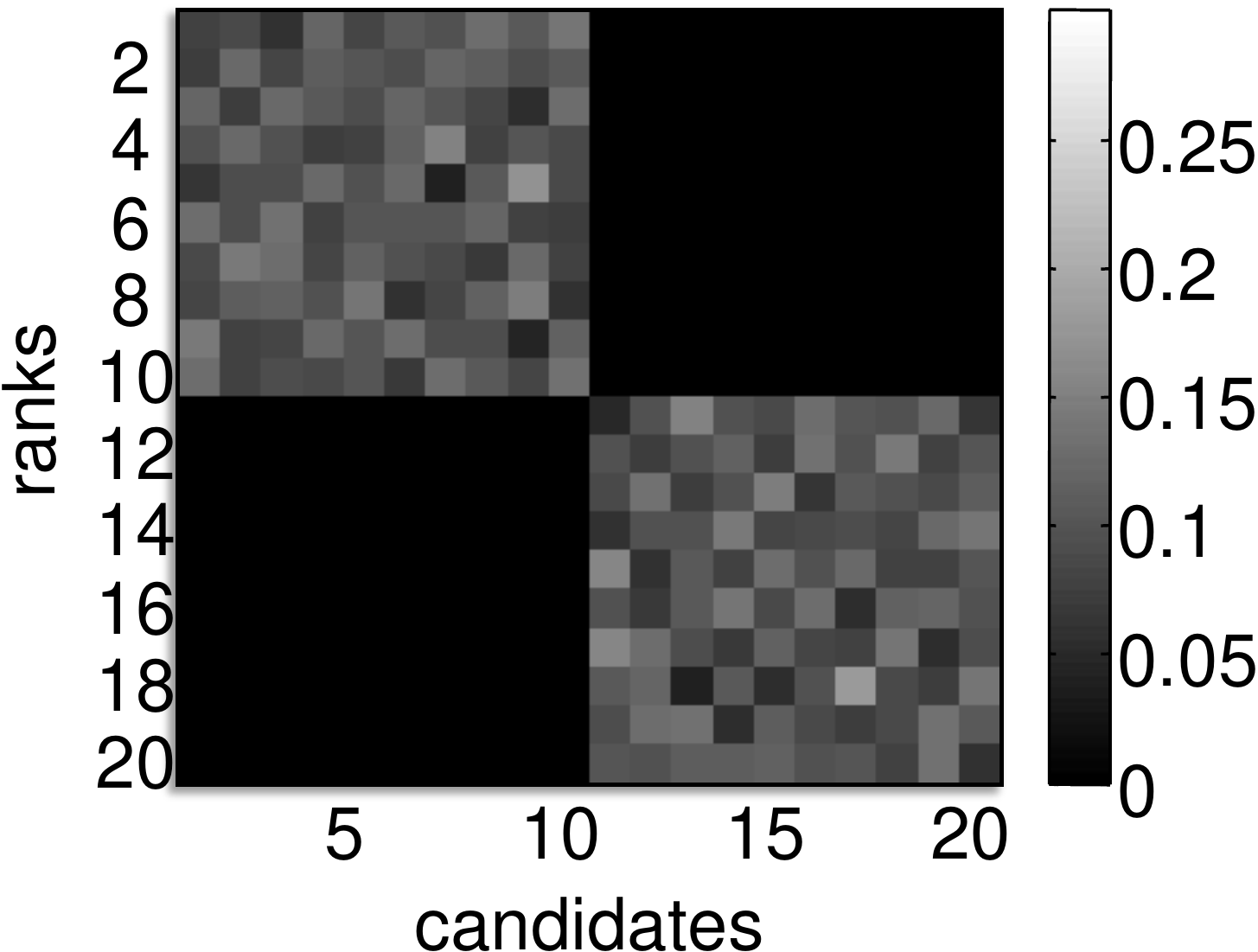}
   \label{fig:notindependent}
}
\subfigure[]{
\includegraphics[width=.22\textwidth]{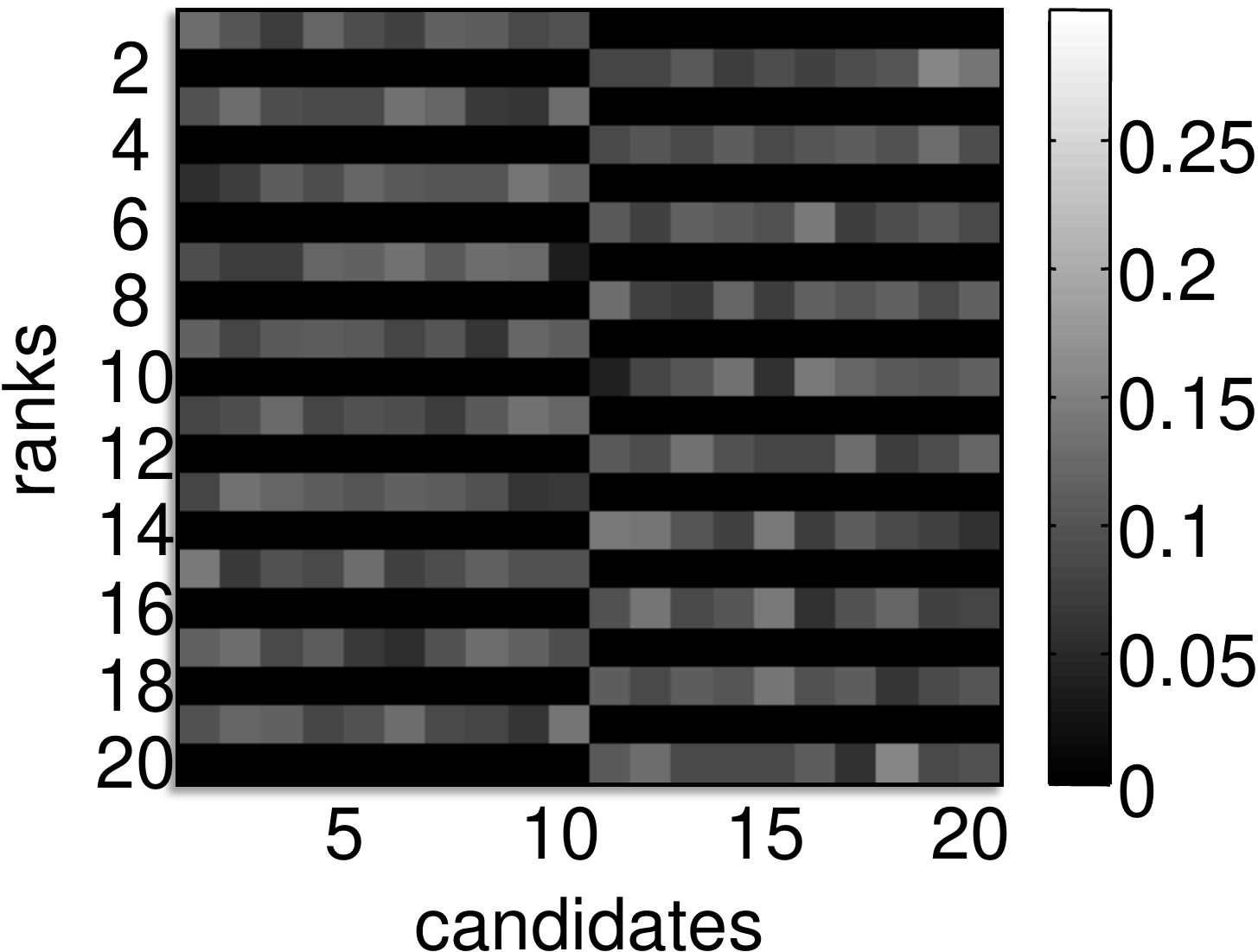}
   \label{fig:fullyindependent2}
}
\subfigure[]{
\includegraphics[width=.22\textwidth]{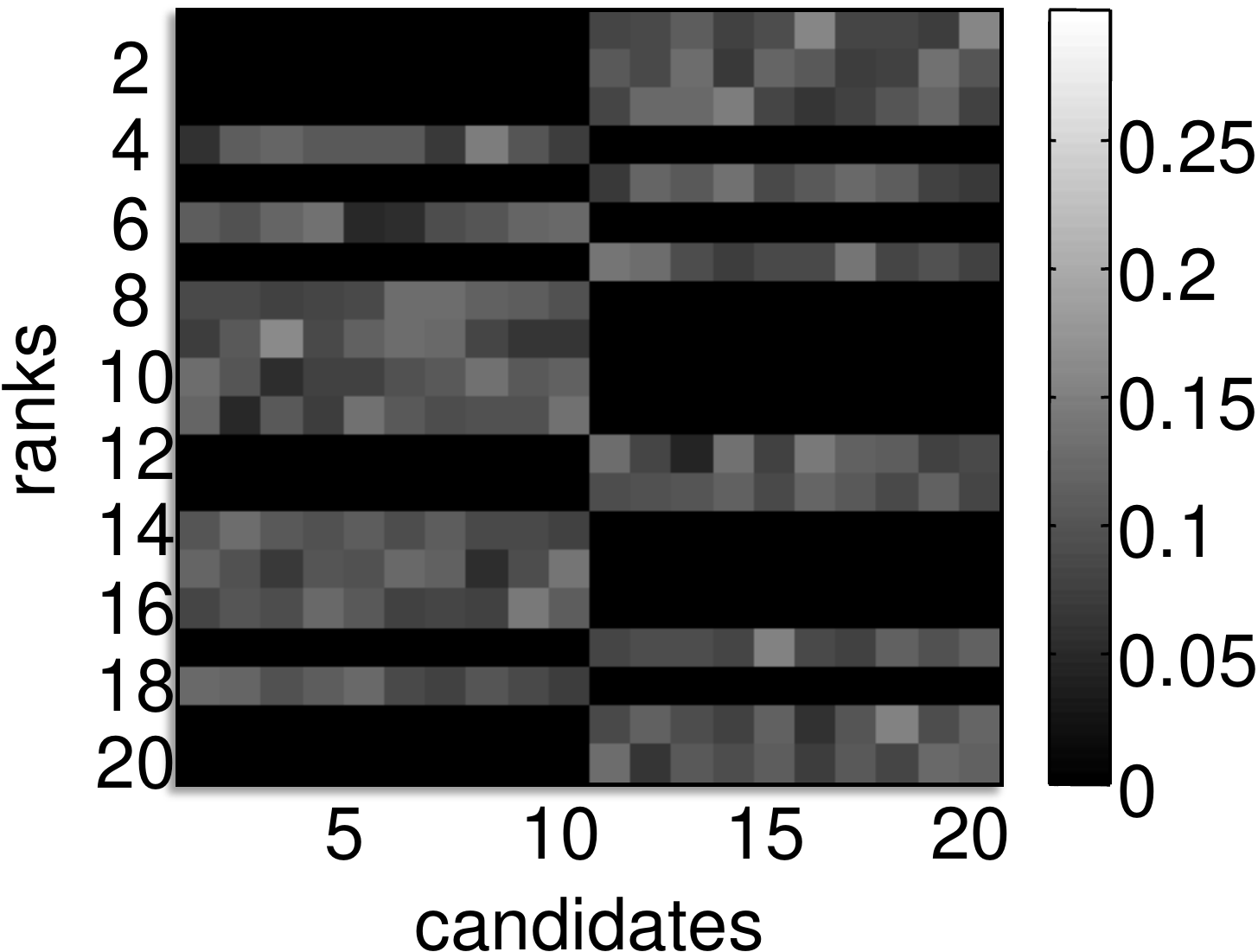}
   \label{fig:fullyindependent3}
}
\subfigure[]{
\includegraphics[width=.22\textwidth]{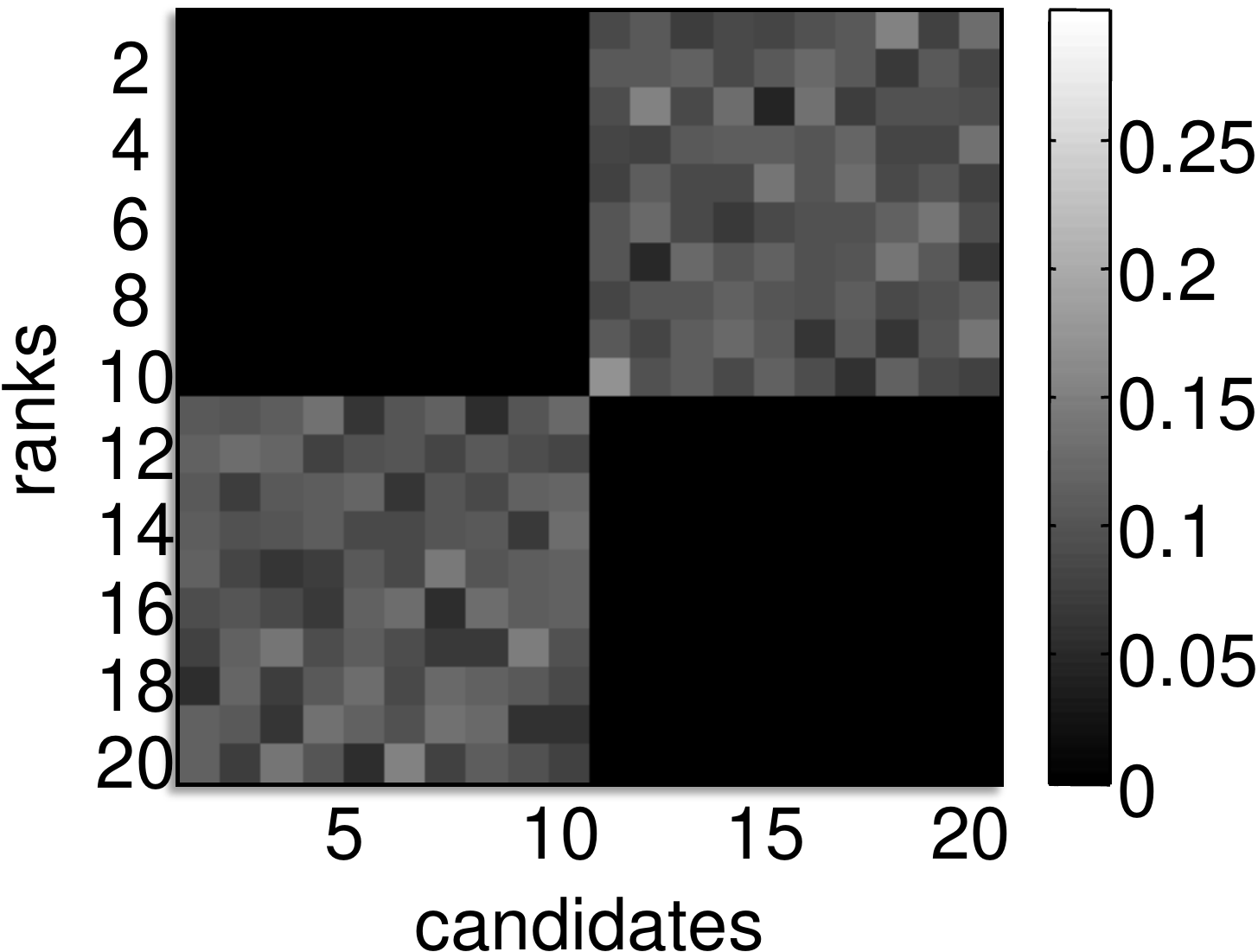}
   \label{fig:fullyindependent4}
}\qquad\qquad
\caption{Example first-order matrices with
$A=\{1,2,3\}$, $B=\{4,5,6\}$  fully independent, where
black means $h(\sigma:\sigma(j)=i)=0$.  In each case,
there is some $3$-subset $A'$ which $A$ is constrained to map to with
probability one.  Notice that, with respect to some rearranging of the rows,
independence imposes a block-diagonal structure on
first-order matrices.}
\label{fig:fullyindependent}
\end{center}
\end{figure*}
\begin{figure*}[t!]
\begin{center}
\subfigure[]{
\includegraphics[width=.50\textwidth]{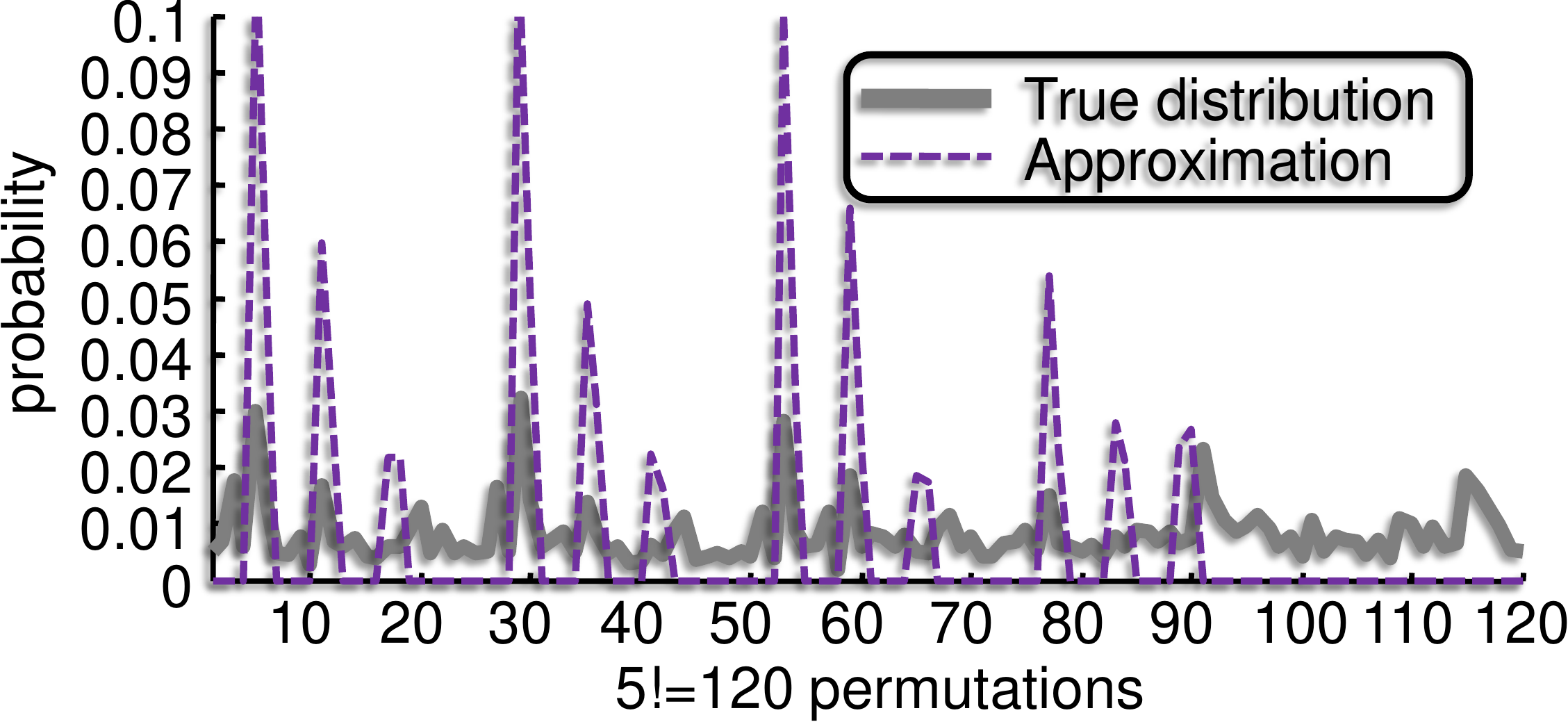}
   \label{fig:apafullyindependent_votedistribution}
}\qquad
\subfigure[]{
\raisebox{3pt}{
\includegraphics[width=.30\textwidth]{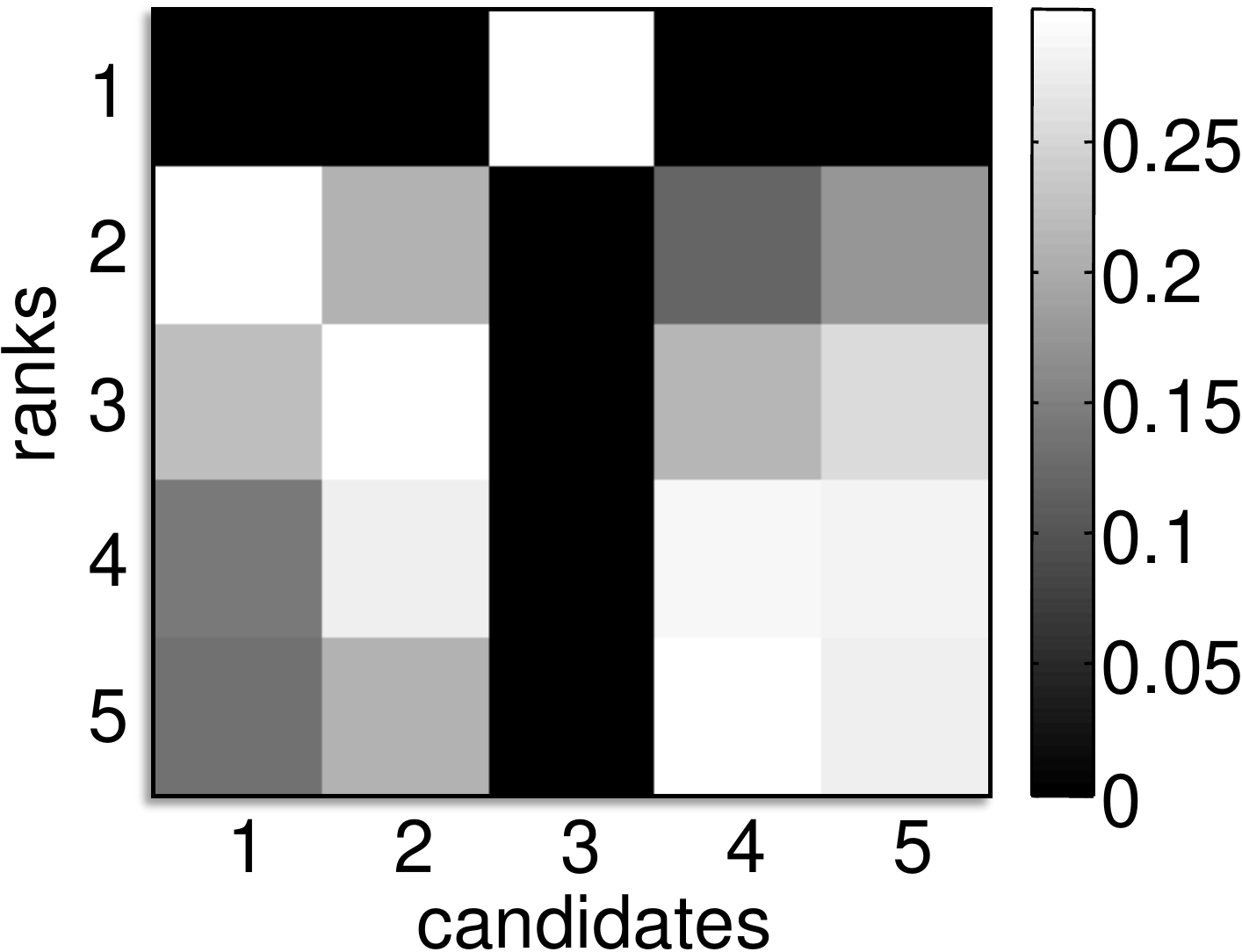}
   \label{fig:apafullyindependent_firstorder}
}
}
\caption{Approximating the APA vote distribution by a factored distribution
in which candidate 3 is independent of candidates $\{1,2,4,5\}$.
\subref{fig:apafullyindependent_votedistribution} in thick gray, the true distribution,
in dotted purple, the approximate distribution.  Notice that the factored
distribution assigns zero probability to most permutations.
\subref{fig:apafullyindependent_firstorder} matrix of first order marginals
of the approximating distribution.
}
\label{fig:apafullyindependentexamples}
\end{center}
\end{figure*}
Despite its utility for many tracking problems, however, we argue
that the independence assumption on permutations implies a rather
restrictive constraint on distributions, rendering independence
highly unrealistic in ranking applications.
In particular, using the mutual exclusivity property,
it can be shown~\citep{huangetal09a}
that, if $\sigma(A)$ and $\sigma(B)$ are independent, then
$A$ and $B$ are not allowed to map to the same ranks.  That is,
for some fixed $p$-subset $A'\subset \{1,\dots,n\}$,
$\sigma(A)$ is a permutation of elements in $A'$ and
$\sigma(B)$ is a permutation of its complement, $B'$, with
probability 1.
\begin{example}
Continuing with our vegetable/fruit example with $n=6$, if the vegetable
and fruit rankings,\vspace{-3mm}

{\footnotesize
\[
\sigma_A=[\sigma(\mbox{Corn}),\sigma(\mbox{Peas})],\; \mbox{and}\;
\sigma_B=[\sigma(\mbox{Lemons}),\sigma(\mbox{Oranges}),
	\sigma(\mbox{Figs}),\sigma(\mbox{Grapes})],
\]\vspace{-3mm}
}
 
\noindent are known to be independent.
Then for $A'=\{1,2\}$, the vegetables occupy the first and
second ranks with probability
one, and the fruits occupy ranks $B'=\{3,4,5,6\}$ with probability one,
reflecting that vegetables are always preferred over fruits according to this 
distribution.
\end{example}
\cite{huangetal09a} refer to this restrictive constraint
as the \emph{first-order condition} because of the block
structure imposed upon first-order
marginals (see Figure~\ref{fig:fullyindependent}).
In sports tracking, permutations represent the mapping
between the identities of players with positions on the field, and in such settings,
the first-order condition might say, quite reasonably,
that there is potential identity confusion within tracks for the red team and
within tracks for the blue team but no confusion between the two
teams.  
In our ranking example however, the first-order condition forces
the probability of any vegetable being in third place to be zero,
even though both vegetables will, in general, have nonzero marginal
probability of being in second place, which seems quite unrealistic.
\begin{example}[APA election data (continued)]
Consider approximating the APA vote distribution
by a factorized distribution (as in Equation~\ref{eqn:fullyfactorized}).  
In Figure~\ref{fig:apafullyindependentexamples},
we plot (in solid purple) the factored distribution
which is closest to the true distribution with respect to total variation
distance.  In our approximation, candidate 3 is constrained
to be independent of the remaining four candidates and maps to rank 1
with probability 1.  

While capturing the fact that the ``winner'' of the election
should be candidate 3, the fully factored distribution can be seen
to be a poor approximation, assigning zero probability to most permutations
even if all permutations received a positive number of votes.
Since the support of the true distribution is not contained within the 
support of the approximation, the KL divergence, 
$D_{KL}(h_{true};h_{approx})$ is infinite.  
\end{example}
In the next section, we overcome the restrictive first-order condition
with the more flexible notion of \emph{riffled independence}.

\section{\emph{Riffled independence}: definitions and examples}\label{sec:riffledindependence}
The \emph{riffle (or dovetail) shuffle}~\citep{diaconis92}
is perhaps the most commonly
used method of card shuffling, in which one cuts a deck of $n$ cards
into two piles, $A=\{1,\dots,p\}$ and $B=\{p+1,\dots,n\}$,
with size $p$ and $q=n-p$, respectively,
and successively drops the cards, one by one,
so that the two piles become interleaved (see Figure~\ref{fig:shuffleporn})
into a single deck again.
Inspired by the riffle shuffle, we present a novel relaxation of the full
independence assumption, which we call \emph{riffled independence}.
Rankings that are riffle independent are formed by independently selecting
rankings for two disjoint
subsets of objects, then interleaving the two
rankings using a riffle shuffle to form a final ranking over all objects.
Intuitively, riffled independence models complex
relationships within each set $A$ and $B$ while allowing
correlations between the sets to be modeled only through a
constrained form of shuffling.  
\begin{example}
Consider generating a ranking of vegetables and fruits.
We might first `cut the deck' into two piles,
a pile of vegetables ($A$) and a pile of fruits ($B$), 
and in a first stage, independently
decide how to rank each pile.
For example, within vegetables, we might decide that Peas 
are preferred to Corn:
$\llbracket \aP,\aC\rrbracket=\llbracket Peas,Corn\rrbracket$.
Similarly, within fruits, we might decide on the following 
ranking:
$\llbracket\aL,\aF,\aG,\aO\rrbracket = \llbracket Lemons,Figs,Grapes,Oranges\rrbracket$ (Lemons preferred over Figs, Figs preferred over Grapes, Grapes
preferred over Oranges).

In the second stage of our model, the fruit and vegetable rankings
are interleaved to form a full preference ranking over all six items.
For example, if the interleaving is given by: 
$\llbracket Veg,Fruit,Fruit,Fruit,Veg,Fruit\rrbracket$, then the resulting
full ranking is:
\[
\sigma = \llbracket Peas,Lemons,Figs,Grapes,Corn,Oranges\rrbracket.
\]
\end{example}
\begin{figure*}[t!]
\begin{center}
\subfigure[]{
\raisebox{6pt}{
\includegraphics[width=.20\textwidth]{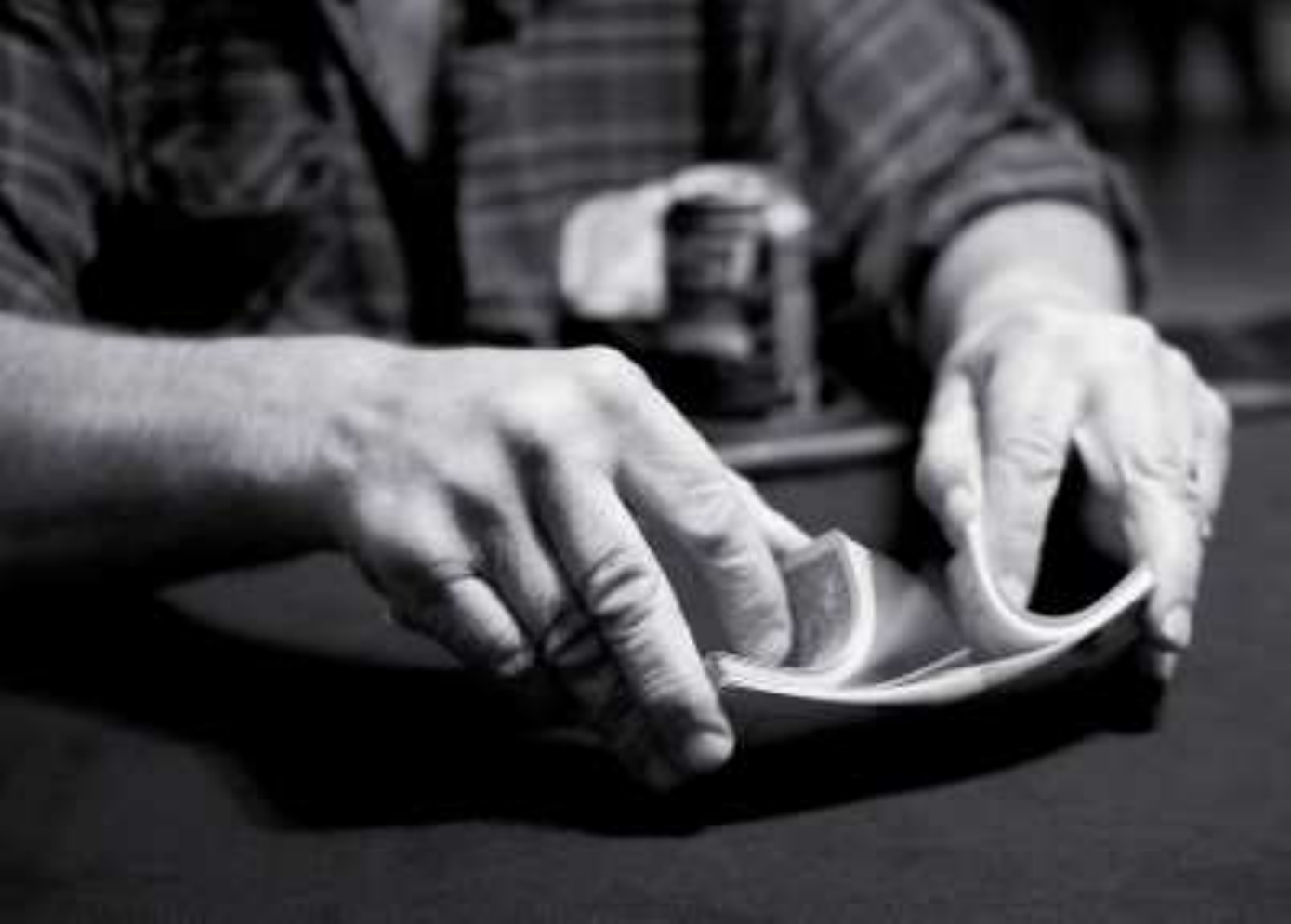}
   \label{fig:shuffleporn}
}
}\qquad\qquad
\subfigure[]{
\includegraphics[width=.375\textwidth]{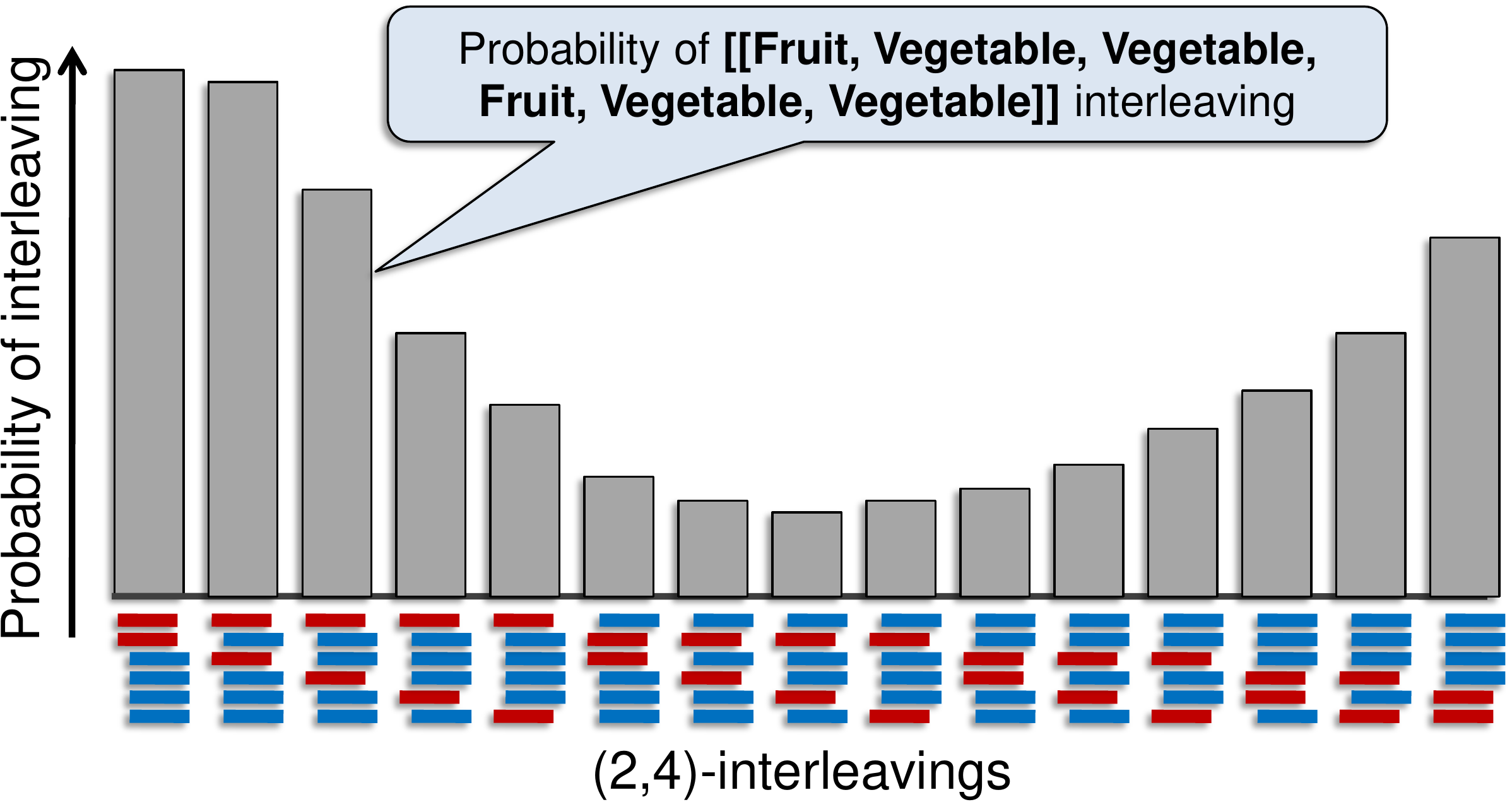}
   \label{fig:interleavings}
}
\caption{\subref{fig:shuffleporn} Photograph of the riffle shuffle executed
on  a standard deck of cards;
\subref{fig:interleavings} Pictorial example of a $(2,4)$-interleaving distribution, with red cards (offset to the left) denoting Vegetables, and blue cards
(offset to the right) denoting Fruits.
}
\label{fig:cardshufflingexamples}
\end{center}
\end{figure*}

\subsection{Convolution based definition of riffled independence}
There are two ways to define riffled independence, 
and, we will first provide a definition using convolutions, a view
inspired by our card shuffling intuitions.
Mathematically, shuffles are modeled as random walks on the symmetric group.  
The ranking $\sigma'$ \emph{after} a shuffle is generated from the
ranking \emph{prior} to that shuffle, $\sigma$, by drawing a
permutation, $\tau$ from an \emph{interleaving distribution} $m(\tau)$,
and setting $\sigma'=\tau\sigma$ (the composition of the mapping $\tau$ with
$\sigma$).
Given the distribution $h'$ over $\sigma$, we can find the
distribution $h(\sigma')$ \emph{after} the shuffle via the formula:
$h(\sigma')
= \sum_{\{\sigma,\tau\,:\,\sigma'=\tau\sigma\}} 
m(\tau)h'(\sigma)$.  This operation which combines 
the distributions $m$ and $h$ is commonly known as \emph{convolution}:
\begin{definition}
Let $m$ and $h'$ be probability distributions on $S_n$.
The \emph{convolution} of the distributions is the function:
$
[m*h'](\sigma) \equiv \sum_{\pi\in S_n} m(\pi)\cdot h(\pi^{-1}\sigma). 
$
We use the $*$ symbol to denote
the convolution operation. Note that $*$ is not 
in general commutative (hence, $m*h'\neq h'*m$).
\end{definition}

Besides the riffle shuffle, there are a number of
different 
shuffling strategies ---  the pairwise shuffle, 
for example, simply selects two cards
at random and swaps them.  The question then, is \emph{what are interleaving 
shuffling distributions $m$ that correspond to riffle shuffles?}
To answer this question, we use
the distinguishing property of the riffle shuffle, that, after
cutting the deck into two piles of size $p$ and $q=n-p$,
it must preserve
the relative ranking relations within each pile.  Thus, if
the $i^{th}$ card appears above the $j^{th}$ card in one of the piles, then
after shuffling, the $i^{th}$ card \emph{remains} above the $j^{th}$ card.
In our example, relative rank preservation says that if Peas
is preferred over Corn prior to shuffling, they continue to be
preferred over Corn after shuffling.
Any allowable riffle shuffling distribution must therefore assign zero
probability to permutations which do not preserve relative
ranking relations.  As it turns out, the set of permutations which
\emph{do} preserve these relations have a simple description.
\begin{definition}[Interleaving distributions]\label{def:riffleshuffle}
The \emph{$(p,q)$-interleavings} are defined as the following set:\vspace{-4mm}

{\footnotesize
\[
\Omega_{p,q} \equiv \{\tau\in S_n\,:\,\tau(1)<\tau(2)<\dots<\tau(p),\;\mbox{and}\;
	\tau(p+1)<\tau(p+2)<\dots<\tau(n)\}.
\]\vspace{-4mm}
}

\noindent A distribution $m_{p,q}$ on $S_n$ is called
an \emph{interleaving distribution}
if it assigns nonzero probability \emph{only} to elements in $\Omega_{p,q}$.
\end{definition}
The $(p,q)$-interleavings
can be shown to preserve relative ranking relations within each of the
subsets $A=\{1,\dots,p\}$ and $B=\{p+1,\dots,n\}$ upon
multiplication: 
\begin{lemma}\label{lem:relativerankpreservation}
Let $i,j\in A=\{1,\dots,p\}$ (or $i,j\in B=\{p+1,\dots,n\}$) and 
let $\tau$ be any $(p,q)$-interleaving
in $\Omega_{p,q}$.  Then $i<j$ \emph{if and only if}
$\tau(i)<\tau(j)$ \;\;(i.e., permutations in $\Omega_{p,q}$ preserve
relative ranking relations).
\end{lemma}

\begin{example}
\newcommand{\colformat}[6]{
({\color{black}{\bf #1}},{\color{black}{\bf #2}},{\color{black}{\bf #3}},{\color{black}{\bf #4}},{\color{black}{\bf #5}},{\color{black}{\bf #6}})}
In our vegetable/fruits example,
In our vegetable/fruits example,
we have $n=6$, $p=2$ (two vegetables, four fruits). 
The set of $(2,4)$-interleavings is:\vspace{-3mm}

{\tiny
\[
\Omega_{2,4} = \left\{
\begin{array}{ccccc}
\colformat{1}{2}{3}{4}{5}{6}, &
\colformat{1}{3}{2}{4}{5}{6}, &
\colformat{1}{4}{2}{3}{5}{6}, &
\colformat{1}{5}{2}{3}{4}{6}, &
\colformat{1}{6}{2}{3}{4}{5}, \\

\colformat{2}{3}{1}{4}{5}{6}, &
\colformat{2}{4}{1}{3}{5}{6}, &
\colformat{2}{5}{1}{3}{4}{6}, &
\colformat{2}{6}{1}{3}{4}{5}, &
\colformat{3}{4}{1}{2}{5}{6}, \\

\colformat{3}{5}{1}{2}{4}{5}, &
\colformat{3}{6}{1}{2}{4}{5}, &
\colformat{4}{5}{1}{2}{3}{6}, &
\colformat{4}{6}{1}{2}{3}{5}, &
\colformat{5}{6}{1}{2}{3}{4} \\
\end{array}
\right\},
\]\vspace{-3mm}
}

{\noindent or written in ordering notation,}\vspace{-3mm}
\newcommand{\cV}{{\color{black}{ {\bf V}} }}
\newcommand{\cF}{{\color{black}{ {\bf F}}} }

{\scriptsize
\[
\Omega_{2,4} =
\left\{
\begin{array}{ccccc}
\llbracket \cV\cV\cF\cF\cF\cF \rrbracket, &
\llbracket \cV\cF\cV\cF\cF\cF \rrbracket, &
\llbracket \cV\cF\cF\cV\cF\cF \rrbracket, &
\llbracket \cV\cF\cF\cF\cV\cF \rrbracket, &
\llbracket \cV\cF\cF\cF\cF\cV \rrbracket, \\

\llbracket \cF\cV\cV\cF\cF\cF \rrbracket, &
\llbracket \cF\cV\cF\cV\cF\cF \rrbracket, &
\llbracket \cF\cV\cF\cF\cV\cF \rrbracket, &
\llbracket \cF\cV\cF\cF\cF\cV \rrbracket, &
\llbracket \cF\cF\cV\cV\cF\cF \rrbracket, \\

\llbracket \cF\cF\cV\cF\cV\cF \rrbracket, &
\llbracket \cF\cF\cV\cF\cF\cV \rrbracket, &
\llbracket \cF\cF\cF\cV\cV\cF \rrbracket, &
\llbracket \cF\cF\cF\cV\cF\cV \rrbracket, &
\llbracket \cF\cF\cF\cF\cV\cV \rrbracket \\
\end{array}
\right\}.
\]\vspace{-3mm}
}

Note that the number of possible interleavings is 
$|\Omega_{p,q}|={n\choose p}={n\choose q}=6!/(2!4!)=15$.
One possible riffle shuffling distribution on $S_6$ might, 
for example, assign uniform probability ($m_{2,4}^{unif}(\sigma)=1/15$) 
to each permutation in $\Omega_{2,4}$ and zero probability to 
everything else, reflecting indifference between vegetables and fruits.
Figure~\ref{fig:interleavings} is a graphical example of
a $(2,4)$-interleaving distribution. 
\end{example}

We now formally define our generalization of independence
where a distribution which fully factors independently
is allowed to undergo a single riffle shuffle.
\begin{definition}[Riffled independence]\label{def:riffledindep}
\looseness -1 The subsets $A=\{1,\dots,p\}$
and $B=\{p+1,\dots,n\}$ are said to be
\emph{riffle independent}
if $h = m_{p,q}*(f_A(\sigma(A))\cdot g_B(\sigma(B)))$,
with respect to
some interleaving distribution $m_{p,q}$
and distributions $f_A, g_B$, respectively.
We will notate the riffled independence relation as $A\perp_m B$,
and refer to $f_A, g_B$ as \emph{relative ranking factors}.
\end{definition}
Notice that without the additional convolution, the definition
of riffled independence reduces to the fully independent
case given by Equation~\ref{eqn:fullyfactorized}.
\begin{example}
Consider drawing a ranking from a riffle independent model.
One starts with two piles of cards, $A$ and $B$, stacked together in a deck.
In our fruits/vegetables setting, if we always prefer
vegetables to fruits, then the vegetables occupy 
positions $\{1,2\}$ and the fruits occupy positions $\{3,4,5,6\}$.
In the first step, rankings of each pile are drawn independent.
For example, we might have the rankings:
$\sigma(\mbox{Veg}) = (2,1)$ and $\sigma(\mbox{Fruit})= (4,6,5,3)$,
constituting a draw from the fully independent model described in 
Section~\ref{sec:fullyindep}.
In the second stage, the deck of cards is cut and interleaved
by an independently selected element $\tau\in\Omega_{2,4}$.
For example, if: \vspace{-3mm}

{\footnotesize
\[
\tau = (2,3,1,4,5,6)=
\llbracket Fruit,Veg,Veg,Fruit,Fruit,Fruit\rrbracket,
\]\vspace{-4mm}
}

\noindent then the joint ranking is:\vspace{-3mm}

{\footnotesize
\begin{align*}
\tau(\sigma(Veg),\sigma(Fruit)) &= (2,3,1,4,5,6)(2,1,4,6,5,3) = (3,2,4,6,5,1), \\
	&= 
	\llbracket Grapes,Peas,Corn,Lemon,Fig,Orange\rrbracket.
\end{align*}
}
\end{example}

\subsection{Alternative definition of riffled independence}
It is possible to
rewrite the definition of riffled independence
so that it does not involve a convolution.
We first define functions which map a given full ranking
to relative rankings and interleavings for $A$ and $B$.
\begin{definition}
\mbox{}
\begin{itemize}\denselist
\item (\emph{Absolute ranks}):
Given a ranking $\sigma\in S_n$, and a subset
$A\subset\{1,\dots,n\}$, $\sigma(A)$ denotes
the \emph{absolute ranks} of items in $A$.
\item (\emph{Relative ranking map}):
Let $\phi_A(\sigma)$
denote the ranks of items in $A$ \emph{relative}
to the set $A$.  For example, in the
ranking $\sigma=\llbracket\aP,\aL,\aF,\aG,\aC,\aO\rrbracket$,
the relative ranks of the vegetables is $\phi_A(\sigma)=\llbracket \aP,\aC\rrbracket=\llbracket Peas,Corn\rrbracket$.  Thus, while corn is ranked fifth
in $\sigma$, it is ranked second in $\phi_A(\sigma)$.
Similarly, the relative ranks of the fruits is
$\phi_B(\sigma)=\llbracket\aL,\aF,\aG,\aO\rrbracket = \llbracket Lemons,Figs,Grapes,Oranges\rrbracket$.
\item (\emph{Interleaving map}):
Likewise, let $\tau_{A,B}(\sigma)$ denote the way in which
the sets $A$ and $B$ are interleaved by $\sigma$.  For example, using
the same $\sigma$ as above, the interleaving of vegetables
and fruits is $\tau_{A,B}(\sigma)=
\llbracket Veg,Fruit,Fruit,Fruit,Veg,Fruit\rrbracket$.
In ranking notation (as opposed to ordering notation), 
$\tau_{A,B}$ can be written as 
$(\mbox{sort}(\sigma(A)),\mbox{sort}(\sigma(B)))$.  
Note that for every possible interleaving, $\tau\in \Omega_{p,q}$ there are 
exactly $p!\times q!$ distinct permutations which are associated to $\tau$
by the interleaving map.
\end{itemize}
\end{definition}
Using the above maps, 
the following lemma provides an algebraic expression for how 
any permutation $\sigma$ can be uniquely decomposed into 
an interleaving composed with relative rankings of $A$ and $B$,
which have been ``stacked'' into one deck.
\begin{lemma}\label{lem:tau}
Let $A=\{1,\dots,p\}$, and $B=\{p+1,\dots,n\}$.
Any ranking $\sigma\in S_n$ can be decomposed \emph{uniquely} as
an interleaving $\tau\in \Omega_{p,q}$ composed with 
a ranking of the form $(\pi_p,\pi_q+p)$, where $\pi_p\in S_p$, $\pi_q\in S_q$,
and $\pi_q+p$ means that the number $p$ is added to every rank
in $\pi_q$.
Specifically, $\sigma=\tau(\pi_p,\pi_q+p)$
with $\tau=\tau_{A,B}(\sigma)$, $\pi_p=\phi_A(\sigma)$, 
and $\pi_q=\phi_B(\sigma)$ (Proof in Appendix).
\end{lemma}
Lemma~\ref{lem:tau} shows that one can think of a triplet
$(\tau\in\Omega_{p,q},\sigma_p\in S_p,\sigma_q\in S_q)$ as being 
coordinates which uniquely specify any ranking of items in $A\cup B$.
Using the decomposition, we can now state a second, perhaps more 
intuitive, definition of riffled independence in terms of the relative 
ranking and interleaving maps.  
\begin{definition}\label{def:riffledindep2}
Sets $A$ and $B$ are said to be \emph{riffle independent} if and only if, for every
$\sigma\in S_n$, the joint distribution $h$ factors as:
\begin{equation}\label{eqn:factorization}
h(\sigma) = m(\tau_{A,B}(\sigma))\cdot f_A(\phi_A(\sigma))\cdot g_B(\phi_B(\sigma)).
\end{equation}
\end{definition}
\begin{proposition}
Definitions~\ref{def:riffledindep} and~\ref{def:riffledindep2}
are equivalent.
\end{proposition}
\begin{proof}
Assume that $A=\{1,\dots,p\}$ and $B=\{p+1,\dots,n\}$ 
are riffle independent with respect to 
Definition~\ref{def:riffledindep}.  We will show
that Definition~\ref{def:riffledindep2} is also satisfied (the
opposite direction will be similar).  Therefore, we assume that
$h = m_{p,q}*(f(\sigma_A)\cdot g(\sigma_B))$.
Note that $f(\sigma_A)\cdot g(\sigma_B)$ is supported on the subgroup
$S_p\times S_q\equiv\{\sigma\in S_n\,:\,1\leq \sigma(i) \leq p,\,\mbox{whenever}\,
1\leq i\leq p \}$.

Let $\sigma=(\sigma_A,\sigma_B)$ be any ranking.
We will need to use a simple claim:
consider the ranking $\tau^{-1}\sigma$
(where $\tau\in \Omega_{p,q}$).
Then $\tau^{-1}\sigma$ is an element of the subgroup $S_p\times S_q$
if and only if $\tau=\tau_{A,B}(\sigma)$.
\vspace{-3mm}

{\footnotesize
\begin{align*}
[m_{p,q}*(f\cdot g)](\sigma) &= \sum_{\sigma'\in S_n} 
	m_{p,q}(\sigma')\cdot [f\cdot g] (\sigma'^{-1}\sigma), \\
	&= \sum_{\tau\in \Omega_{p,q}} m_{p,q}(\tau)
	\cdot[f\cdot g](\tau^{-1}\sigma), \qquad(\mbox{\em since $m_{p,q}$ is supported 		on $\Omega_{p,q}$}) \\
	&= m_{p,q}(\tau_{A,B}(\sigma))\cdot [f\cdot g]
		((\tau_{A,B}^{-1}(\sigma))\sigma), \\
	&\qquad\qquad(\mbox{\em by the claim above and since $f\cdot g$
		is supported on $S_p\times S_q$}) \\
	&= m_{p,q}(\tau_{A,B}(\sigma))\cdot [f\cdot g](\phi_A(\sigma),\phi_B(\sigma)), \qquad(\mbox{\em by Lemma~\ref{lem:tau}}) \\
	&= m_{p,q}(\tau_{A,B}(\sigma))\cdot f(\phi_A(\sigma))\cdot g(\phi_B(\sigma)). \qquad(\mbox{\em by independence of $f\cdot g$})
\end{align*}\vspace{-3mm}
}

Thus, we have shown that Definition~\ref{def:riffledindep2} has been
satisfied as well.
\end{proof}
\paragraph{Discussion}
We have presented two ways of thinking about riffled independence.
Our first formulation, in terms of convolution, is motivated by
the connections between riffled independence and card shuffling theory.
As we show in Section~\ref{sec:algorithms}, the convolution based
view is also crucial for working with Fourier coefficients of 
riffle independent distributions and analyzing the theoretical properties
of riffled independence.
Our second formulation on the other hand, shows the concept
of riffled independence to be remarkably simple --- that the
probability of a single ranking can be computed without summing
over all rankings (required in convolution) ---
a fact which may not have been obvious from Definition~\ref{def:riffledindep}.

Finally, for interested readers, 
the concept of riffled independence also has a simple and 
natural group theoretic description. 
By a fully factorized distribution, we refer to a distribution 
supported on the subgroup $S_p\times S_q$, 
which factors along the $S_p$ and $S_q$ ``dimensions''.
As we have discussed, such sparse distributions are not appropriate
for ranking applications, and one would like to work with distributions
capable of placing nonzero probability mass on all rankings.
In the case of the symmetric group, however, there is a third ``missing
dimension'' --- the coset space, $S_n/(S_p\times S_q)$.
Thus, the natural extension of full independence is to randomize
over a set of coset representatives of $S_p\times S_q$,
what we have referred to in the above discussion 
as \emph{interleavings}.
The draws from each set, $S_p$, $S_q$, and $S_n/(S_p\times S_q)$
are then independent in the ordinary sense, and we say
that the item sets $A$ and $B$ are riffle independent.

\paragraph{Special cases}
There are a number of special case distributions captured
by the riffled independence model that are useful for honing intuition.
We discuss these extreme cases in the following list.
\begin{itemize}\denselist
\item (\emph{Uniform and delta distributions}):
Setting the interleaving distribution and both relative ranking
factors to be uniform distributions yields
the uniform distribution over all full rankings.
Similarly, setting the same distributions to be delta distributions 
(which assign zero probability to all rankings but one)
always yields a delta distribution.

It is interesting to note that while $A$ and $B$ are always
fully independent
under a delta distribution, they are never independent under a uniform
distribution.  However,
both uniform and delta
distributions factor \emph{riffle independently} with respect
to any partitioning of the item set.  Thus, not only is
$A=\{1,\dots,p\}$ riffle independent $B=\{p+1,\dots,n\}$,
but in fact, any set $A$ is riffle independent of its complement.
\item (\emph{Uniform interleaving distributions}):
Setting the interleaving distribution to be uniform,
as we will discuss more in detail later, reflects complete indifference
between the sets $A$ and $B$, even if $f$ and $g$ encode complex
preferences within each set alone.
\item (\emph{Uniform relative ranking factors}):
Setting the relative ranking factors, $f$ and $g$ to be uniform distributions
means that with respect to the joint distribution $h$, all items in 
$A$ are completely interchangeable amongst each other (as are all items
in $B$). 
\item (\emph{Delta interleaving distributions}):
Setting the interleaving distribution, $m_{p,q}$, 
to be a delta distribution
on \emph{any} of the $(p,q)$-interleavings in $\Omega_{p,q}$
recovers the definition of ordinary
probabilistic independence, and thus riffled independence
is a strict generalization thereof (see Figure~\ref{fig:fullyindependent}).
Just as in the full independence regime,
where the distributions $f$ and $g$ are marginal distributions of
absolute rankings of $A$ and $B$, in the riffled independence regime,
$f$ and $g$ can be thought of as marginal distributions of the \emph{relative
rankings} of item sets $A$ and $B$.
\item (\emph{Delta relative ranking factor}):
On the other hand, if one of the relative ranking factors, say $f$, is a
delta distribution and the other two distributions $m_{p,q}$ and $g$
are uniform, then the resulting riffle independent distribution $h$ 
can be thought of as an indicator function for the set of 
rankings that are consistent with one particular incomplete ranking
(in which only the relative ranking of $A$ has been specified).
Such distributions can be useful in practice when the input data comes
in the form of incomplete rankings rather than full rankings.
\end{itemize}

\begin{figure*}[t!]
\begin{center}
\subfigure[]{
\includegraphics[width=.5\textwidth]{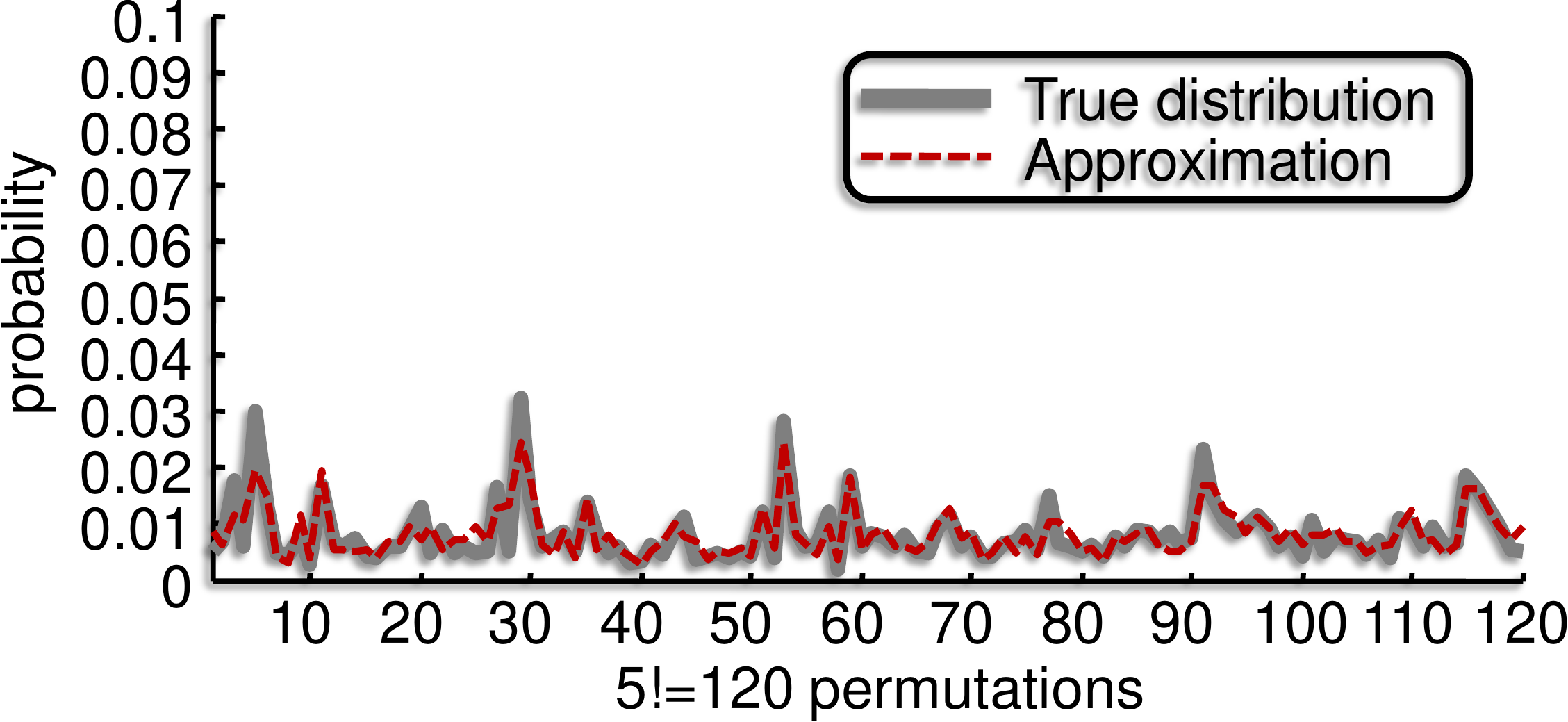}
   \label{fig:removecandidate2}
}\qquad
\subfigure[]{
\raisebox{3pt}{
\includegraphics[width=.3\textwidth]{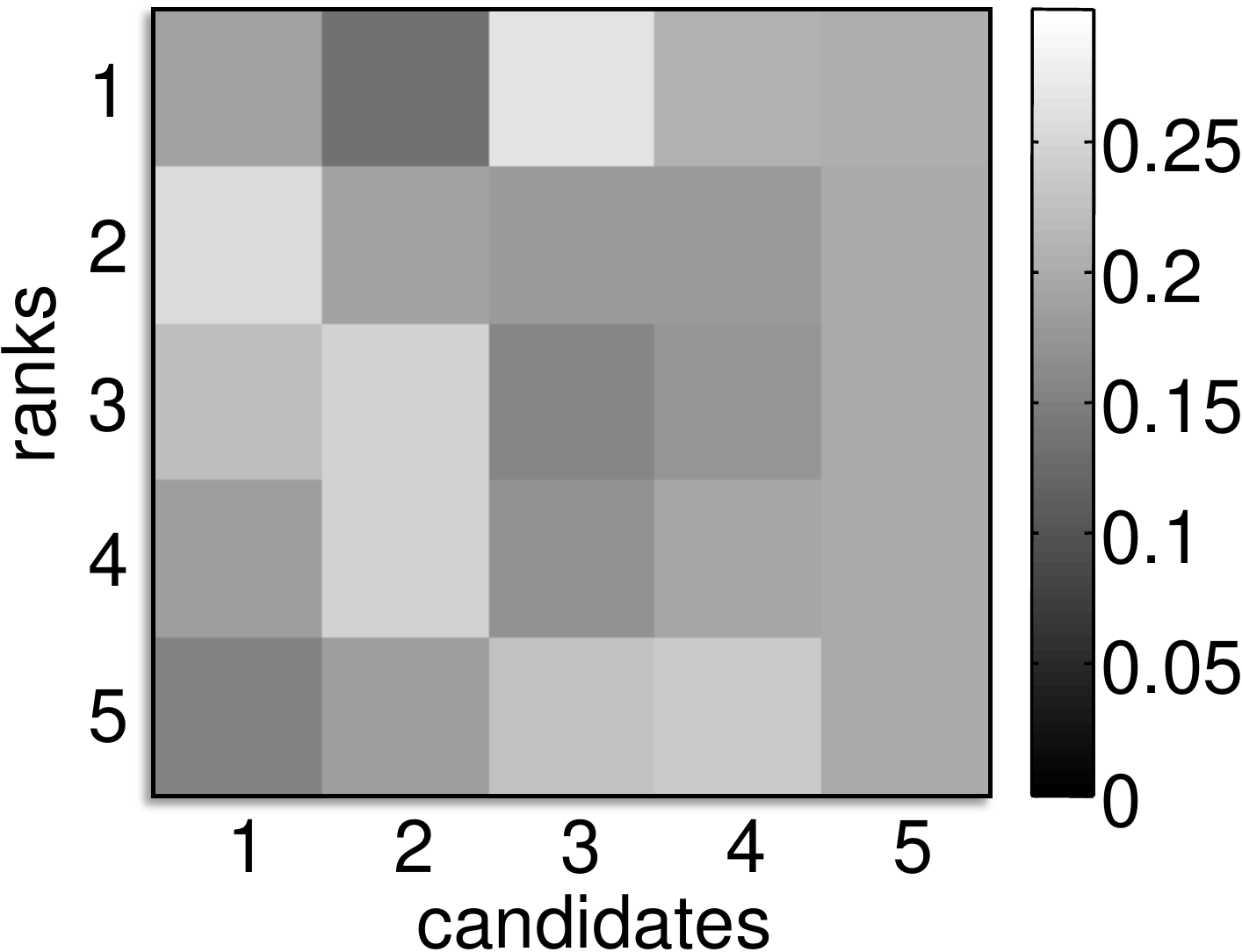}
   \label{fig:apafirstorder_removecandidate2}
}
}
\subfigure[]{
\includegraphics[width=.5\textwidth]{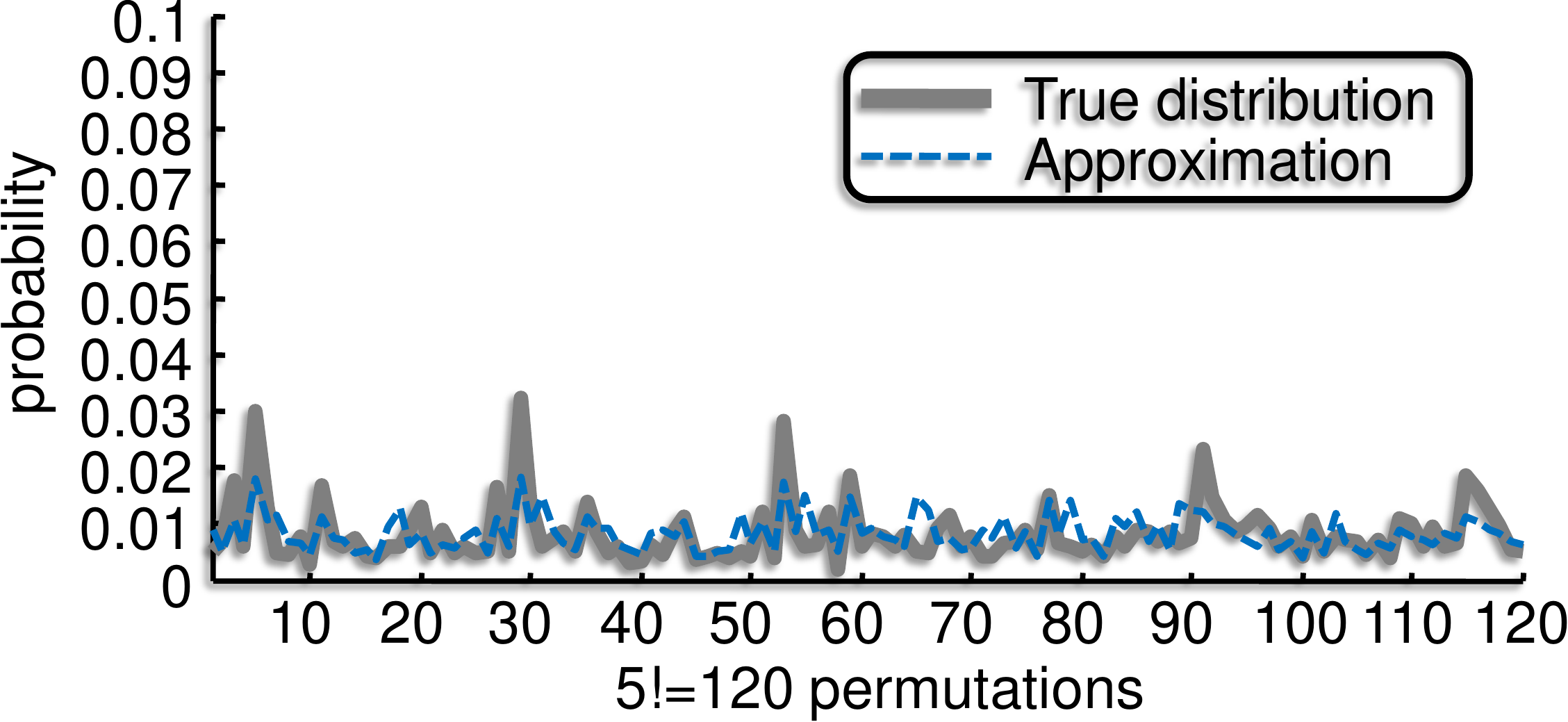}
   \label{fig:removecandidate3}
}\qquad
\subfigure[]{
\raisebox{3pt}{
\includegraphics[width=.3\textwidth]{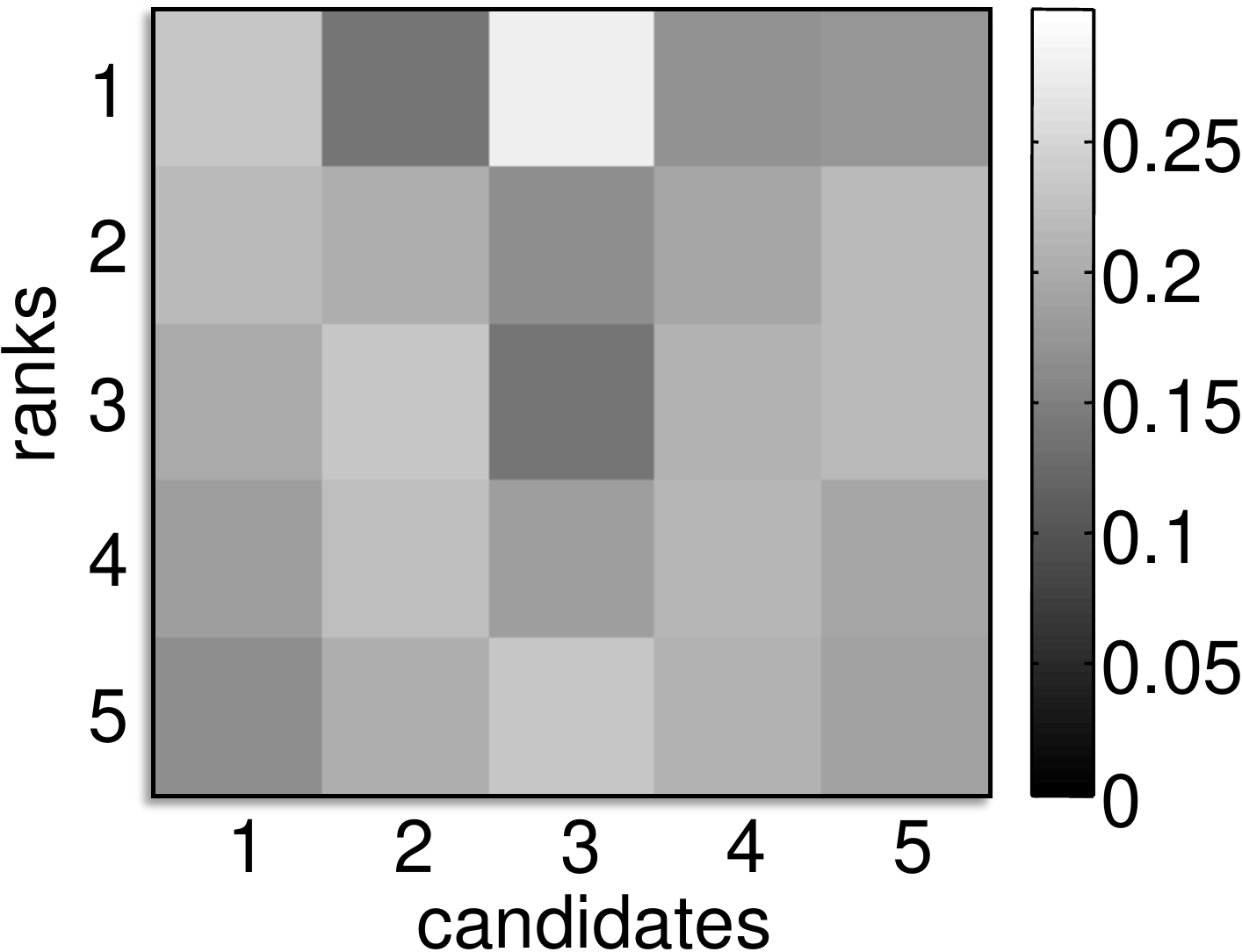}
   \label{fig:apafirstorder_removecandidate3}
}
}
\caption{Approximating the APA vote distribution by riffle independent
distributions.
\subref{fig:removecandidate2} approximate distribution when candidate 2
is riffle independent of remaining candidates;
\subref{fig:removecandidate3} approximate distribution when candidate 3
is riffle independent of remaining candidates;
\subref{fig:apafirstorder_removecandidate2}
and \subref{fig:apafirstorder_removecandidate3} corresponding first order
marginals of each approximate distribution.
}
\label{fig:apa_removesingleton}
\end{center}
\end{figure*}

\begin{example}[APA election data (continued)]\label{ex:removecandidate2}
Like the independence assumptions commonly used in naive Bayes models,
we would rarely expect riffled independence to exactly hold in real data.
Instead, it is more appropriate to view riffled independence assumptions
as a form of model bias that ensures learnability for small
sample sizes, which as we have indicated, is almost always the case
for distributions over rankings.

Can we ever expect riffled independence to be manifested in a real dataset?
In Figure~\ref{fig:removecandidate2}, we plot (in dotted red) a riffle 
independent approximation to the true APA vote distribution (in thick gray)
which is optimal with respect to KL-divergence (we will explain how
to obtain the approximation in the remainder of the paper).
The approximation in Figure~\ref{fig:removecandidate2} is obtained
by assuming that the candidate set $\{1,3,4,5\} $ is riffle independent
of $\{2\}$, and as can be seen, is quite accurate compared to the truth
(with the KL-divergence from the true to the factored distribution
being $d_{KL}=.0398$).
Figure~\ref{fig:apafirstorder_removecandidate2} exhibits the first order
marginals of the approximating distribution, which can also visually be seen 
to be a faithful approximation (see Figure~\ref{fig:apafirstorder}).
We will discuss the interpretation of the result further in Section~\ref{sec:hierarchical}.

For comparison, we also display 
(in Figures~\ref{fig:removecandidate3} and \ref{fig:apafirstorder_removecandidate3}) 
the result of approximating the true
distribution by one in which candidate $\{3\}$, the winner, is riffle independent
of the remaining candidate.  
The resulting approximation is inferior, and the lesson
to be learned in the example is that finding the correct/optimal
partitioning of the item set is important in practice.
We remark however, that the approximation obtained by factoring
out candidate 3 is not a terrible approximation (especially on examining first
order marginals), and that both approximations are far more accurate than
the fully independent approximation showed earlier in
Figure~\ref{fig:apafullyindependentexamples}. The KL divergence
from the true distribution to the factored distribution (with candidate 3 riffle
independent of the remaining candidates)
is $d_{KL}=.0841$.
\end{example}

\begin{figure}[t!]
\incmargin{1em}
\begin{algorithm2e}[H]
{\scriptsize
\dontprintsemicolon
{\sc DrawRiffleUnif($p,q,n$)} \xspace \tcp*[f]{\footnotesize ($p+q=n$)} \\
\SetKwIF{If}{ElseIf}{Else}{with prob}{}{otherwise}{otherwise}{endif}
\ \If{$q/n$ \tcp*[f]{\footnotesize drop from right pile}}{
\       $\sigma^-\leftarrow \mbox{\sc DrawRiffleUnif($p,q-1,n-1$)}$ \;
\       \lForEach{i}{$\sigma(i)\leftarrow 
                \left\{\begin{array}{cc}
                \sigma^{-}(i) & \mbox{if $i<n$} \\
                n & \mbox{if $i=n$}
                \end{array}\right.$}\;
}\ElseIf{\tcp*[f]{\footnotesize drop from left pile}}{
\       $\sigma^-\leftarrow \mbox{\sc DrawRiffleUnif($p-1,q,n-1$)}$ \;
\       \lForEach{i}{$\sigma(i)\leftarrow 
                \left\{\begin{array}{cc}
                \sigma^{-}(i) & \mbox{if $i<p$} \\
                n & \mbox{if $i=p$} \\
                \sigma^{-}(i-1) & \mbox{if $i>p$}
                \end{array}\right.$}\;
}
\ \Return{$\sigma$} \;
\caption{\scriptsize Recurrence for drawing $\sigma\sim m^{unif}_{p,q}$
(Base case: return $\sigma=[1]$ if $n=1$).}
\label{alg:riffle}
}
\end{algorithm2e}
\decmargin{1em}
\end{figure}

\subsection{Interleaving distributions}\label{sec:interleavings}
There is, in the general case, a significant increase in
storage required for riffled independence over full independence.
In addition to the $O(p!+q!)$ storage
required for distributions $f$ and $g$, we now require $O({n\choose p})$
storage for the nonzero terms of the riffle shuffling distribution $m_{p,q}$.
We now introduce a family of useful riffle shuffling distributions
which can be described using only a handful of parameters.
The simplest riffle shuffling distribution is the
\emph{uniform riffle shuffle}, $m_{p,q}^{unif}$,
which assigns uniform probability to all $(p,q)$-interleavings
and zero probability to all other elements in $S_n$.
Used in the context
of riffled independence, $m_{p,q}^{unif}$ models potentially
complex relations within $A$ and
$B$, but only captures the simplest possible correlations across
subsets.  We might, for example, have complex preference relations amongst
vegetables and amongst fruits, but be completely indifferent with respect
to the subsets, vegetables and fruits, as a whole.

There is a simple recursive method for uniformly drawing
$(p,q)$-interleavings.
Starting with a deck of $n$ cards cut
into a left pile ($\{1,\dots,p\}$)
and a right pile ($\{p+1,\dots,n\}$),  pick one of the piles
with probability proportional to its size ($p/n$ for the left pile, $q/n$
for the right) and drop the bottommost card, thus mapping either
card $p$ or card $n$ to rank $n$.  Then recurse on
the $n-1$ remaining undropped cards, drawing a $(p-1,q)$-interleaving
if the right pile was picked, or a $(p,q-1)$-interleaving
if the left pile was picked. See Algorithm~\ref{alg:riffle}.

\begin{figure*}[t!]
\begin{center}
\subfigure[$\alpha=0$]{
\includegraphics[width=.22\textwidth]{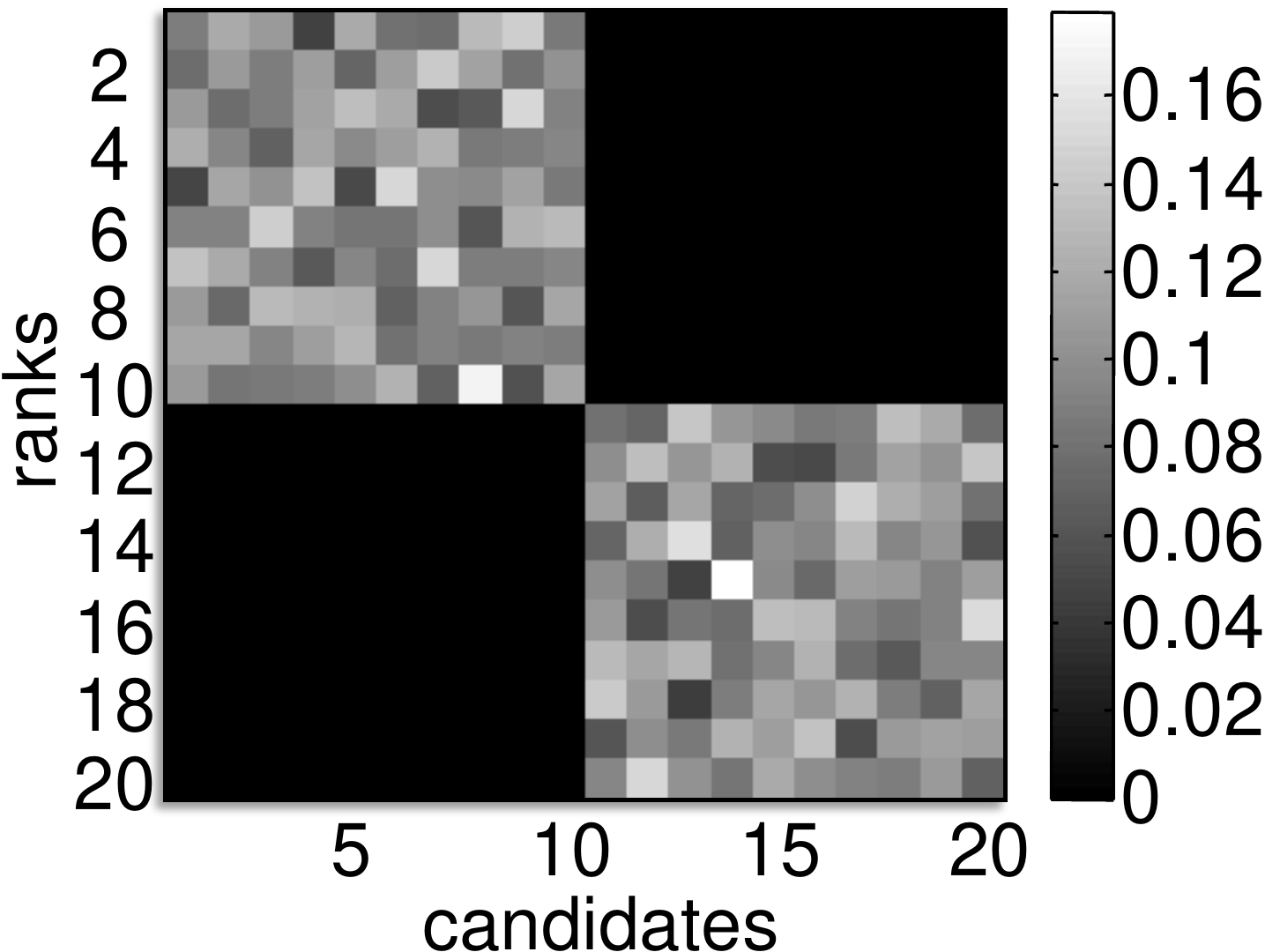}
   \label{fig:riffleindependent1}
}
\subfigure[$\alpha=1/6$]{
\includegraphics[width=.22\textwidth]{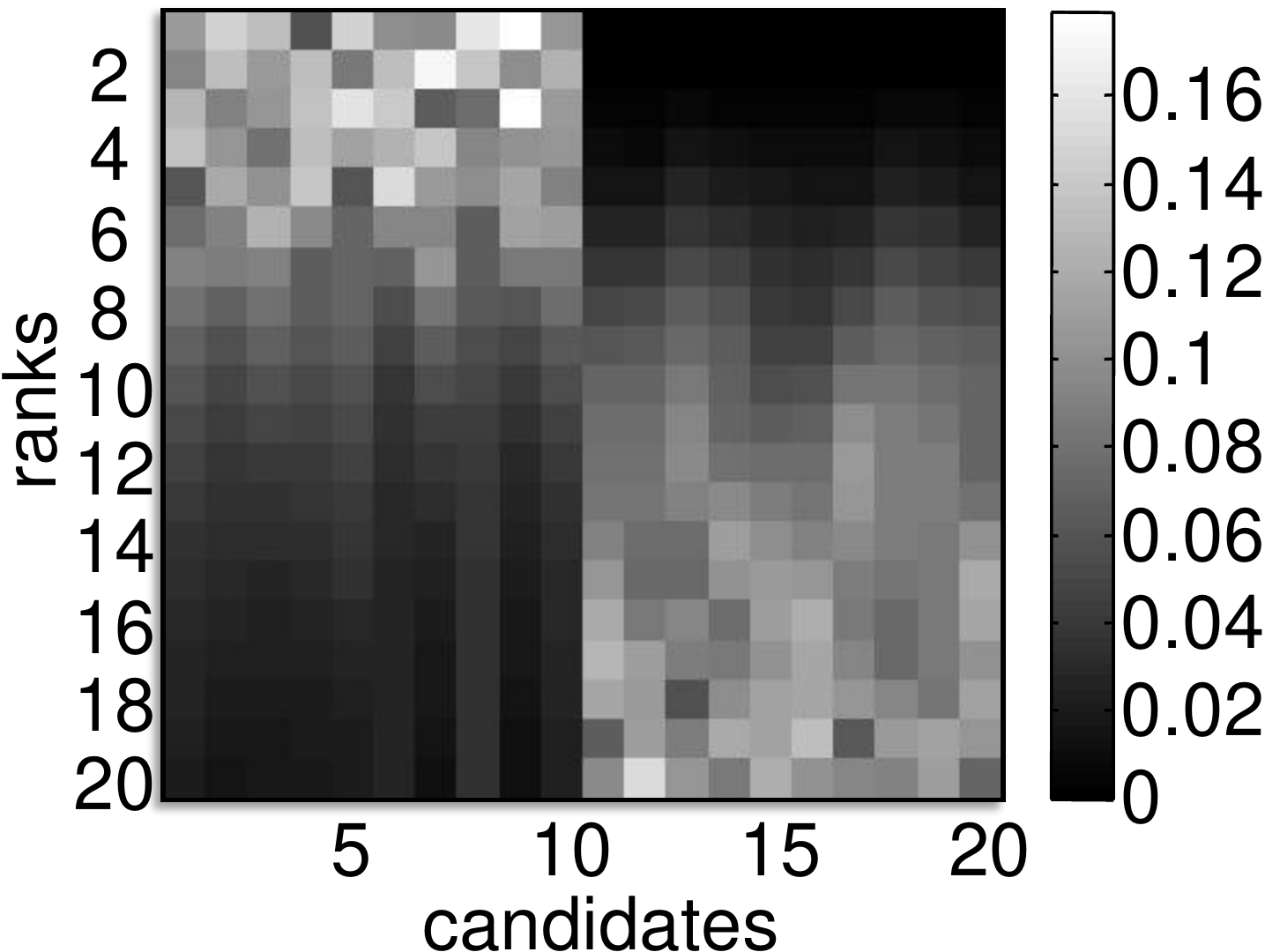}
   \label{fig:riffleindependent2}
}
\subfigure[$\alpha=1/3$]{
\includegraphics[width=.22\textwidth]{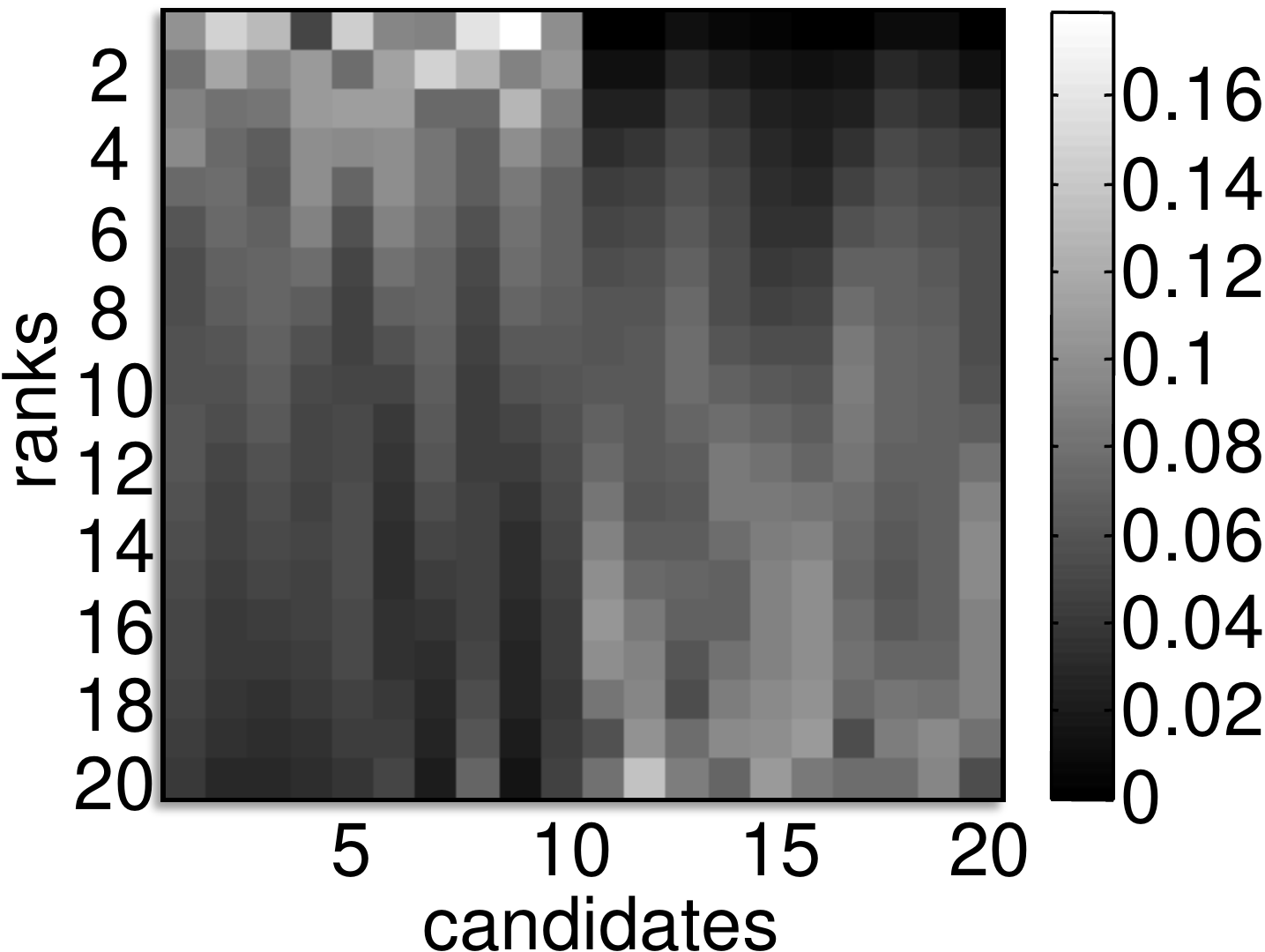}
   \label{fig:riffleindependent3}
}
\subfigure[$\alpha=1/2$]{
\includegraphics[width=.22\textwidth]{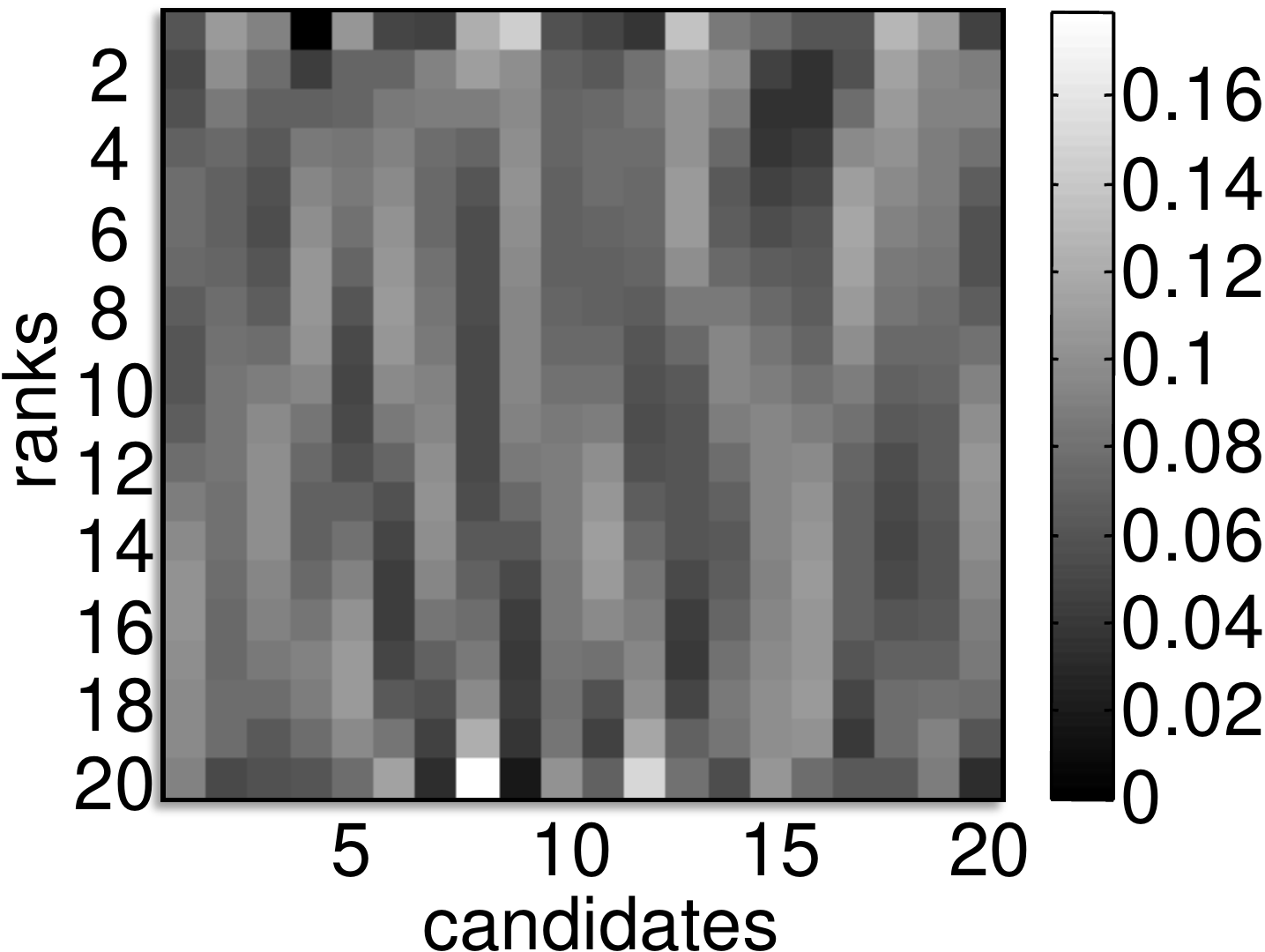}
   \label{fig:riffleindependent4}
}
\subfigure[$\alpha=2/3$]{
\includegraphics[width=.22\textwidth]{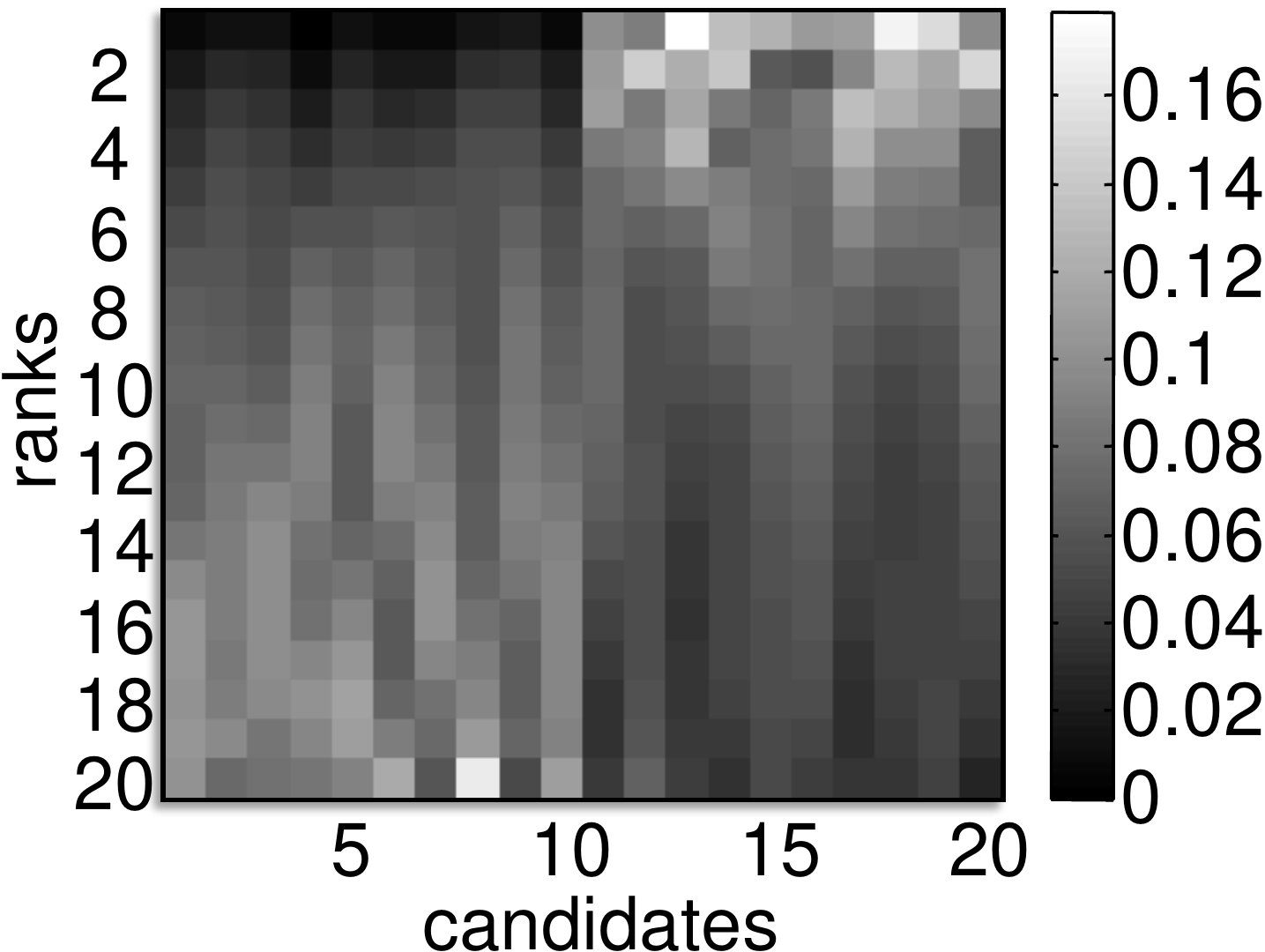}
   \label{fig:riffleindependent5}
}
\subfigure[$\alpha=5/6$]{
\includegraphics[width=.22\textwidth]{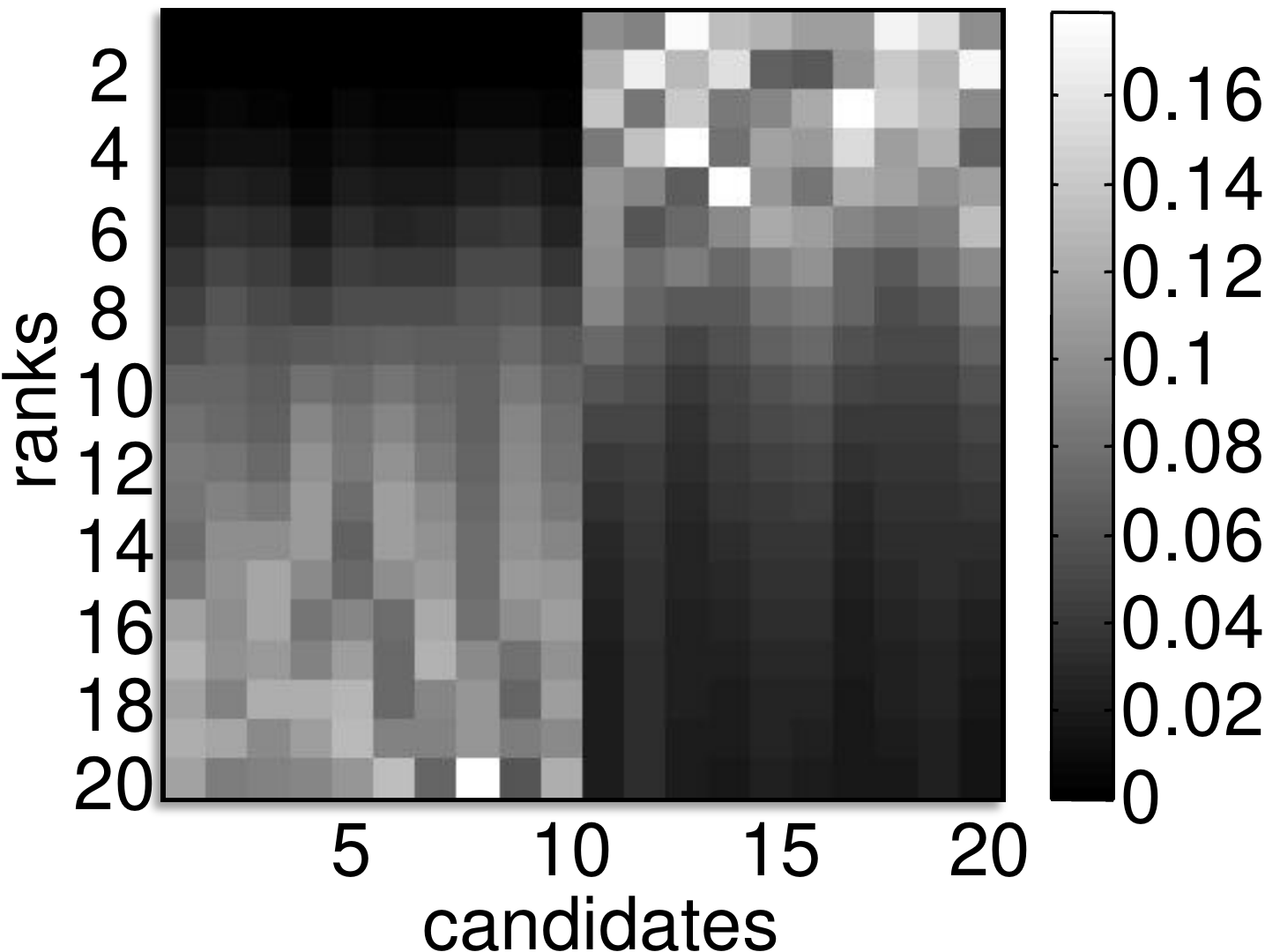}
   \label{fig:riffleindependent6}
}
\subfigure[$\alpha=1$]{
\includegraphics[width=.22\textwidth]{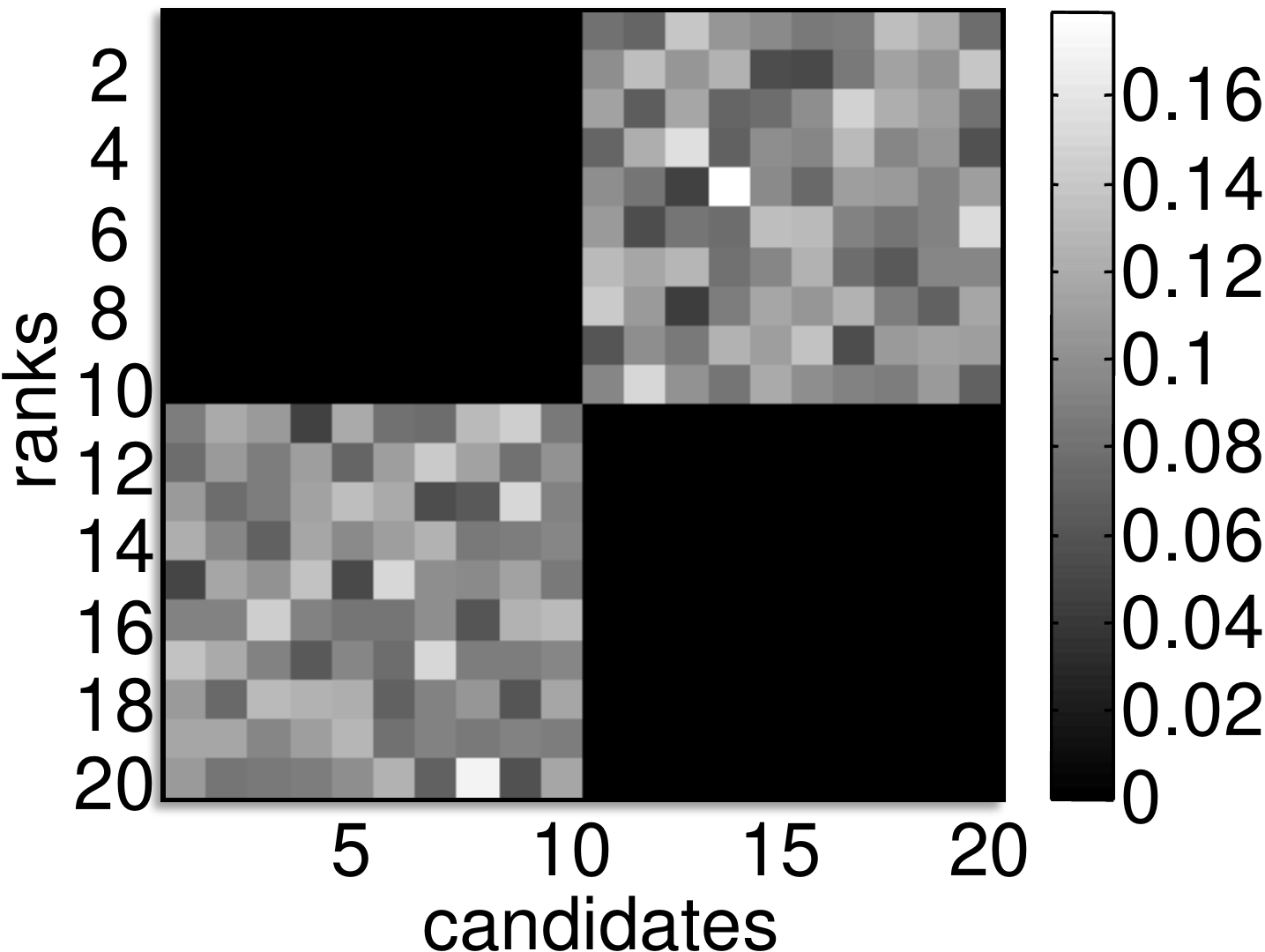}
   \label{fig:riffleindependent7}
}
\caption{First-order matrices
with a deck of 20 cards,
$A=\{1,\dots,10\}$, $B=\{11,\dots,20\}$,
\emph{riffle independent} and various
settings of $\alpha$.
Compare these matrices to the fully independent first order
marginal matrices of Figure~\ref{fig:fullyindependent} and note that here,
the nonzero blocks are allowed to `bleed' into zero regions.
Setting $\alpha=0$ or $1$, however,
recovers the fully independent case, where a subset of objects is preferred over
the other with probability one.
}
\label{fig:biasedriffle}
\end{center}
\end{figure*}

It is natural to consider generalizations
where one is preferentially biased towards dropping cards from the left hand
over the right hand (or vice-versa).
We model this bias using a simple one-parameter family of distributions in
which cards from the left and right piles drop with probability proportional
to $\alpha p$ and $(1-\alpha)q$, respectively, instead of $p$ and $q$.
We will refer to $\alpha$ as the \emph{bias parameter}, and the family
of distributions parameterized by $\alpha$ as the \emph{biased
riffle shuffles}.\footnote{The recurrence in Alg.~\ref{alg:riffle}
has appeared in various forms in literature~\citep{diaconis92}.
We are the first to (1) use the recurrence to Fourier transform
$m_{p,q}$, and to (2) consider biased versions.  The biased riffle shuffles
in~\cite{fulman98} are not similar to our biased riffle shuffles.
}

In the context of rankings, biased riffle shuffles provide a simple model
for expressing groupwise preferences (or indifference)
for an entire subset $A$ over $B$ or vice-versa.
The bias parameter $\alpha$
can be thought of as a knob controlling the preference for one
subset over the other, and might reflect,
for example, a preference for fruits over vegetables, or perhaps
indifference between the two subsets.
Setting $\alpha=0$ or $1$ recovers the full independence assumption, preferring
objects in $A$ (vegetables) over objects in $B$ (fruits)
with probability one (or vice-versa),
and setting $\alpha=.5$, recovers the uniform riffle shuffle
(see Fig.~\ref{fig:biasedriffle}).
Finally, there are a number of straightforward
generalizations of the
biased riffle shuffle that one can use to realize richer
distributions.  For example, $\alpha$ might depend on the number of
cards that have been dropped from each pile (allowing perhaps, for distributions
to prefer crunchy \emph{fruits} over crunchy \emph{vegetables}, but
soft \emph{vegetables} over soft \emph{fruits}).

\section{Exploiting structure for probabilistic inference}
In this section, we discuss a number of basic properties
of riffled independence, which show that certain probabilistic
inference operations can be accomplished by operating on a single factor rather
than the entire joint distribution.

Upon knowing that $A$ is riffle independent of $B$,
an immediate consequence is that we can show,
just as in the full independence regime, that
conditioning operations on certain observations and
MAP (maximum a posteriori) assignment
problems decompose according to riffled independence structure.
All of the following properties are straightforward to derive using the 
factorization in Definition~\ref{def:riffledindep2}.
\begin{proposition}[Probabilistic inference decompositions]\label{prop:inference}
\mbox{}
\begin{itemize}\denselist
\item 
(\emph{Conditioning}):
Consider
prior and likelihood functions, $h_{prior}$ and $h_{like}$,
on $S_n$ in which subsets $A$ and $B$ are riffle independent,
with parameters $(m_{prior},f_{prior},g_{prior})$ and 
$(m_{like},f_{like},g_{like})$, respectively.
Let $\odot$ denote the pointwise product operation between two functions.
Then $A$ and $B$ are also riffle independent with respect to the 
posterior distribution under Bayes rule, which has interleaving distribution
$m_{like}\odot m_{prior}$ with relative ranking factors 
$f_{like}\odot f_{prior}$ and $g_{like}\odot g_{prior}$, for
$A$ and $B$ respectively.
\item (\emph{MAP assignment}):
Let $A$ and $B$ be riffle independent subsets.
Consider the following permutations:
\[
\pi_p^*=\arg\max_\pi f_A(\pi), \;\; \pi_q^*=\arg\max_\pi g_B(\pi),\;\;
\tau^*=\arg\max_\tau m_{p,q}(\tau).
\]
Then the mode of $h$ is $\tau^*$ composed with $\pi^* = (\pi_p^*,\pi_q^*)$ 
(i.e., $\arg\max_\sigma h(\sigma) = \tau^*\pi^*$).
\item (\emph{Entropy}):
Consider riffle independent subsets $A$ and $B$.
The entropy of the joint distribution is given by:
$H[h] = H[m_{p,q}]+H[f_A]+H[g_B]$.
\end{itemize}
\end{proposition}

Some ranked datasets come in the form of pairwise
comparisons, with records of the form ``object $i$ is
preferred to object $j$''.  
As a corollary to Proposition~\ref{prop:inference}, we now argue
that conditioning on these
pairwise ranking likelihood functions (that depend only on
whether object $i$ is preferred to object $j$)
decomposes along riffled independence structures.
The pairwise ranking model~\citep{huangetal09b}
for objects $i$ and $j$, is defined over $S_n$ as:\vspace{-3mm}

{\footnotesize
\[
f_{like}(\sigma)=\delta^n_{\sigma(i)<\sigma(j)}(\sigma) = \left\{\begin{array}{cc}
\beta & \mbox{if $\sigma(i)<\sigma(j)$} \\
1-\beta & \mbox{otherwise}
\end{array}\right.,\;\;0\leq\beta\leq 1,
\]\vspace{-3mm}
}

\noindent and reflects the fact that object $i$ is preferred to object $j$
(with probability $\beta$).  If objects
$i$ and $j$ both belong to one of the sets, say $A$, 
then only one factor requires an update using Bayes rule.
If vegetables and fruits are riffle independent, for example,
then less computation would be required
to compare a vegetable against a vegetable than to compare a fruit
against a vegetable.
For example, the observation that Corn is preferred over Peas
affects only the distribution, $f_A$, over vegetables.
More formally, we state this intuitive corollary as follows:
\begin{corollary}\label{cor:pairwiseranking}
Consider conditioning on the pairwise ranking model, 
$f_{like}(\sigma)=\delta^n_{\sigma(i)<\sigma(j)}(\sigma)$,
and suppose that $A$ and $B$ are riffle independent subsets with respect
to the prior distribution $h_{prior}$, with parameters 
$(m_{prior}$, $f_{prior}$, $g_{prior})$.
If $i,j\in A$, then $A$ and $B$ are riffle independent with 
respect to the posterior distribution, whose parameters are identical
to those of the prior, except for the relative ranking 
factor corresponding to $A$, which is 
$f_{post}=f_{prior}\odot \delta_{\sigma(i)<\sigma(j)}^{p}$.
\end{corollary}
\begin{proof}[Proof sketch]
First show that the subsets $A$ and $B$ are 
riffle independent with respect to the likelihood function, 
$\delta^n_{\sigma(i)<\sigma(j)}$ by equating the likelihood function
to a product of a uniform interleaving distribution, $m^{unif}_{p,q}$,
and relative ranking factor $f_A=\delta^n_{\sigma(i)<\sigma(j)}$ for $A$,
and a uniform relative ranking factor for $B$.
Then apply Proposition~\ref{prop:inference}.
\end{proof}
Let us compare the result of the corollary to what is possible 
with a fully factored distribution.
If $A$ and $B$ were fully independent, then conditioning 
on any distribution which involved items in $A$ (or only in $B$)
would require only updating the factor associated with item set $A$.
For example, if $i\in A$, then first-order observations of the form
``item $i$ is in rank $j$'' can be efficiently conditioned
in the fully independent scenario.
With riffled independence, it is not, in general, possible to condition
on such first-order observations without modifying all of the $O(n!)$ 
parameters.  However, as Corollary~\ref{cor:pairwiseranking} shows,
pairwise comparisons involving $i,j$ both in $A$ (or both in $B$)
\emph{can} be performed exactly by updating either $f$ (or $g$) without
having to touch all $O(n!)$ probabilities.


\section{Algorithms for a fixed partitioning of the item set}\label{sec:algorithms}
We have thus far covered a number of intuitive examples and properties
of riffled independence.
Given a set of rankings drawn from some distribution $h$, 
we are now interested in estimating a number of statistical quantities, such
as the parameters of a riffle independent model.
In this section, we will assume a \emph{known structure}
(that the partitioning of the item set into subsets $A$ and $B$ is known), 
and given such a partitioning of the item set, 
we are interested in the 
problem of estimating parameters (which we will refer to as \emph{RiffleSplit}),
and the inverse problem of computing probabilities (or marginal
probabilities) with given parameters (which we will refer to 
as \emph{RiffleJoin}).

\paragraph{RiffleSplit} 
In RiffleSplit (which we will also refer to as the \emph{parameter
estimation} problem), we would like to estimate various statistics of the 
relative ranking and interleaving
distributions of a riffle independent distribution
($m_{p,q}$, $f_A$, and $g_B$).
Given a set of i.i.d. training examples, 
$\sigma^{(1)},\dots, \sigma^{(m)}$, we might, for example, want to estimate
each raw probability (e.g., estimate $m_{p,q}(\tau)$ for each interleaving
$\tau$).  In general, we may be 
interested in estimating more general statistics (e.g., what
are the second order relative ranking probabilities of the set of
fruits?).

Since our variables are discrete, computing the maximum likelihood
parameter estimates consists of forming counts of the number of 
training examples consistent with a given interleaving or relative ranking.
Thus, the MLE parameters in our problem are simply given by the following
formulas: \vspace{-3mm}

{\footnotesize\allowdisplaybreaks
\begin{align}
m_{p,q}^{MLE}(\tau) &\propto \sum_{i=1}^m \mathds{1}\left[\tau=\tau_{A,B}(\sigma^{(i)})\right], 
	\label{eqn:m_mle}\\
f_A^{MLE}(\sigma_A) &\propto \sum_{i=1}^m \mathds{1}\left[\sigma_A = \phi_A(\sigma^{(i)})\right], 
	\label{eqn:f_mle} \\
g_B^{MLE}(\sigma_B) &\propto \sum_{i=1}^m \mathds{1}\left[\sigma_B = \phi_B(\sigma^{(i)})\right].
	\label{eqn:g_mle}
\end{align}\vspace{-2mm}
}

\paragraph{RiffleJoin}
Having estimated parameters of a riffle independent distribution, 
we would like to now compute various statistics of the data itself.
In the simplest case, we are interested in estimating $h(\sigma)$,
the joint probability of a single ranking, which can be evaluated simply
by plugging parameter estimates of $m_{p,q}$, $f_A$, and $g_B$ into our
second definition of riffled independence (Definition~\ref{def:riffledindep2}).

More generally however, we may be interested
in knowing the low-order statistics of the data (e.g.,
the first order marginals, second order marginals, etc.),
or related statistics (such as $h(\sigma(i)<\sigma(j))$, the probability
that object $i$ is preferred to object $j$).
And typically for such low-order statistics, one must compute
a sum over rankings.  For example, to compute the probability that 
item $j$ is ranked in position $i$, one must sum over $(n-1)!$ rankings:\vspace{-3mm}

{\footnotesize
\begin{equation}\label{eqn:bigsum}
h(\sigma\,:\,\sigma(j)=i) =
	\sum_{\sigma\,:\,\sigma(j)=i} m_{p,q}(\tau_{A,B}(\sigma))\cdot
		(f_A(\phi_A(\sigma))\cdot g_B(\phi_B(\sigma))).
\end{equation}\vspace{-3mm}
}

While Equation~\ref{eqn:bigsum} may be feasible for small $n$
(such as on the APA dataset), the sum quickly grows to be
intractable for larger $n$.
One of the main observations of the remainder of this section,
however, is that low-order marginal probabilities of the joint
distribution can always be computed directly from low-order 
marginal probabilities of the relative ranking and interleaving
distributions without explicitly computing intractable sums.

\subsection{Fourier theoretic algorithms for riffled independence}
We now present algorithms for working with riffled independence 
(solving the RiffleSplit and RiffleJoin problems) in the Fourier
theoretic framework of~\cite{kondor07,huangetal09a,huangetal09b}.
The Fourier theoretic perspective of riffled independence presented
here is valuable because it will allow us to work directly with 
low-order statistics instead of having to form the necessary
raw probabilities first.
Note that readers who are primarily interested in the structure learning
can jump directly to Section~\ref{sec:hierarchical}.

We begin with a brief introduction to Fourier theoretic inference
on permutations
(see~\cite{kondor08b,huangetal09b} for a detailed exposition).
Unlike its analog on the real line, the Fourier
transform of a function on $S_n$ takes the form of a collection
of Fourier coefficient \emph{matrices} ordered with respect to frequency.
Discussing the analog
of frequency for functions on $S_n$,
is beyond the scope of our
paper, and, given a distribution $h$, we simply index the
Fourier coefficient matrices of $h$ as $\widehat{h}_0$, $\widehat{h}_1$,
$\dots$, $\widehat{h}_K$ ordered with respect to some measure of
increasing complexity.
We use $\widehat{h}$ to denote the complete collection of Fourier
coefficient matrices.
One rough way to understand this complexity, as mentioned in
Section~\ref{sec:distributions}, is by the fact that
the low-frequency Fourier coefficient matrices of a distribution
can be used to reconstruct low-order marginals.
For example, the first-order matrix of marginals of $h$ can always
be reconstructed from the matrices $\hat{h}_0$ and $\hat{h}_1$.
As on the real line, many of the familiar properties of the Fourier
transform continue to hold.  The following are several basic properties
used in this paper:
\begin{proposition}[Properties of the Fourier transform,
~\cite{diaconis88}]\label{prop:fourierprops}
\mbox{}
Consider any $f,g:S_n\to\reals$.
\begin{itemize}\denselist
\item (Linearity) For any
$\alpha,\beta\in\reals$,
$[\widehat{\alpha f+\beta g}]_i=\alpha\widehat{f}_i+\beta\widehat{g}_i$ holds
at all frequency levels $i$.
\item (Convolution) The Fourier transform of a convolution is a
product of Fourier transforms:
$[\widehat{f*g}]_{i}=\widehat{f}_{i}\cdot \widehat{g}_{i}$,
for each frequency level $i$, where the operation $\cdot$
is matrix multiplication.
\item (Normalization) The first coefficient matrix, $\hat{f}_0$,
is a scalar and equals
$\sum_{\sigma\in S_n} f(\sigma)$.
\end{itemize}
\end{proposition}
A number of papers in recent years
(\cite{kondor07,huangetal07,huangetal09a,huangetal09b})
have considered approximating distributions over permutations
using a truncated (bandlimited) set of Fourier coefficients and have proposed
inference algorithms that operate on these
Fourier coefficient matrices.  For example, one can perform
generic marginalization, Markov chain prediction, and
conditioning operations using only Fourier coefficients
without ever having to perform an inverse Fourier transform.

In this section, we provide
generalizations of the algorithms in~\cite{huangetal09a} that
tackle the RiffleJoin and RiffleSplit problems.
We will assume, without loss of generality 
that $A=\{1,\dots,p\}$ and $B=\{p+1,\dots,n\}$ 
(this assumption will be discarded in later sections),
Although we begin each of the following discussions as if all of the
Fourier coefficients are provided, we will be especially interested
in algorithms that work well in cases where only a truncated set of
Fourier coefficients are present, and where $h$ is only
\emph{approximately} riffle independent.

For both problems, we will rely on two Fourier domain algorithms introduced 
in~\cite{huangetal09a}, \emph{Join} and \emph{Split}, as subroutines.
Given independent factors $f:S_p\to\reals$
and $g:S_q\to\reals$, Join returns 
the joint distribution $f\cdot g$.  Conversely, given 
a distribution $h:S_n\to\reals$, Split computes
$f$ and $g$ by marginalizing over $S_q$ or $S_p$, respectively.
For example, ${\mbox{\sc Split}}[h]$ returns a function defined on $S_p$, and 
${\mbox{\sc Split}}[h](\sigma_p) = \sum_{\sigma_q\in S_q} h((\sigma_p,\sigma_q))$.
We will overload the ${\mbox{\sc Join/Split}}$ names to 
refer to both the ordinary and Fourier theoretic formulations 
of the same procedures.

\begin{figure}[t!]
\incmargin{1em}
\begin{algorithm2e}[H]
{\scriptsize
\SetKwInOut{Input}{input}\SetKwInOut{Output}{output}
{\sc RiffleJoin($\widehat{f},\widehat{g},\widehat{m_{p,q}}$)} \xspace \\
\Input{Fourier transforms of $f_A$, $g_B$, and $m$ ($\hat{f}, \hat{g},\widehat{m_{p,q}}$ respectively)}
\Output{Fourier transform of the joint distribution, $\hat{h}$}
\BlankLine
 $\widehat{h'} = \mbox{\sc Join}(\widehat{f},\widehat{g})$ \;
 \ForEach{frequency level $i$}{
       $\widehat{h}_{i}\;\leftarrow\; \left[\widehat{m_{p,q}}\right]_{i}\cdot \widehat{h'}_{i}$ \;
    }
 \Return $\widehat{h}$ \;
\caption{\footnotesize Pseudocode for \emph{RiffleJoin}}
\label{alg:rifflejoin}
}
\end{algorithm2e}
\decmargin{1em}
\end{figure}

\begin{figure}[t!]
\incmargin{1em}
\begin{algorithm2e}[H]
{\scriptsize
\SetKwInOut{Input}{input}\SetKwInOut{Output}{output}
{\sc RiffleSplit($\widehat{h}$)} \xspace \\
\Input{Fourier transform of the empirical joint distribution $\hat{h}$}
\Output{Fourier transform of MLE estimates of $f_A$, $g_B$ ($\hat{f},\hat{g}$)}
\BlankLine
 \ForEach{frequency level $i$}{
       $\widehat{h'}_{i}\;\leftarrow\; \left[\widehat{m}_{p,q}^{unif}\right]^T_{i}\cdot \widehat{h}_{i}$ \;
    }
 $[\widehat{f},\widehat{g}]\;\leftarrow\; \mbox{\sc Split}(\widehat{h'})$ \;
 Normalize $\hat{f}$ and $\hat{g}$\;
 \Return $\hat{f},\hat{g}$\;
\caption{\footnotesize Pseudocode for \emph{RiffleSplit}}
\label{alg:rifflesplit}
}
\end{algorithm2e}
\decmargin{1em}
\end{figure}

\subsection{\emph{RiffleJoin} in the Fourier domain}
Given the Fourier coefficients of $f$, $g$, and $m$, we
can compute the Fourier coefficients of $h$ using
Definition~\ref{def:riffledindep} (our first definition) 
by applying the Join algorithm
from~\cite{huangetal09a} and the \emph{Convolution Theorem}
(Proposition~\ref{prop:fourierprops}), which tells
us that the Fourier transform of a convolution can be written
as a pointwise product of Fourier transforms.
To compute the $\hat{h}_i$, the Fourier theoretic formulation of
the \emph{RiffleJoin} algorithm simply
calls the Join algorithm on $\widehat{f}$ and $\widehat{g}$,
and convolves the result by $\widehat{m}$
(see Algorithm~\ref{alg:rifflejoin}).

In general, it may be intractable to Fourier transform
the riffle shuffling distribution $m_{p,q}$.  However, there are
some cases in which $m_{p,q}$ can be computed.
For example, if $m_{p,q}$ is computed directly from a set of training examples,
then one can simply compute the desired Fourier coefficients
using the definition of the Fourier transform
given in~\cite{huangetal09b}, which is tractable as long as
the samples can be tractably stored in memory.
For the class of biased riffle shuffles  
that we discussed in Section~\ref{sec:riffledindependence},
one can also efficiently compute the low-frequency terms
of $\widehat{m_{p,q}^\alpha}$ by employing the recurrence relation in
Algorithm~\ref{alg:riffle}. 
In particular, Algorithm~\ref{alg:riffle}
expresses a biased riffle shuffle on $S_n$
as a linear combination of biased riffle shuffles on $S_{n-1}$.
By invoking linearity of the Fourier transform
(Proposition~\ref{prop:fourierprops}), one can efficiently
compute $\widehat{m_{p,q}^\alpha}$
via a dynamic programming approach quite reminiscent
of Clausen's FFT (\emph{Fast Fourier transform}) 
algorithm~\cite{clausen93}.
We describe our algorithm in more 
detail in Appendix~\ref{sec:fourierinterleavings}.
To the best of our knowledge, we are the first to compute the Fourier
transform of riffle shuffling distributions.

\subsection{\emph{RiffleSplit} in the Fourier domain}
Given the Fourier coefficients of a riffle independent
distribution $h$, we would like to tease apart the factors.
In the following, we show how to recover the relative
ranking distributions, $f_A$ and $g_B$, and
defer the problem of recovering the interleaving distribution 
for Appendix~\ref{sec:fourierinterleavings}.

From the RiffleJoin algorithm, we saw that for each frequency level
$i$, $\hat{h}_{i}=\left[\widehat{m_{p,q}}\right]_{i}
\cdot[\widehat{f\cdot g}]_{i}$.
The first solution to the splitting problem that might occur is to
perform a deconvolution by multiplying each $\widehat{h}_{i}$
term by the inverse of the matrix $\left[\widehat{m_{p,q}}\right]_{i}$
(to form $\left[\widehat{m_{p,q}}\right]^{-1}_{i}\cdot 
\widehat{h}_{i}$) and call the Split algorithm from~\cite{huangetal09a}
on the result.  Unfortunately, the matrix $\left[\widehat{m_{p,q}}\right]_{i}$
is, in general, non-invertible.
Instead, our RiffleSplit algorithm left-multiplies each
$\widehat{h}_{i}$ term by
$[\widehat{m}_{p,q}^{unif}]^T_{i}$, which can be shown
to be equivalent to convolving the distribution $h$ by the `\emph{dual
shuffle}', $m^*$, defined as $m^*(\sigma)=m^{unif}_{p,q}(\sigma^{-1})$.
While convolving by $m^*$ does not produce a distribution that
factors independently, the Split algorithm from~\cite{huangetal09a}
can still be shown to recover the Fourier transforms
$\hat{f}_A^{MLE}$ and $\hat{g}^{MLE}_B$ of the maximum likelihood
parameter estimates:
\begin{theorem}\label{thm:rifflesplit}
Given a set of rankings with empirical distribution $\tilde{h}$,
the maximum likelihood estimates of the relative ranking distributions over
item sets $A$ and $B$ are given by:
\begin{equation}
[f^{MLE}_A,g^{MLE}_B] \propto {\mbox{\sc Split}}\left[m^*_{p,q} * \tilde{h}\right],
\end{equation}
where $m^*_{p,q}$ is the dual shuffle (of the uniform interleaving distribution).
Furthermore, the Fourier transforms of the relative ranking distributions are:
\begin{equation*}
[(\widehat{f^{MLE}})_i,(\widehat{g^{MLE}})_i] \propto 
	{\mbox{\sc Split}}\left[\left(\hat{m}_{p,q}^{unif}\right)^T_i\cdot \hat{\tilde{h}}_i\right],\;\mbox{for all frequency levels $i$}.
\end{equation*}
\end{theorem}
\begin{proof}
We will use $\pi_p\in S_p$ and $\pi_q\in S_q$
to denote relative rankings of $A$ and $B$ respectively.
Let us consider estimating $f_A^{MLE}(\pi_p)$.
If $\tilde{h}$ is the empirical distribution of the training examples,
then $f_A^{MLE}(\pi_p)$ can be computed by summing over examples
in which the relative ranking of $A$ is consistent with $\pi_p$ (Equation~\ref{eqn:f_mle}), or
equivalently, by marginalizing $\tilde{h}$ over the interleavings
and the relative rankings of $B$.  Thus, we have:\vspace{-3mm}

{\footnotesize
\begin{equation}\label{eqn:marginalizerelranks}
f_A^{MLE}(\pi_p) = \sum_{\pi_q\in S_q} \left(\sum_{\tau\in \Omega_{p,q}} \tilde{h}(\tau(\pi_p,\pi_q+p))\right),
\end{equation}\vspace{-3mm}
}

\noindent where we have used Lemma~\ref{lem:tau} to decompose a ranking $\sigma$
into its component relative rankings and interleaving.

The second step is to notice that the outer summation of
Equation~\ref{eqn:marginalizerelranks} is exactly the type of marginalization
that can already be done in the Fourier domain via the Split
algorithm of~\cite{huangetal09a}, and thus, $f_A^{MLE}$ can be
rewritten as $f_A^{MLE}=\mbox{\sc Split}(h')$,
where the function $h':S_n\to\reals$ is defined as
$h'(\sigma) = \sum_{\tau\in \Omega_{A,B}} h(\tau\sigma)$.
Hence, if we could compute the Fourier transform of the function
$h'$, then we could apply the ordinary Split algorithm to recover
the Fourier transform of $f_A^{MLE}$.

In the third step, we observe that the function $h'$ can be written as a convolution
of the dual shuffle with $h$, thus establishing the first part of the theorem:\vspace{-3mm}

{\footnotesize
\begin{equation*}
h'(\sigma) = \sum_{\tau\in \Omega_{A,B}} h(\tau\sigma)
        \propto \sum_{\pi\in S_n} m^{unif}_{p,q}(\pi) h(\pi\sigma)
        \propto \sum_{\pi\in S_n} m^*_{p,q}(\pi) h(\pi^{-1}\sigma)
        \propto [m^*_{p,q} * h](\sigma).
\end{equation*}\vspace{-3mm}
}

\noindent Next, we use a standard fact about Fourier 
transforms~\citep{diaconis88} --- given a function $m^*:S_n\to\reals$ 
defined as $m^*(\sigma)=m(\sigma^{-1})$,
the Fourier coefficient matrices of $m^*$ are related to those of $m$ by
the transpose.  Hence, $\hat{m^*}_i^T=\hat{m}_i$, for every frequency level $i$.
Applying the convolution theorem to the Fourier coefficients of the dual shuffle and the
empirical distribution establishes the final part of the theorem.
\end{proof}
Notice that to compute the MLE relative ranking factors 
in the Fourier domain, it is not necessary to know the interleaving
distribution.  
It is necessary, however,
to compute the Fourier coefficients
of the \emph{uniform} interleaving distribution ($m^{unif}_{p,q}$), 
which we discuss in Appendix~\ref{sec:fourierinterleavings}.  
It is also necessary to normalize the output
of Split to sum to one, but fortunately, normalizing a function $h$
can be performed in the Fourier
domain simply by dividing each Fourier coefficient matrix by $\widehat{h}_0$
(Proposition~\ref{prop:fourierprops}).  See Algorithm~\ref{alg:rifflesplit}
for pseudocode.

\subsection{Marginal preservation guarantees}
Performing our Fourier domain algorithms with a complete set
of Fourier coefficients is just as intractable as performing
the computations naively.  Typically, in the Fourier setting, one
hopes instead to work with a set of low-order terms.
For example, in the case of RiffleJoin, we might only receive 
the second order marginals of the parameter distributions as input.
A natural question to ask then, is what is the approximation 
quality of the output given a bandlimited input?
We now state a result below, which shows how
our algorithms perform when called with a truncated set
of Fourier coefficients.
\begin{theorem}\label{thm:bandlimiting}
Given enough Fourier terms to reconstruct the $k^{th}$-order marginals
of $f$ and $g$, RiffleJoin returns enough Fourier terms to \emph{exactly}
reconstruct the $k^{th}$-order marginals of $h$.  Likewise,
given enough Fourier terms to reconstruct the $k^{th}$-order marginals
of $h$, RiffleSplit returns enough Fourier terms to \emph{exactly}
reconstruct the $k^{th}$-order marginals of both $f$ and $g$.
\end{theorem}
\begin{proof}
This result is a simple consequence of the well-known
convolution theorem (Proposition~\ref{prop:fourierprops})
and Theorems 9 and 12 from~\cite{huangetal09a}.
Theorem 9 from~\cite{huangetal09a} states that, given $s^{th}$-order
marginals of factors $f$ and $g$, the Join algorithm
can reconstruct the $s^{th}$-order marginals of the joint distribution
$f\cdot g$, \emph{exactly}.  Since the riffle independent joint distribution
is $m*(f\cdot g)$ and convolution operations are pointwise in the
Fourier domain (Proposition~\ref{prop:fourierprops}), 
then given enough Fourier
terms to reconstruct the $s^{th}$-order marginals of the
function  $m^{p,q}$, we can also reconstruct the $s^{th}$-order
marginals of the riffle independent joint from the output of RiffleSplit.
\end{proof}
\subsection{Running time}
If the Fourier coefficient matrix for frequency level $i$ of a
joint distribution
is $d\times d$  then 
the running time complexity of the Join/Split algorithms
of~\cite{huangetal09a} are at worst, cubic in the dimension, $O(d^3)$.
If the interleaving Fourier coefficients are precomputed ahead of
time, then the complexity of RiffleJoin/RiffleSplit
is also $O(d^3)$.

If not, then we must Fourier transform the interleaving distribution.
For RiffleJoin, we can Fourier transform the empirical distribution 
directly from the definition, or use the Algorithms presented in 
Appendix~\ref{sec:fourierinterleavings} in the case of biased riffle shuffles,
which has $O(n^2d^3)$ running time
in the worst case when $p\sim O(n)$.
For RiffleSplit, one must compute the Fourier transform
of the uniform interleaving distribution, which, as we have shown in 
Section~\ref{sec:interleavings}, also takes the form of a biased riffle 
shuffle and therefore also can be computed in $O(n^2d^3)$ time.
In Section~\ref{sec:experiments}, we plot experimental running times.

\section{Hierarchical riffle independent decompositions}\label{sec:hierarchical}
Thus far throughout the paper, we have focused exclusively on understanding
riffled independent models with a single binary partitioning of the 
full item set.
In this section we explore a natural model simplification 
which comes from the simple
observation that, since the relative ranking distributions
$f_A$ and $g_B$ are again distributions over rankings,
the sets $A$ and $B$ can further be decomposed into 
riffle independent subsets.
We call such models \emph{hierarchical riffle independent decompositions}.
Continuing with our running example, one can imagine that 
the fruits are further partitioned
into two sets, a set consisting of citrus fruits ((\aL) Lemons 
and (\aO) Oranges) and a set consisting of mediterranean 
fruits ((\aF) Figs and (\aG) Grapes).  To generate a full
ranking, one first draws rankings of the citrus and 
mediterranean fruits independently
($\llbracket\aL,\aO\rrbracket$ and $\llbracket \aG,\aF\rrbracket$, 
for example).  Secondly, the two sets are interleaved to form a 
ranking of all fruits ($\llbracket\aG,\aL,\aO,\aF\rrbracket$).
Finally, a ranking of the vegetables is drawn 
($\llbracket \aP,\aC\rrbracket$) and interleaved with the fruit 
rankings to form a full joint ranking: 
$\llbracket \aP,\aG,\aL,\aO,\aF,\aC\rrbracket$.
Notationally, we can express the hierarchical decomposition
as $\{\aP,\aC\}\perp_{m_1} (\{\aL,\aO\}\perp_{m_2}\{\aF,\aG\})$.  We can also 
visualize hierarchies using trees (see Figure~\ref{fig:tree1} for our
example).  The subsets of items which appear as leaves in the 
tree will be referred to as \emph{leaf sets}.
\begin{figure*}[t!]
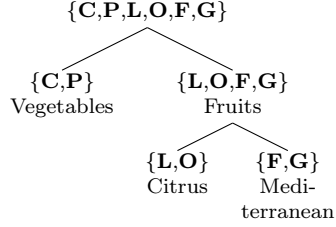
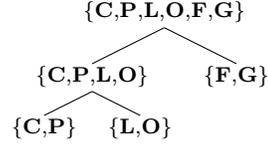
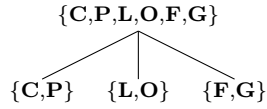
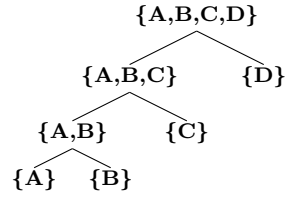

\begin{center}
{\scriptsize
\subfigure[Example of hierarchical riffled independence structure on $S_6$]{
\begin{minipage}[t][1.3in]{0.33\linewidth}
        \qquad
\Tree [.{\{\aC,\aP,\aL,\aO,\aF,\aG\}} {\{\aC,\aP\}\vspace{0mm}}\\{\scriptsize Vegetables}
[.{\{\aL,\aO,\aF,\aG\}\vspace{0mm}}\\{\scriptsize Fruits}
{\{\aL,\aO\}\vspace{0mm}}\\{\scriptsize Citrus}
{\{\aF,\aG\}\vspace{0mm}}\\{\scriptsize Medi-\vspace{0mm}}\\{\scriptsize terranean} ]   ]
        \qquad
   \label{fig:tree1}
\end{minipage}
}
\qquad
\qquad
\qquad
\subfigure[Another example, not equivalent to~\subref{fig:tree1}]{
\begin{minipage}[t][1.3in]{0.3\linewidth}
        \qquad
\Tree [.{\{\aC,\aP,\aL,\aO,\aF,\aG\}}
[.{\{\aC,\aP,\aL,\aO\}}
{\{\aC,\aP\}}
{\{\aL,\aO\}} ]
{\{\aF,\aG\}} ]
        \qquad
   \label{fig:tree4}
\end{minipage}
}
\qquad
\subfigure[3-way decomposition for $S_6$ (generalizes the class of distributions
parameterized by~\subref{fig:tree1},~\subref{fig:tree4}]{
\begin{minipage}[t][1.3in]{0.33\linewidth}
        \qquad
\Tree [.{\{\aC,\aP,\aL,\aO,\aF,\aG\}} {\{\aC,\aP\}} {\{\aL,\aO\}} {\{\aF,\aG\}}  ]
        \qquad
   \label{fig:tree2}
\end{minipage}
}
\qquad
\qquad
\qquad
\subfigure[Hierarchical decomposition into singleton subset, where each leaf
set consists of a single item (we will also refer to this particular
type of tree as a 1-thin chain)]{
\begin{minipage}[t][1.3in]{0.33\linewidth}
        \qquad{\bf
\Tree [.{\{A,B,C,D\}} [.{\{A,B,C\}} [.\{A,B\} {\{A\}} {\{B\}} ] {\{C\}} ]  {\{D\}} ]
        }\qquad
   \label{fig:tree3}
\end{minipage}
}
}
\caption{Examples of distinct hierarchical riffle independent structures.}
\end{center}
\end{figure*}

A natural question to ask is: if we used a different
hierarchy with the same leaf sets, would we capture the same
distributions?  For example, does a distribution which decomposes according
to the tree in Figure~\ref{fig:tree4} also decompose according
to the tree in Figure~\ref{fig:tree1}?
The answer, in general, is no, due to the fact that
distinct hierarchies impose different sets of independence
assumptions, and as a result, different structures can be
well or badly suited for modeling a given dataset.
Consequently, it is important to use the ``correct'' structure
if possible.

\subsection{Shared independence structure}
It is interesting to note, however, that 
while the two structures in Figures~\ref{fig:tree1} and~\ref{fig:tree4}
encode distinct families of distributions, it is possible
to identify a set of independence assumptions common to both structures.
In particular since both structures have the same leaf sets,
any distributions consistent with either of the two hierarchies 
must also be consistent with what we call a \emph{$3$-way decomposition}.
We define a $d$-way decomposition to be a distribution with
a single level of hierarchy, but instead of partitioning the
entire item set into just two subsets, one partitions into
$d$ subsets, then interleaves the relative rankings of each of the
$d$ subsets together
to form a joint ranking of items.
Any distribution consistent with either Figure~\ref{fig:tree4}
or~\ref{fig:tree1} must consequently also be consistent
with the structure of Figure~\ref{fig:tree2}.  More generally, we have:
\begin{proposition}\label{prop:leafsets}
If $h$ is a hierarchical riffle independent model with $d$ leaf
sets, then $h$ can also be written as a $d$-way decomposition.
\end{proposition}
\begin{proof}
We proceed by induction.  Suppose the result holds for $S_{n'}$
for all $n'<n$. We want to establish that the result also holds for $S_n$.
If $h$ factors according to a hierarchical riffle independent model, then it
can be written as $h=m\cdot f_A\cdot g_B$, where $m$ is the interleaving
distribution, and $f_A$, $g_B$ themselves factor as hierarchical
riffle independent distributions with, say, $d_1$ and $d_2$ leaf
sets, respectively (where $d_1+d_2=d$).  By the hypothesis, since $|A|,|B|<n$,
we can factor both $f_A$ and $g_B$ as $d_1$ and $d_2$-way decompositions
respectively.
We can therefore write $f_A$ and $g_B$ as:\vspace{-3mm}

{\footnotesize
\begin{align*}
f_A(\pi_A) = m_A(\tau_{A_1,\dots,A_{d_1}})
        \cdot \prod_{i=1}^{d_1} f_{A_i} \left(\phi_{A_i}(\pi_A)\right),\;\;
g_B(\pi_B) = m_B(\tau_{B_1,\dots,B_{d_2}})
        \cdot \prod_{i=1}^{d_2} g_{B_i} \left(\phi_{B_i}(\pi_B)\right).
\end{align*}\vspace{-3mm}
}

\noindent Substituting these decompositions into the factorization of the
distribution $h$,  we have:\vspace{-3mm}

{\footnotesize\allowdisplaybreaks
\begin{align*}
h(\sigma) &= m(\tau_{A,B}(\sigma))f_A(\phi_A(\sigma))g_B(\phi_B(\sigma)), \\
        &= \left(m(\tau_{A,B}(\sigma))m_A(\tau_{A_1\dots,A_{d_1}})m_B(\tau_{B_1,\dots,B_{d_2}})\right) \\
	&\qquad\qquad\cdot\prod_{i=1}^{d_1} f_{A_i}\left(\phi_{A_i}(\phi_A(\sigma))\right)\prod_{i=1}^{d_2} g_{B_i}\left(\phi_{B_i}(\phi_B(\sigma))\right), \\
        &= \tilde{m}(\tau_{A_1,\dots,A_{d_1},B_1\dots,B_{d_2}}) 
		\cdot\prod_{i=1}^{d_1} f_{A_i}\left(\phi_{A_i}(\sigma)\right) \prod_{i=1}^{d_2} g_{B_i}\left(\phi_{B_i}(\sigma)\right),
\end{align*}\vspace{-3mm}
}

\noindent where the last line follows because any legitimate
interleaving of the sets $A$ and $B$
is also a legitimate interleaving of the
sets $A_1,\dots,A_{d_1},B_1,\dots,B_{d_2}$
 and since $\phi_{A_i}(\phi_A(\sigma))=\phi_{A_i}(\sigma)$.
This shows that the distribution $h$ factors as a $d_1+d_2$-way
decomposition, and concludes the proof.
\end{proof}
In general, knowing the hierarchical decomposition
of a model is more desirable than knowing its $d$-way
decomposition which may require many more parameters 
$\left(O(\frac{n!}{\prod_i d_i!}), \mbox{where $i$ indexes over leaf sets}\right)$.
For example, the $n$-way decomposition requires $O(n!)$ parameters
and captures every distribution over permutations.

\subsection{Thin chain models}\label{sec:thinchains}
There is a class of particularly simple hierarchical models
which we will refer to as $k$-thin chain models.
By a $k$-thin chain model, we refer to a hierarchical structure
in which the size of the
smaller set at each split in the
hierarchy is fixed to be a constant 
and can therefore be expressed as:
\begin{align*}
(A_1\perp_m (A_2 \perp_m (A_3\perp_m \dots))), \; |A_i|=k,\mbox{for all $i$}.
\end{align*}
See Figure~\ref{fig:tree3} for an example of $1$-thin chain.
We view thin chains as being somewhat analogous to
thin junction tree models~\citep{bach01}, in which
cliques are never allowed to have more than $k$ variables.
When $k\sim O(1)$, for example, the number of model parameters scales polynomially in $n$.
To draw rankings from a thin chain model, one sequentially
inserts items independently,
one group of size $k$ at a time, into the full ranking.

\begin{theorem}\label{thm:suffstats}
The $k^{th}$ order marginals 
are sufficient statistics for a $k$-thin chain model.
\end{theorem}
\begin{proof}
Corollary of Theorem~\ref{thm:bandlimiting}
\end{proof}


\begin{figure*}[t!]
\begin{center}
\includegraphics[width=.28\textwidth]{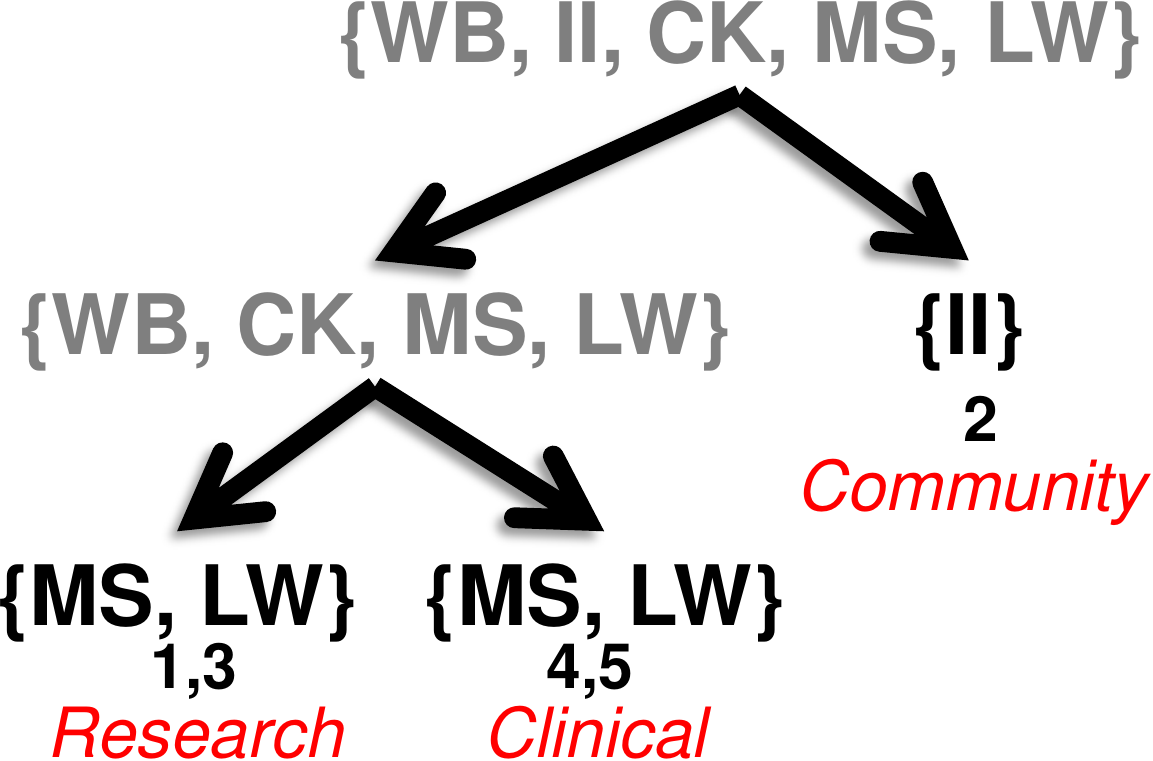}
\caption{Hierarchical structure learned from APA data.}
   \label{fig:apatree}
\end{center}
\end{figure*}
\begin{example}[APA election data (continued)]\label{ex:apahierarchy}
The APA, as described by~\cite{diaconis89}, is divided
into ``\emph{academicians and clinicians who are on uneasy terms}''.
In 1980, candidates $\{1,3\}$ (W. Bevan and C. Kiesler
who were research psychologists) and 
$\{4,5\}$ (M.Siegle and L. Wright, who were clinical 
psychologists) fell on opposite
ends of this political spectrum with candidate 2 (I. Iscoe) 
being somewhat independent.
Diaconis conjectured that voters choose one group over the other, and then
choose within.  We are now able to verify Diaconis' conjecture using our
riffled independence framework.
After removing candidate 2 from the distribution, we perform
a search within candidates $\{1,3,4,5\}$ to again find \emph{nearly}
riffle independent subsets.  We find that
$A=\{1,3\}$ and $B=\{4,5\}$ are very nearly riffle independent (with 
respect to KL divergence) and
thus are able to verify that candidate sets $\{2\}$, $\{1,3\}$,
$\{4,5\}$ are indeed grouped in a riffle independent sense
in the APA data.  We remark that in a later work, ~\cite{marden95}
identified candidate 2 (I. Iscoe)
as belonging to yet a third group of psychologists
called \emph{community psychologists}.  The hierarchical structure 
that best describes the APA data is shown in Figure~\ref{fig:apatree}
and the KL-divergence from the true distribution 
to the hierarchical model is $d_{KL}=.0676$.

Finally for the two main opposing groups within the APA,
the riffle shuffling distribution for sets $\{1,3\}$ and $\{4,5\}$
is not well approximated by a biased riffle shuffle.  
Instead, since there are two coalitions, 
we fit a mixture of two biased riffle shuffles to the data
and found the bias parameters of the mixture components to be
$\alpha_1\approx.67$ and $\alpha_2\approx .17$, indicating that the two
components oppose each other (since $\alpha_1$ and $\alpha_2$ lie on
either side of $.5$).
\end{example}

\section{Structure discovery I: objective functions}\label{sec:structure1}
Since different hierarchies impose different
independence assumptions, we would like to find the structure
that is best suited for modeling a given ranking dataset.
On some datasets, a natural hierarchy might be available ---
for example, if one were familiar with the typical politics of APA elections, 
then it may have been possible to ``guess'' the optimal hierarchy.
However, for general ranked data, it is not always obvious 
what kind of groupings riffled independence will lead to, particularly
for large $n$.
Should fruits really be riffle independent of vegetables?  Or
are green foods riffle independent of red foods?

Over the next three sections, we address the problem of 
automatically discovering hierarchical 
riffle independent structures from training data.
Key among our observations is the fact that
while item ranks cannot be independent due to mutual
exclusivity, relative ranks between sets of items are not
subject to the same constraints.
More than simply being a `clustering' algorithm, however, our procedure
can be thought of as a structure learning algorithm, like those from the
graphical models literature~\cite{koller09}, which find
the optimal (riffled) independence decomposition of a distribution.

The base problem that we address in this current section is how to
find the best structure if there is only one level of partitioning
and two leaf sets, $A$, $B$.  Alternatively, we want to find the topmost
partitioning of the tree.  In Section~\ref{sec:structure2}, we use
this base case as part of a top-down approach for learning a full hierarchy.
\subsection{Problem statement}
Given then, a training set of rankings,
$\sigma^{(1)}$, $\sigma^{(2)}$, $\dots,\sigma^{(m)}\sim h$, drawn i.i.d.
from a distribution in which a
subset of items, $A\subset\{1,\dots,n\}$,
is riffle independent of its complement, $B$, the problem which 
we address in this section is that of
automatically determining the sets $A$ and $B$.
If $h$ does not \emph{exactly} factor riffle independently, then we
would like to find the riffle independent approximation
which is \emph{closest} to $h$ in some sense.  Formally, we
would like to solve the problem:
\begin{align}\label{eqn:mainproblem}
\arg\min_A \min_{m,f,g}&\;\; D_{KL}(\hat{h}(\sigma) \,||\, 
m(\tau_{A,B}(\sigma)) f(\phi_A(\sigma)) g(\phi_B(\sigma))
),
\end{align}
where $\hat{h}$ is the empirical distribution of
training examples and $D_{KL}$ is the Kullback-Leibler divergence measure.
Equation~\ref{eqn:mainproblem} is a seemingly reasonable
objective since it can also be interpreted as maximizing
the likelihood of the training data.
In the limit of infinite data, 
Equation~\ref{eqn:mainproblem} can be shown via
the Gibbs inequality to attain its minimum, zero, at
the subsets $A$ and $B$, if and only if the sets
$A$ and $B$ are truly riffle independent of each other.

For small problems, 
one can actually solve Problem~\ref{eqn:mainproblem} using
a single computer by evaluating the approximation
quality of each subset $A$ and taking the minimum, which
was the approach taken in Example~\ref{ex:apahierarchy}. 
However, for larger problems, one runs into time and sample
complexity problems
since optimizing the globally defined objective function
(Equation~\ref{eqn:mainproblem}) requires relearning all model parameters
($m$, $f_A$, and $g_B$)
for each of the exponentially many
subsets of $\{1,\dots,n\}$.
In fact, for large sets $A$ and $B$, it is rare that one would have enough
samples to estimate the relative ranking parameters $f_A$ and $g_B$ without
already having discovered the hierarchical 
riffle independent decompositions of $A$ and $B$.
We next propose a more locally defined
objective function, reminiscent of clustering, which we will
use instead of Equation~\ref{eqn:mainproblem}.
As we show, our new objective will be more tractable to compute and
have lower sample complexity for estimation.

\subsection{Proposed objective function}
The approach we take is to
minimize a different measure that exploits
the observation that
\emph{absolute ranks of items in $A$ are fully
independent of relative ranks of items in $B$, and vice versa} (which we
prove in Proposition~\ref{prop:micriterion}).
With our vegetables and fruits, for example, knowing that
Figs is ranked first among all six items (the absolute rank of a fruit)
should give no information about whether Corn is preferred
to Peas (the relative rank of vegetables).
More formally, given a subset $A=\{a_1,\dots,a_\ell\}$, recall that
$\sigma(A)$ denotes 
the vector of (absolute) ranks assigned to items in $A$ by $\sigma$
(thus, $\sigma(A) = (\sigma(a_1),\sigma(a_2),\dots,\sigma(a_\ell))$).
We propose to minimize an alternative objective function:
\begin{equation}\label{eqn:micriterion}
\mathcal{F}(A) \equiv I(\sigma(A)\;;\;\phi_B(\sigma)) 
        +I(\sigma(B)\;;\;\phi_A(\sigma)),
\end{equation}
where $I$ denotes the mutual information (defined between
two variables $X_1$ and $X_2$
by $I(X_1;X_2)\equiv D_{KL}(P(X_1,X_2) || P(X_1)P(X_2))$.

The function $\mathcal{F}$ does not have the same
likelihood interpretation
as the objective function of Equation~\ref{eqn:mainproblem}.
However, it can be thought of as a composite likelihood
of two models, one in which the relative rankings of $A$ are independent
of absolute rankings of $B$, and one in which the relative rankings of $B$
are independent of absolute rankings of $A$ 
(see Appendix~\ref{sec:likelihood}).
With respect to distributions which satisfy (or approximately satisfy)
both models (i.e., the riffle independent distributions),
minimizing $\mathcal{F}$ \emph{is} equivalent to 
(or approximately equivalent to) maximizing the log likelihood of the data.
Furthermore, we can show that $\mathcal{F}$ is guaranteed
to detect riffled independence:
\begin{proposition}\label{prop:micriterion}
$\mathcal{F}(A)=0$ is
a necessary and sufficient criterion for a subset $A\subset\{1,\dots,n\}$ to be riffle independent
of its complement, $B$.
\end{proposition}
\begin{proof}
Suppose $A$ and $B$ are riffle independent.
We first claim that $\sigma(A)$ and $\phi_B(\sigma)$ are independent.
To see this, observe that the absolute ranks of $A$, $\sigma(A)$, are
determined by the relative rankings of $A$, $\phi_A(\sigma)$
and the interleaving
$\tau_{A,B}(\sigma)$.
By the assumption that $A$ and $B$ are riffle independent,
we know that the relative rankings of $A$
and $B$ ($\phi_A(\sigma)$ and $\phi_B(\sigma)$), and the interleaving
$\tau_{A,B}(\sigma)$ are independent, establishing the claim.
The argument that $\sigma(B)$ and $\phi_A(\sigma)$ are independent is
similar, thus establishing one direction of the proposition.

To establish the reverse direction, assume that
Equation~\ref{eqn:micriterion} evaluates to zero on sets $A$ and $B$.
It follows that $\sigma(A)\perp \phi_B(\sigma)$ and $\phi_A(\sigma)\perp\sigma(B)$.
Now, as a converse to the observation from above, note that the absolute ranks of $A$
\emph{determine} the relative ranks of $A$, $\phi_A(\sigma)$, as well as the interleaving
$\tau_{A,B}(\sigma)$.  Similarly, $\sigma(B)$ determines $\phi_B(\sigma)$ and $\tau_{A,B}(\sigma)$.
Thus, $\left(\phi_A(\sigma),\tau_{A,B}(\sigma)\right)\perp \phi_B(\sigma)$
and $\phi_A(\sigma)\perp \left(\tau_{A,B}(\sigma),\phi_B(\sigma)\right)$.
It then follows that
$\phi_A(\sigma)\perp \tau_{A,B}(\sigma)\perp \phi_B(\sigma)$.
\end{proof}
%
As with Equation~\ref{eqn:mainproblem},
optimizing $\mathcal{F}$ is still
intractable for large $n$. However, $\mathcal{F}$ motivates
a natural proxy, in which we replace the mutual informations
defined over all $n$ variables by a sum of mutual informations defined
over just three variables at a time.
\begin{definition}[Tripletwise mutual informations]
Given any triplet of distinct items, $(i,j,k)$,
we define the tripletwise mutual information term, 
$I_{i;j,k}\equiv I(\sigma(i)\,;\,\sigma(j)<\sigma(k))$.
\end{definition}
The tripletwise mutual information $I_{i;j,k}$ can be 
computed as follows:\vspace{-3mm}

{\footnotesize
\begin{equation*}
I(\sigma(i)\,;\,\sigma(j)<\sigma(k))
	= \sum_{\sigma(i)} \sum_{\sigma(j)<\sigma(k)} h(\sigma(i),\sigma(j)<\sigma(k))\log \frac{h(\sigma(i),\sigma(j)<\sigma(k))}{h(\sigma(i))h(\sigma(j)<\sigma(k))},
\end{equation*}\vspace{-3mm}
}

\noindent where the inside summation runs over two values, true/false, for the binary variable $\sigma(j)<\sigma(k)$.
To evaluate how riffle independent two subsets $A$ and $B$ are,
we want to examine the triplets that straddle the two sets.
\begin{definition}[Internal and Cross triplets]
We define $\ABedge$ to be the set of triplets which ``cross''
from set $A$ to set $B$:
$
\ABedge \equiv \{(i;j,k)\,:\,i\in A, j,k\in B \}.
$
$\BAedge$ is similarly defined.
We also define $\Aedge$ to be the set of triplets that are internal to $A$:
$
\Omega^{int}_A \equiv \{(i;j,k)\,:\,i,j,k\in A\},
$
and again, $\Bedge$
is similarly defined.
\end{definition}
Our proxy objective function can be written as the sum of the
mutual information evaluated over all of the crossing triplets:\vspace{-3mm}

{\footnotesize
\begin{equation}\label{eqn:floworder}
\FF(A) \equiv \sum_{(i,j,k)\in \ABedge} I_{i;j,k} 
        + \sum_{(i,j,k)\in \BAedge} I_{i;j,k}.
\end{equation}\vspace{-3mm}
}

\begin{figure*}[t!]
\begin{center}
\subfigure[]{
\raisebox{20pt}{
\includegraphics[width=.4\textwidth]{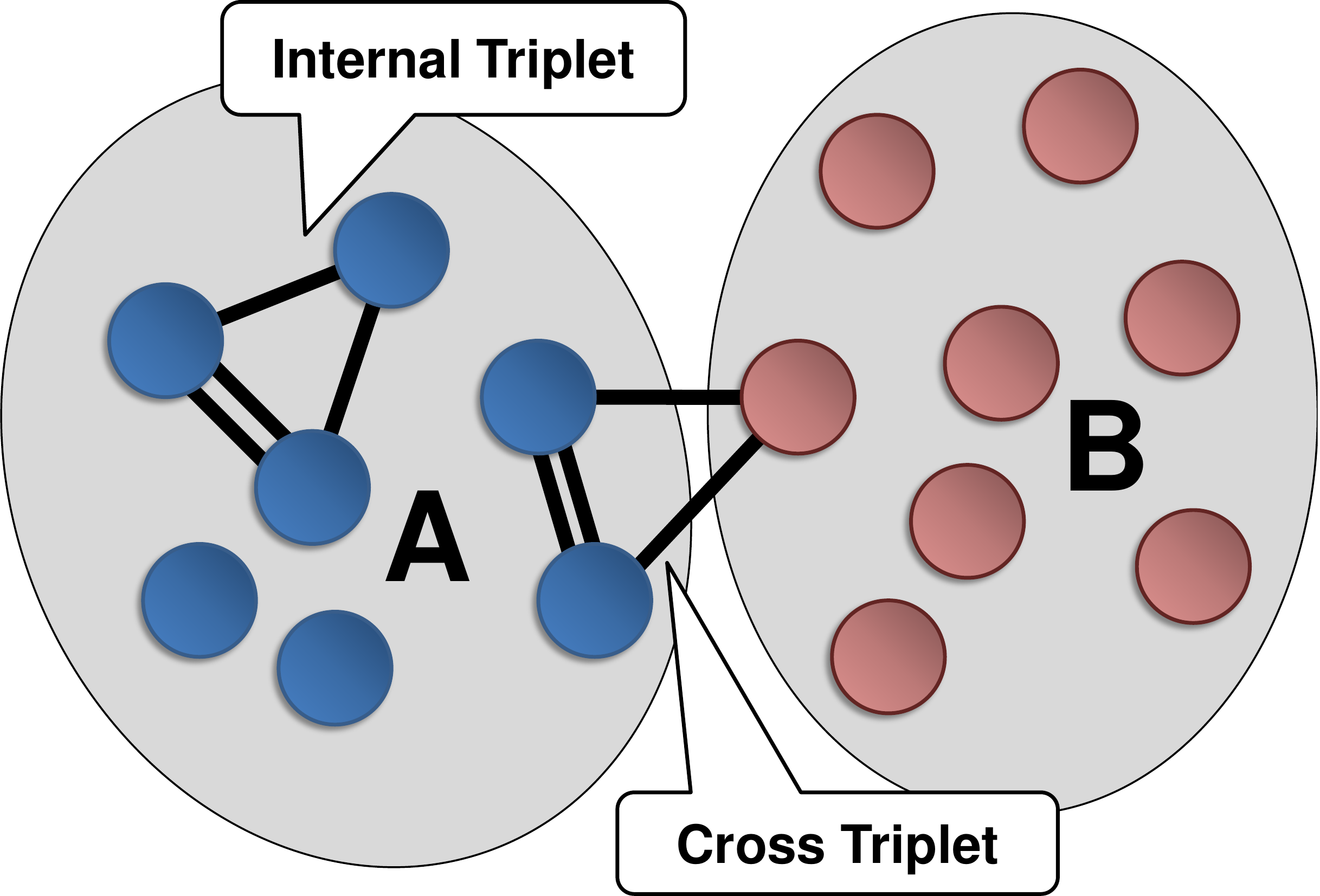}
\label{fig:diagram}
}
}\qquad\qquad
\subfigure[]{
\includegraphics[width=.4\textwidth]{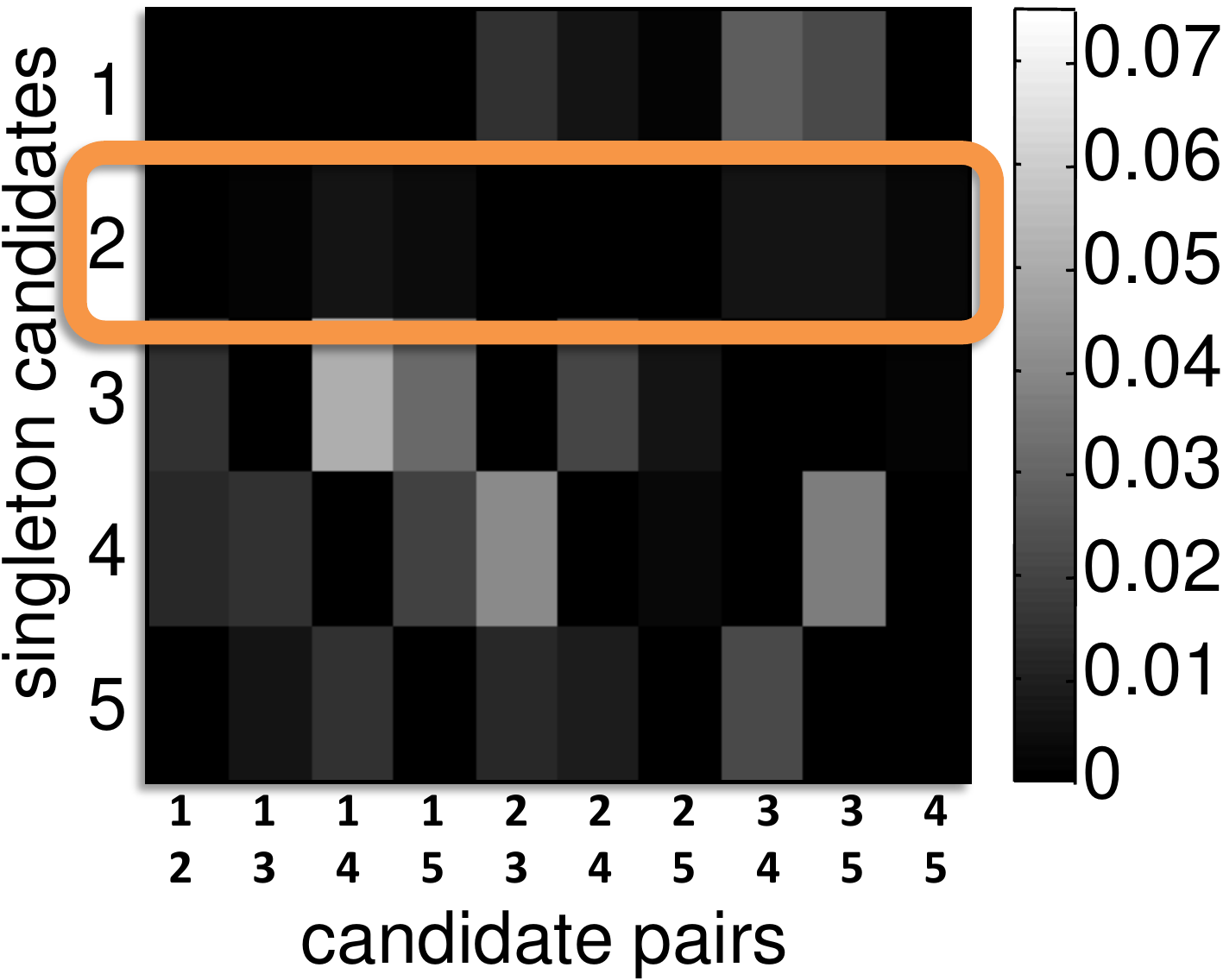}
\label{fig:apamutualinfo}
}
\caption{Examples: \subref{fig:diagram} shows a 
graphical depiction of the problem of finding
riffle independent subsets.  A triangle with vertices $(i,j,k)$
represents the term $I_{i;j,k}$
Since the $I_{i;j,k}$ are not invariant with respect to a permutation of 
the indices $i$, $j$, and $k$, the triangles are directed, and we therefore use
double bars represent the nodes $j,k$ for the term $I_{i;j,k}$.
Note that if the tripletwise terms were instead replaced by edgewise terms,
the problem would simply be a standard clustering problem;
\subref{fig:apamutualinfo} shows the matrix of tripletwise mutual informations computed from the APA dataset (see Example~\ref{ex:apami}).}
\end{center}
\end{figure*}
$\FF$ can be viewed as a low order version of $\mathcal{F}$,
involving mutual information computations over triplets
of variables at a time instead of $n$-tuples.  The mutual
information $I_{i;j,k}$, for
example, reflects how much the rank of a vegetable ($i$)
tells us about how two fruits ($j$, $k$) compare.
If $A$ and $B$ are riffle independent, then we know that $I_{i;j,k}=0$
for any $(i,j,k)$ such that $i\in A$, $j,k\in B$ (and similarly
for any $(i,j,k))$ such that $i \in B$, $j,k\in A$.
Given that fruits and vegetables are riffle independent sets, 
knowing that Grapes is preferred to Figs should give no
information about the absolute rank of Corn,
and therefore $I_{Corn; Grapes, Figs}$ should be zero.
Note that such tripletwise independence assertions bear resemblance to
assumptions sometimes made in social choice theory,
commonly referred to as \emph{Independence of Irrelevant Alternatives}
~\citep{arrow63},
where the addition of a third element $i$, is assumed to not affect 
whether one prefers an element $j$ over $k$.

The objective $\FF$ is somewhat
reminiscent of typical graphcut
and clustering objectives. Instead of partitioning a set of nodes
based on sums of pairwise similarities, we partition based on sums
of tripletwise affinities.  We show a graphical depiction of the problem
in Figure~\ref{fig:diagram}, where cross triplets (in $\ABedge$, $\BAedge$)
have low weight and internal triplets
(in $\Aedge$, $\Bedge$) have high weight.
The objective is to find a partition such that the sum over cross triplets
is low.
In fact, the problem of
optimizing $\FF$ can be seen as an instance of the weighted,
directed hypergraph cut problem~\citep{galloetal93}.
Note that the word \emph{directed} is significant for us, because,
unlike typical clustering problems, our triplets
are not symmetric
(for example, $I_{i;jk}\neq I_{j;ik}$), resulting in a nonstandard
and poorly understood optimization problem.
\begin{example}[APA election data (continued)]\label{ex:apami}
Figure~\ref{fig:diagram} visualizes the tripletwise mutual informations
computed from the APA dataset.
Since there are five candidates, there are ${5\choose 2}=10$ pairs
of candidates.  The $(i,(j,k))$ entry in the matrix corresponds to
$I(\sigma(i);\sigma(j)<\sigma(k))$.  For easier visualization, we have 
set entries of the form $(i,(i,k))$ and $(i,(j,i))$ to be zero since
they are not counted in the objective function.

The highlighted row corresponds to candidate 2, in which all of the 
mutual information terms are close to zero.  
We see that the tripletwise mutual information terms tell a story
consistent with the conclusion of Example~\ref{ex:removecandidate2}, 
in which we showed that candidate 2 was
approximately riffle independent of the remaining candidates.

Finally, it is also interesting to examine the $(3,(1,4))$ entry.
It is the largest mutual information in the matrix, a fact which should 
not be surprising since candidates 1 and 3 are politically aligned
(both research psychologists).  Thus, knowing, for example, that 
candidate 3 was ranked first is a strong indication that candidate
1 was preferred over candidate 4.
\end{example}

\subsection{Encouraging balanced partitions}
In practice, like the minimum cut
objective for graphs, the tripletwise objective
of Equation~\ref{eqn:floworder} has a tendency to ``prefer''
small partitions (either $|A|$ or $|B|$ very small) to more balanced partitions
($|A|,|B|\approx n/2$) due to the fact that unbalanced partitions
have fewer triplets that cross between $A$ and $B$.
The simplest way to avoid this bias is to optimize the objective
function over subsets of a fixed size $k$.  As we discuss in the
next section, optimizing with a fixed $k$
can be useful for building thin hierarchical riffle independent models.
Alternatively, one can use a modified objective function that encourages
more balanced partitions.
For example, we have found the following \emph{normalized
cut}~\citep{shi00} inspired
variation of our objective to be useful for detecting riffled independence when the
size $k$ is unknown:\vspace{-3mm}

{\footnotesize
\begin{align}
\mathcal{F}^{balanced}(A) \equiv &
\frac{\sum_{\ABedge} I_{i;j,k}}{\sum_{\ABedge} I_{i;j,k}+\sum_{\Aedge} I_{i;j,k}}  +\frac{\sum_{\BAedge} I_{i;j,k}}{\sum_{\BAedge} I_{i;j,k}+\sum_{\Bedge} I_{i;j,k}}
.\label{eqn:balanced}
\end{align}\vspace{-3mm}
}

\noindent Intuitively, the denominator in Equation~\ref{eqn:balanced}
penalizes subsets whose interiors have small weight.
Note that there exist many variations on the objective function
that encourage balance, but $\mathcal{F}^{balanced}$ is the
one that we have used in our experiments.

\subsection{Low-order detectability assumptions.}
When does $\FF$ detect riffled independence?
It is not difficult to see, for example, that $\FF=0$
is a necessary condition for riffled independence, since
$A\perp_m B$ implies $I_{a;b,b'}=0$.  We have:
\begin{proposition}\label{prop:ffnecessity}
If $A$ and $B$ are riffle independent sets, then $\FF(A)=0$.
\end{proposition}
However, the converse of
Proposition~\ref{prop:ffnecessity} is not true in full generality
without accounting
for dependencies that involve larger subsets of variables.
Just as the pairwise independence assumptions
that are commonly used for randomized algorithms~\citep{motwani96}\footnote{
A pairwise independent family of random variables
is one in which any two members are marginally independent.
Subsets with larger than two members may not necessarily
factor independently, however.
}
do not imply full independence between two sets of variables,
there exist distributions which ``look'' riffle
independent from tripletwise marginals but do not
factor upon examining higher-order terms.
Nonetheless, in most practical scenarios, we expect $\FF=0$ to imply
riffled independence.

\subsection{Quadrupletwise objective functions for riffled independence}
A natural variation of our method
is to base the objective function on 
the following quantities, defined over quadruplets of items
instead of triplets:
\begin{equation}\label{eqn:quad}
I_{ij;kl} \equiv I(\sigma(i)<\sigma(j)\,;\,\sigma(k)<\sigma(\ell)).
\end{equation}
Intuitively, $I_{ij;kl}$ measures how much knowing that, say, Peas is preferred
to Corn, tells us about whether Grapes are preferred to Oranges.
Again, if the fruits and vegetables are riffle independent, then the mutual
information should be zero.
Summing over terms which cross between the cut, we obtain a quadrupletwise objective function defined as:
$
\mathcal{F}^{quad}(A)\equiv \sum_{(i,j)\in A,(k,\ell)\in B} I_{ij;kl}.
$
If $A$ and $B$ are riffle independent with 
$i,j\in A$ and $k,\ell\in B$, then the mutual information $I_{ij;kl}$
is zero.  Unlike their tripletwise counterparts, however,
the $I_{ij,kl}$ do not arise from a global measure that is 
both necessary and sufficient for detecting riffled independence.  In
particular, $I(\phi_A(\sigma);\phi_B(\sigma))=0$ is insufficient
to guarantee riffled independence.  For example, if the interleaving
depends on the relative rankings of $A$ and $B$, then riffled independence
is not satisfied, yet $\mathcal{F}^{quad}(A)=0$.
Moreover, it is not clear how one would detect riffle independent
subsets consisting of a single element using a
quadrupletwise measure.
As such, we have focused on tripletwise measures in our experiments.
Nonetheless, quadrupletwise measures may potentially be useful 
in practice (for detecting larger subsets) and
have the significant advantage that the $I_{ij;kl}$ can be estimated
with fewer samples and using almost any imaginable form of 
partially ranked data.

\subsection{Estimating the objective from samples}
We have so far argued that $\FF$ is a reasonable function
for finding riffle independent subsets.
However, since we only have access to samples rather than the true
distribution $h$ itself, it will only be possible to compute an
approximation to the objective $\FF$.
In particular, for every triplet
of items, $(i,j,k)$, we must compute an estimate of the mutual information $I_{i;j,k}$
from i.i.d. samples drawn from $h$, and the main question is:
how many samples will we need in order for the approximate version of
 $\FF$ to remain a reasonable objective function?

In the following, we denote the estimated value of $I_{i;j,k}$ 
by $\hat{I}_{i;j,k}$.
For each triplet, we use a regularized procedure due to~\cite{hoffgen93}
to estimate mutual information.  We adapt his sample complexity bound
to our problem below.
\begin{lemma}
\looseness -1 For any fixed triplet $(i,j,k)$, the mutual information
$I_{i;j,k}$ can be
estimated to within an accuracy of $\Delta$ with probability at least
$1-\gamma$ using
$S(\Delta,\gamma)\equiv 
O\left(\frac{n^2}{\Delta^2}\log^2\frac{n}{\Delta}\log\frac{n}{\gamma}\right)$
i.i.d. samples and the same amount of time.
\end{lemma}
The approximate objective function is therefore:\vspace{-2mm}

{\footnotesize
\[
\FFF(A) \equiv \sum_{(i,j,k)\in \ABedge} 
        \hat{I}_{i;j,k} + \sum_{(i,j,k)\in\BAedge} \hat{I}_{i;j,k}.
\]\vspace{-3mm}
}

What we want to now show is that, if there exists
a unique way to partition $\{1,\dots,n\}$ into riffle independent
sets, then given enough training examples,
our approximation $\FFF$ uniquely
singles out the correct partition as its minimum with high probability.
A class of riffle independent distributions for which the
uniqueness requirement is satisfied consists of the
distributions for which
$A$ and $B$ are \emph{strongly connected}
according to the following definition.
\begin{definition}\label{def:strongconnectivity}
A subset $A\subset\{1,\dots,n\}$ is called
\emph{$\epsilon$-third-order strongly connected}
if, for every triplet $i,j,k\in A$
with $i,j,k$ distinct, we have $I_{i;j,k}>\epsilon$.
\end{definition}
If a set $A$ is riffle independent of $B$ and both sets are
third order strongly connected, then we can ensure that
riffled independence is detectable from third-order terms
and that the partition is unique.  We have the following
probabilistic guarantee.
\begin{theorem}\label{thm:samplecomplexity}
\looseness -1 Let $A$ and $B$ be $\epsilon$-third order strongly connected
riffle independent sets, and suppose $|A|=k$.
Given $S(\Delta,\epsilon)\equiv O\left(\frac{n^4}{\epsilon^2}\log^2\frac{n}{\epsilon}\log\frac{n}{\gamma}\right)$ i.i.d. samples,
the minimum of $\FFF$ is achieved at exactly the subsets $A$ and $B$
with probability at least $1-\gamma$.
\end{theorem}
See the Appendix for details.
Finally, we remark that the strong connectivity assumptions used in 
Theorem~\ref{thm:samplecomplexity} are stronger than necessary ---
and with respect to certain interleaving distributions, 
it can even be the case that the estimated objective function singles
out the correct partition when all of internal triplets belonging to $A$ and $B$
have zero mutual information.
Moreover, in some cases, there are multiple valid partitionings of the item set.
For example the uniform distribution is a distribution in which every
subset $A\subset\{1,\dots,n\}$ is riffle independent of its complement.
In such cases, multiple solutions are equally good when evaluated
under $\FF$, but not its sample approximation, $\FFF$.

\section{Structure discovery II: algorithms}\label{sec:structure2}
Having now designed a function that is tractable to estimate
from both perspectives of computational and sample
complexity, we turn to the problem of learning the
hierarchical riffle independence structure of a distribution
from training examples.  
Instead of directly optimizing an objective in the space of possible
hierarchies, we take a simple top-down approach in which the item sets are
recursively partitioned by optimizing $\FFF$
until some stopping criterion is met (for example, when the
leaf sets are smaller than some $k$, or simply stopping after
a fixed number of splits).

\subsection{Exhaustive optimization}
Optimizing the function $\FFF$ requires searching through
the collection of subsets of size $|A|=k$, which, when
performed exhaustively, requires $O\left({n\choose k}\right)$ time.
An exhaustive approach thus runs in exponential time, for example,
when $k\sim O(n)$.  

However, when the size of $k$ is known and small ($k\sim O(1)$), the optimal partitioning
of an item set can be found in polynomial time by exhaustively evaluating $\FFF$
over all $k$-subsets. 
\begin{corollary}
Under the conditions of Theorem~\ref{thm:samplecomplexity},
one needs at most $S(\Delta,\epsilon)\equiv 
O\left(\frac{n^2}{\epsilon^2}\log^2\frac{n}{\epsilon}\log\frac{n}{\gamma}\right)$
samples to recover the exact riffle independent partitioning
with probability $1-\gamma$.
\end{corollary}

When $k$ is small, we can therefore use exhaustive optimization 
to learn the structure of  $k$-thin chain models (Section~\ref{sec:thinchains})
in polynomial time.
The structure learning problem for thin chains is to
discover how the items are partitioned into groups,
which group is inserted first, which group is inserted second, and so
on.  To learn the structure of a thin chain, we can use exhaustive
optimization to learn the topmost partitioning of the item set, then recursively
learn a thin chain model for the items in the larger subset.

\begin{figure}[t!]
\incmargin{1em}
\begin{algorithm2e}[H]
{\scriptsize
\SetKwInOut{Input}{input}\SetKwInOut{Output}{output}
{\sc AnchorsPartition} \xspace \\
\Input{training set $\{\sigma^{(1)},\dots,\sigma^{(m)}\}$, $k\equiv|A|$}
\Output{Riffle independent partitioning of item set, $(A_{best},B_{best})$}
\BlankLine
Fix $a_1$ to be any item \;
 \ForAll{$a_2\in \{1,\dots,n\},\;a_1\neq a_2$}{
	Estimate $\hat{I}_{x;a_1,a_2}$ for all $x\neq a_1,a_2$\;
	$\hat{I}^k \leftarrow$ $k^{th}$ smallest item in $\{\hat{I}_{x;a_1,a_2};x\neq a_1,a_2\}$ \;
	$A_{a_1,a_2}\leftarrow \{x\,:\,\hat{I}_{x;a_1,a_2}\leq \hat{I}^k \}$ \;
	}
 $A_{best} \leftarrow \arg\min_{a_1,a_2} \FFF(A_{a_1,a_2})$\;
$B_{best} \leftarrow \{1,\dots,n\} \backslash A_{best}$ \;
 \Return $[A_{best},B_{best}]$\;
\caption{\scriptsize Pseudocode for partitioning using the \emph{Anchors} method}
\label{alg:anchors}
}
\end{algorithm2e}
\decmargin{1em}
\end{figure}

\subsection{Handling arbitrary partitions using anchors}
When $k$ is large, or even unknown, $\FFF$ cannot be
optimized using exhaustive methods.
Instead, we propose a simple algorithm for finding
$A$ and $B$ based on the following observation.
If an oracle could identify any two
elements of the set $A$, say, $a_1,a_2$, in
advance, then the quantity $I_{x;a_1,a_2}=I(x;a_1<a_2)$ indicates
whether the item $x$ belongs to $A$ or $B$ since $I_{x;a_1,a_2}$ is nonzero
in the first case, and zero in the second case.

For finite training sets, when $I$ is only known approximately, one can
sort the set $\{I_{x;a_1,a_2}\,;\,x\neq a_1,a_2\}$ and if $k$
is known, take the $k$
items closest to zero to be the set $B$ (when $k$ is unknown, one can
use a threshold to infer $k$).
Since we compare all items against $a_1,a_2$, we refer to
these two fixed items as ``anchors''.  

Of course $a_1,a_2$ are not known in advance, but by fixing $a_1$ to be
an arbitrary item, one can repeat the above method for all
$n-1$ settings of $a_2$ to produce 
a collection of $O(n^2)$ candidate partitions.  
Each partition
can then be scored using the approximate objective $\FFF$,
and a final optimal partition can be selected as the minimum over the
candidates. See Algorithm~\ref{alg:anchors}.
In cases when $k$ is not known a priori, we evaluate partitions
for all possible settings of $k$ using $\FFF$.

Since the Anchors method does not require searching over subsets, it
can be significantly faster 
than an exhaustive optimization of $\FFF$.
Moreover, by assuming $\epsilon$-third order strong connectivity
as in the previous section, one can use similar arguments to derive
sample complexity bounds.
\begin{corollary}[of Theorem~\ref{thm:samplecomplexity}]\label{cor:anchors}
Let $A$ and $B$ be $\epsilon$-third order strongly connected
riffle independent sets, and suppose $|A|=k$.
Given $S(\Delta,\epsilon)$ i.i.d. samples, the output of the Anchors algorithm
is exactly $[A,B]$ with probability $1-\gamma$.  
In particular, the Anchors estimator is  consistent.
\end{corollary}
We remark, however, that there are practical differences that can at times
make the Anchors method somewhat less robust than an exhaustive search.
Conceptually, anchoring works well when there exists two elements
that are strongly
connected with all of the other elements in its set, which can then be 
used as the anchor elements $a_1, a_2$.  
An exhaustive search can work well in weaker conditions
such as when items are strongly connected through longer
paths.
We show in our experiments that the Anchors method can nonetheless
be quite effective for learning hierarchies.

\subsection{Running time}
We now consider the running time of our structure learning 
procedures.  In both cases,
it is necessary to precompute the mutual information quantities
$I_{i;j,k}$ for all triplets $i,j,k$ from $m$ samples.
For each triplet, we can compute $I_{i;j,k}$ in linear time
with respect to the sample size.
The set of all triplets can therefore be computed in $O(mn^3)$ time.

The exhaustive method for finding the $k$-subset
which minimizes $\FFF$ requires evaluating the objective
function at ${n\choose k}=O(n^k)$ subsets.  What is the complexity
of evaluating $\FFF$ at a particular partition $A,B$?
We need to sum the precomputed mutual informations
over the number of triangles that cross between $A$ and $B$.
If $|A|=k$ and $|B|=n-k$, then we can bound the number of such
triangles by $k(n-k)^2+k^2(n-k) = O(kn^2)$.
Thus, we require $O(n^k+kn^2)$ optimization time, leading
to a bound of $O(kn^{k+2}+mn^3)$ total time.

The Anchors method requires us to (again) precompute mutual informations.
The other seeming bottleneck is the last step, in which we must
evaluate the objective function $\FFF$ at $O(n^2)$ partitions.
In reality, if $|A|$ and $|B|$ are both larger than 1, then
$a_1$ can be held fixed at any arbitrary element, and we must only
optimize over $O(n)$ partitions.
When $|A|=|B|=1$, then $n=2$, in which case the two sets are
trivially riffle independent (independent of the actual
distribution).
As we showed in the previous
paragraph, evaluating $\FFF$ requires $O(kn^2)$ time, and thus
optimization using the Anchors method = $O(n^3(k+m))$ total time.
Since $k$ is much smaller than $m$ (in any meaningful training set),
we can drop it from the big-O notation to get $O(mn^3)$
time complexity, showing that the Anchors method is dominated
by the time that is required to precompute and cache mutual informations.

\section{Structure discovery III: quantifying stability}\label{sec:structure3}
Given a hierarchy estimated from data, we now discuss 
how one might practically quantify how confident we should be
about the hypothesized structure.  We might like to know if the 
amount of data that was used for estimating the structure
was adequate the support the learned structure, and,
if the the data looked slightly different, would the hypothesis change?

Bootstrapping~\citep{efron93} offers a simple approach ---
repeatedly resample the data with replacement, and estimate
a hierarchical structure for each resampling.
The difference between our setting and typical bootstrapping
settings, however, is that our structures lie in a large 
discrete set. Thus, unlike continuous parameters,
whose confidence we can often summarize with intervals or ellipses,
it is not clear how one might compactly summarize
a collection of many hierarchical clusterings of items.

The simplest way to summarize the collection of 
hierarchies obtained via the bootstrap is to 
measure the fraction of the estimated structures which 
are identical to the structure estimated from the original unperturbed dataset.
If, for small sets of resampled data, the estimated hierarchy is
consistently identical to that obtained from the original
data, then we can be confident that the data supports the hypothesis.
We show in the following example that, for the structure which was learned
from the APA dataset, a far smaller dataset would have sufficed.
\begin{figure*}[t!]
\begin{center}
\subfigure[]{
\includegraphics[width=.35\textwidth]{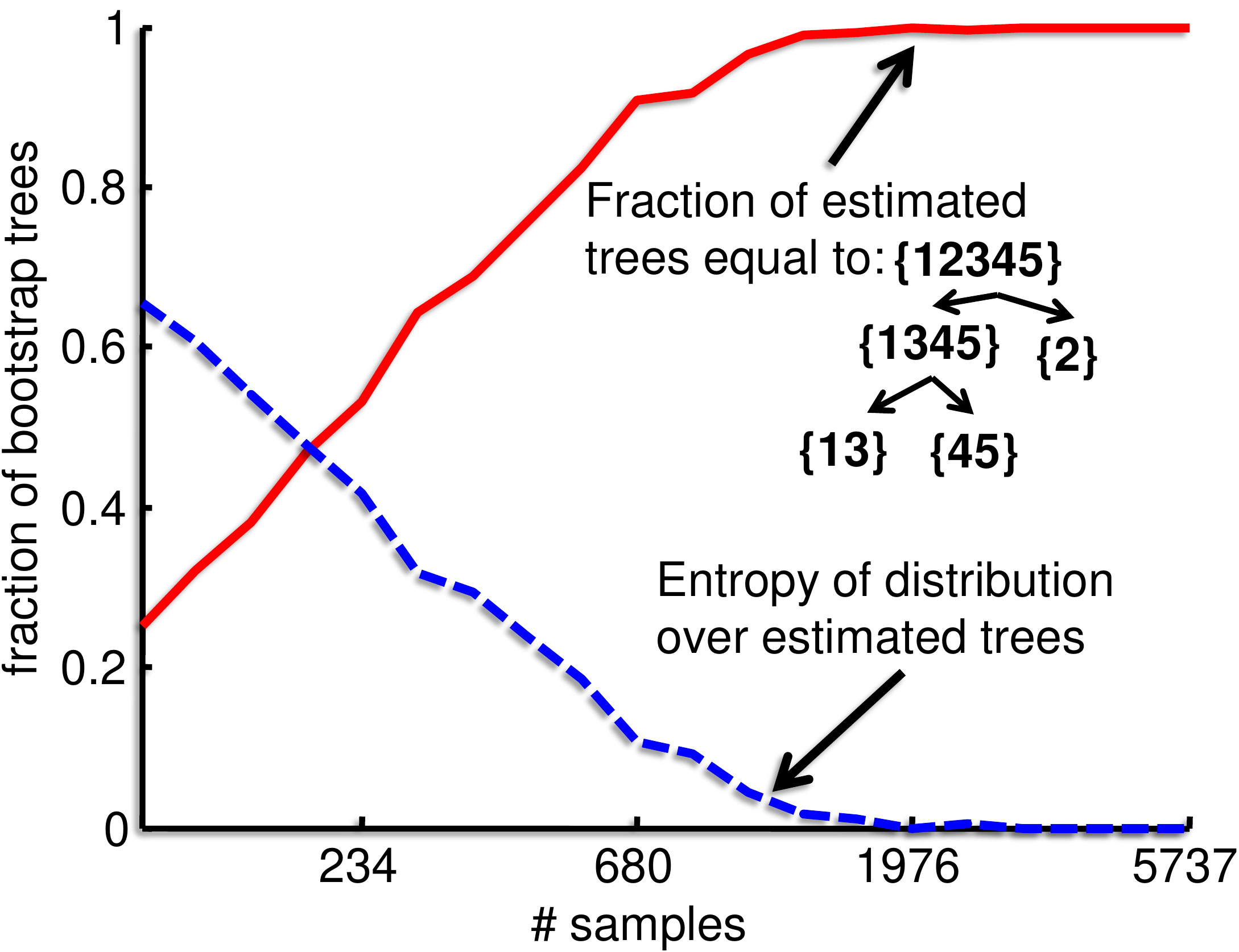}
\label{fig:apabootstrap1}
}\qquad\qquad
\subfigure[]{
\includegraphics[width=.50\textwidth]{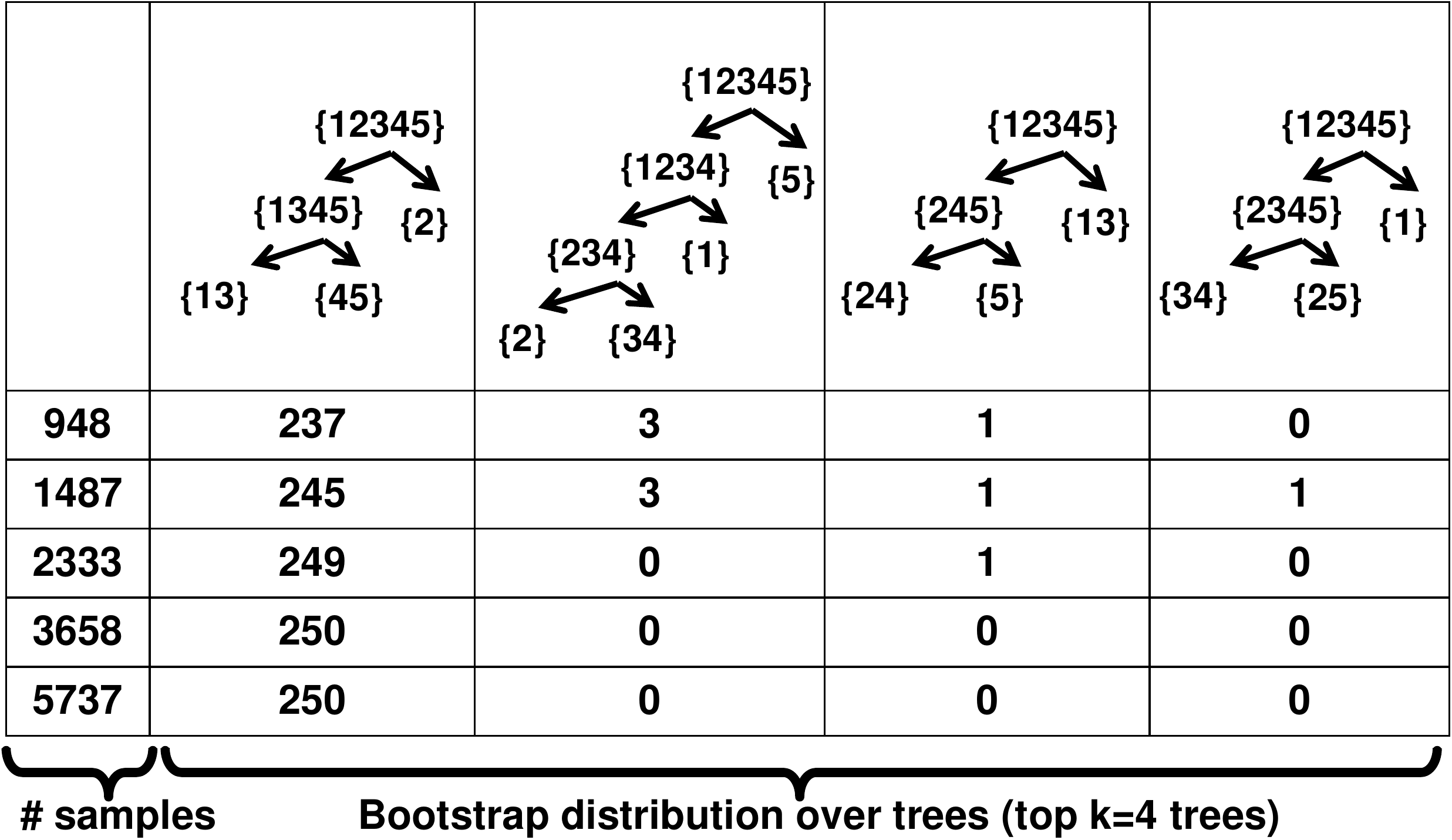}
\label{fig:apabootstrap2}
}
\caption{We show the distribution of structures estimated
from bootstrapped samples of the APA data (with varying sample sizes):
\subref{fig:apabootstrap1} plots (in solid red)
the fraction of bootstrapped trees for each sample size
which agree exactly with the hierarchy given in Figure~\ref{fig:apatree};
In \subref{fig:apabootstrap2}, we summarize the boostrap distribution
for the largest sample sizes.}
\label{fig:apabootstrap}
\end{center}
\end{figure*}
\begin{example}[APA Election data (continued)]
As our final APA related example, we show the results of 
bootstrap resampling in Figure~\ref{fig:apabootstrap}.
To generate the plots, we resampled the APA dataset with replacement 
200 times each for varying sample sizes, and ran our Anchors
algorithm on each resulting sample.
Figure~\ref{fig:apabootstrap1} plots (in solid red)
the fraction of bootstrapped trees for each sample size
which agree exactly with the hierarchy given in Figure~\ref{fig:apatree}.
Given that we forced sets to be partitioned until they had at most 2 items, 
there are 120 possible hierarchical structures for the APA dataset.

It is interesting to see that the hierarchies returned
by the algorithm are surprisingly stable even given fewer than 100 samples,
with about 25\% of bootstrapped trees agreeing with the optimal
hierarchy.  At 1000 samples, almost all trees agree with the 
optimal hierarchy.
In Figure~\ref{fig:apabootstrap2}, we show a table of the 
bootstrap distribution for the largest sample sizes
(which were concentrated at only a handful of trees).
\end{example}
For larger item sets $n$, however, it is rarely the case that there is
enough data to strongly support the hierarchy in terms of the above
measure.  In these cases, instead of asking whether entire structures
agree with each other \emph{exactly}, it makes sense to ask whether
estimated \emph{substructures} agree.
For example, a simple measure might amount to computing the fraction
of structures estimated from resampled datasets
which agreed with the original structure at the topmost partition.
Another natural measure is to count the fraction of structures
which correctly recovered all (or a subset of) 
leaf sets for the original dataset, but not necessarily the correct
hierarchy.  By Proposition~\ref{prop:leafsets},
correctly discovering the leaf set partitioning is probabilistically 
meaningful, and corresponds to correctly identifying
the $d$-way decomposition corresponding to a distribution, but 
failing to identifying the specific hierarchy.

We remark that sometimes, there is no one unique
structure corresponding to a distribution.   The uniform 
distribution, for example, is consistent with any hierarchical
riffle independent structure, and so bootstrapped hierarchies
will not concentrate on any particular structure or even substructure.
Moreover, even when there \emph{is} true unique structure
corresponding to the generating distribution, it may be the 
case that other simpler structures perform better when there is not
much available training data.

\section{Related work}\label{sec:relatedwork}

%

Our work draws from several literatures: card shuffling
research due primarily to Persi Diaconis and collaborators
\citep{diaconis92,fulman98}, papers about Fourier theoretic
probabilistic inference over permutations from the machine learning
community\citep{kondor08b,huangetal09b,huangetal09a}, as well
as graphical model structure learning research.

\subsection{Card shuffling theory}
\cite{diaconis92} provided a
a convergence analysis of repeated riffle shuffles.
Our novelty lies in the combination
of shuffling theory with independence, which was first exploited in
\cite{huangetal09a}, for scaling inference operations
to large problems.  
Finally, we remark that \cite{fulman98} introduced
a class of shuffles known as biased riffle shuffles which
are not the same as the biased riffle shuffles discussed in our paper.
The fact that the uniform riffle shuffling can be realized
by dropping card with probability proportional to the number of cards
remaining in each hand has been observed in a number of
papers~\citep{diaconis92}, but we are the first to (1) formalize 
this in the form
of the recurrence given in Equation~\ref{eqn:recurrence},
and (2) to compute the Fourier transform
of the uniform and biased riffle shuffling distributions.

\subsection{Fourier analysis on permutations}
Our dynamic programming
approach bears some similarities to the FFT (Fast Fourier
Transform) algorithm proposed by~\cite{clausen93}, and in particular,
relies on the same branching rule recursions~\citep{sagan01}.  While
the Clausen FFT requires $O(n!\log(n!))$ time, since our biased riffle shuffles
are parameterized by a single $\alpha$, we can use the recurrence
to compute low-frequency Fourier terms in polynomial time.

\subsection{Learning structured representations}
Our insights for the structure learning problems are inspired
by some of the recent approaches in the machine learning literature
for learning the structure of thin junction trees~\citep{bach01}.
In particular, the idea of using a low order proxy objective with a 
graph-cut like optimization algorithm is similar to an idea which 
was recently introduced in~\cite{shahaf09}, which determines optimally
thin separators with respect to the Bethe free energy approximation (of the entropy)
rather than a typical log-likelihood objective.
Our sample analysis is based on the mutual information sample
complexity bounds derived in~\cite{hoffgen93}, which was also used in 
~\cite{chechetka07} for developing a structure learning algorithm
for thin junction trees with provably polynomial sample complexity.
Finally, the bootstrap methods which we have employed in our experiments
for verifying robustness bear much resemblance to some of the
common bootstrapping methods which have been used in bioinformatics
for analyzing phylogenetic trees \citep{holmes99,holmes03}.

\section{Experiments}\label{sec:experiments}
In this section, we present a series of experiments
to validate our models and methods.
All experiments were implemented in Matlab, except for the Fourier
theoretic routines, which were written in C++.
We tested on lab machines with two AMD quadcore Opteron 2.7GHz processors
with 32 Gb memory.
We have already analyzed the APA data extensively throughout the
paper.  Here, we demonstrate our algorithms on simulated data 
as well as other real datasets, namely, sushi preference data, and  
Irish election data.

\subsection{Simulated data}
We begin with a discussion of our simulated data experiments.
We first consider approximation quality and timing issues
for a single binary partition of the item set.

\paragraph{Binary partitioning of the item set}
To understand the behavior of RiffleSplit in approximately riffle
independent situations,
we drew sample sets of varying sizes
from a riffle independent distribution on $S_8$ (with bias
parameter $\alpha=.25$) and use RiffleSplit
to estimate the relative ranking factors and interleaving 
distribution from the empirical distribution.
In Figure~\ref{fig:approxaccuracy}, we plot the KL-divergence between
the true distribution and that obtained by applying
RiffleJoin to the estimated riffle factors.  With small
sample sizes (far less than $8!=40320$), we are able to recover
accurate approximations despite the fact that the empirical distributions
are not exactly riffle independent.  For comparison, we ran
the experiment using the Split algorithm~\cite{huangetal09a}
to recover the parameters.
Perhaps surprisingly, one can show 
that the Split algorithm from ~\cite{huangetal09a} is
also an unbiased, consistent estimator of the riffle factors,
but it does not return the maximum likelihood parameter estimates
because it effectively ignores rankings which are not contained
in the subgroup $S_p\times S_q$.  Consequently, our RiffleSplit
algorithm converges to the correct parameters with far fewer samples.

Next, we show that our Fourier domain algorithms are capable
of handling sizeable item sets (with size $n$) when
working with low-order terms.
In Figure~\ref{fig:scalingfigure}
we ran our Fourier domain RiffleJoin algorithm
on various simulated distributions.
We plot running times
of RiffleJoin (without precomputing the interleaving distributions)
as a function of $n$ (setting $p=\lceil n/2\rceil$, which is
the worst case) scaling up to $n=40$.

\begin{figure*}[t!]
\begin{center}
\subfigure[Parameter learning from synthetic data]{
\includegraphics[width=.42\textwidth]{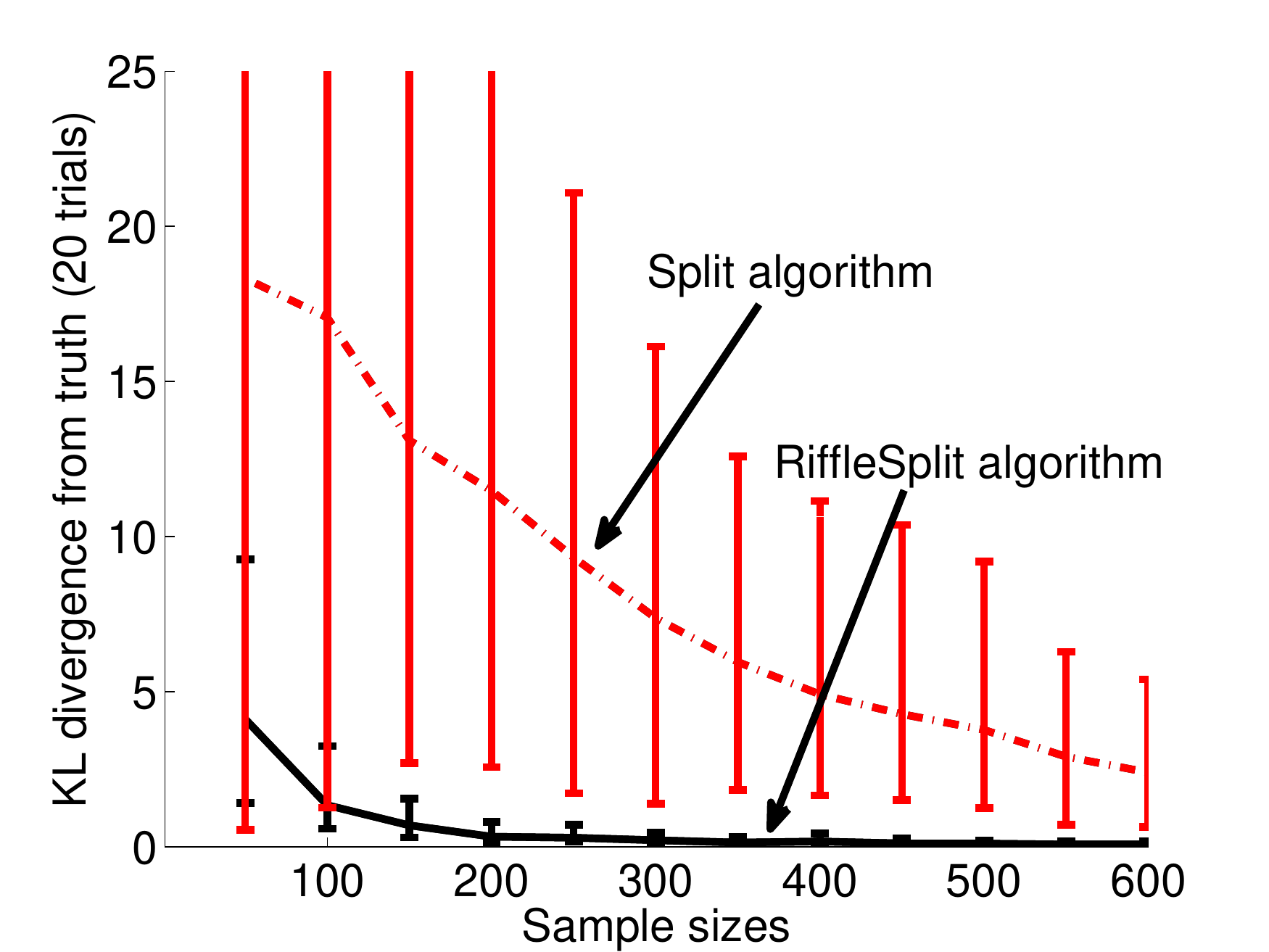}
   \label{fig:approxaccuracy}
}\qquad\qquad
\subfigure[Running times of RiffleJoin]{
\includegraphics[width=.39\textwidth]{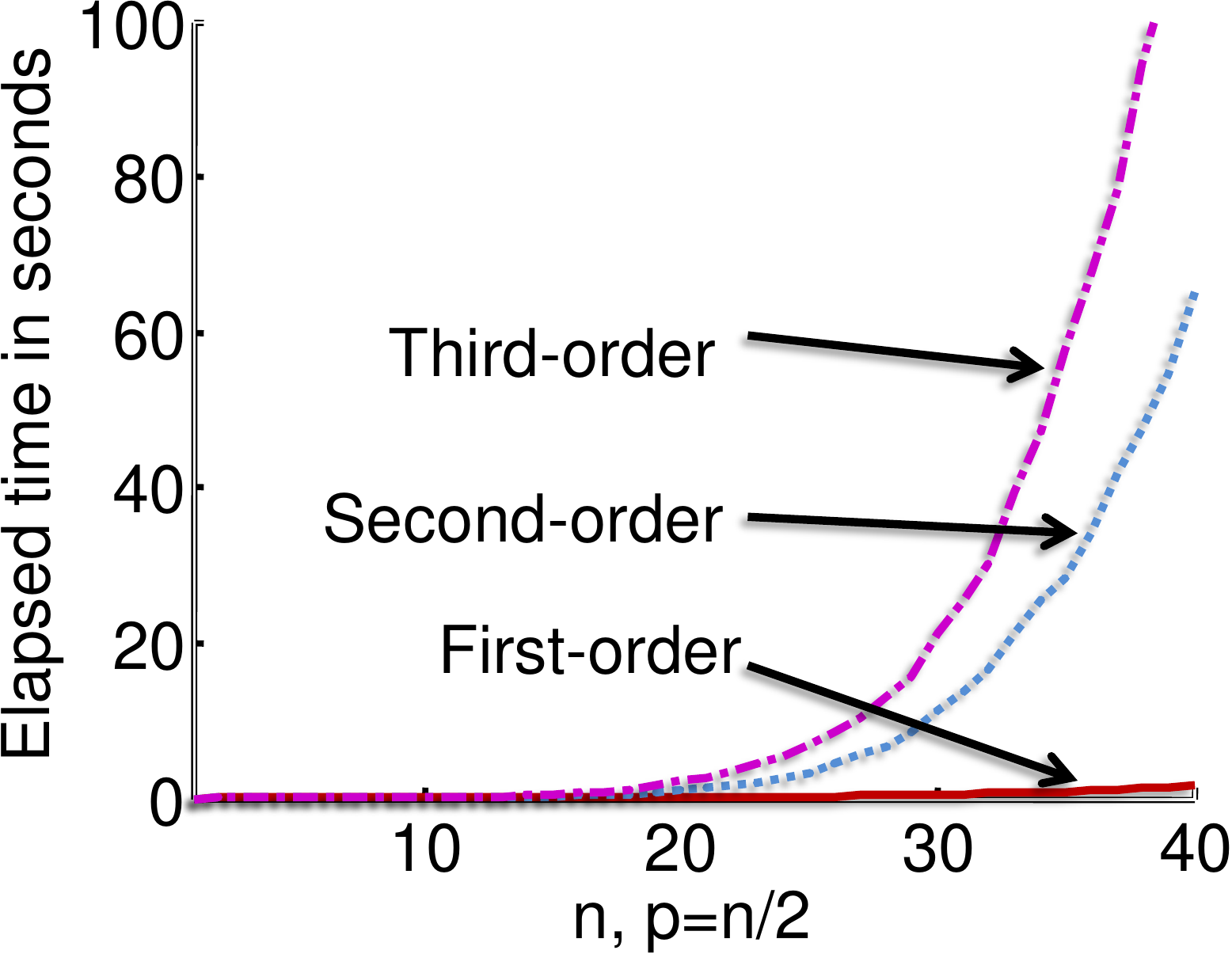}
   \label{fig:scalingfigure}
}
\caption{Synthetic data experiments for a single partitioning of the item set}
\label{fig:experiments_singlesplit}
\end{center}
\end{figure*}

\paragraph{Learning a hierarchy of items}
\begin{figure*}[t!]
\begin{center}
\subfigure[Success rate for structure recovery vs. sample size $(n=16, k=4)$]{
        \includegraphics[width=.37\textwidth]{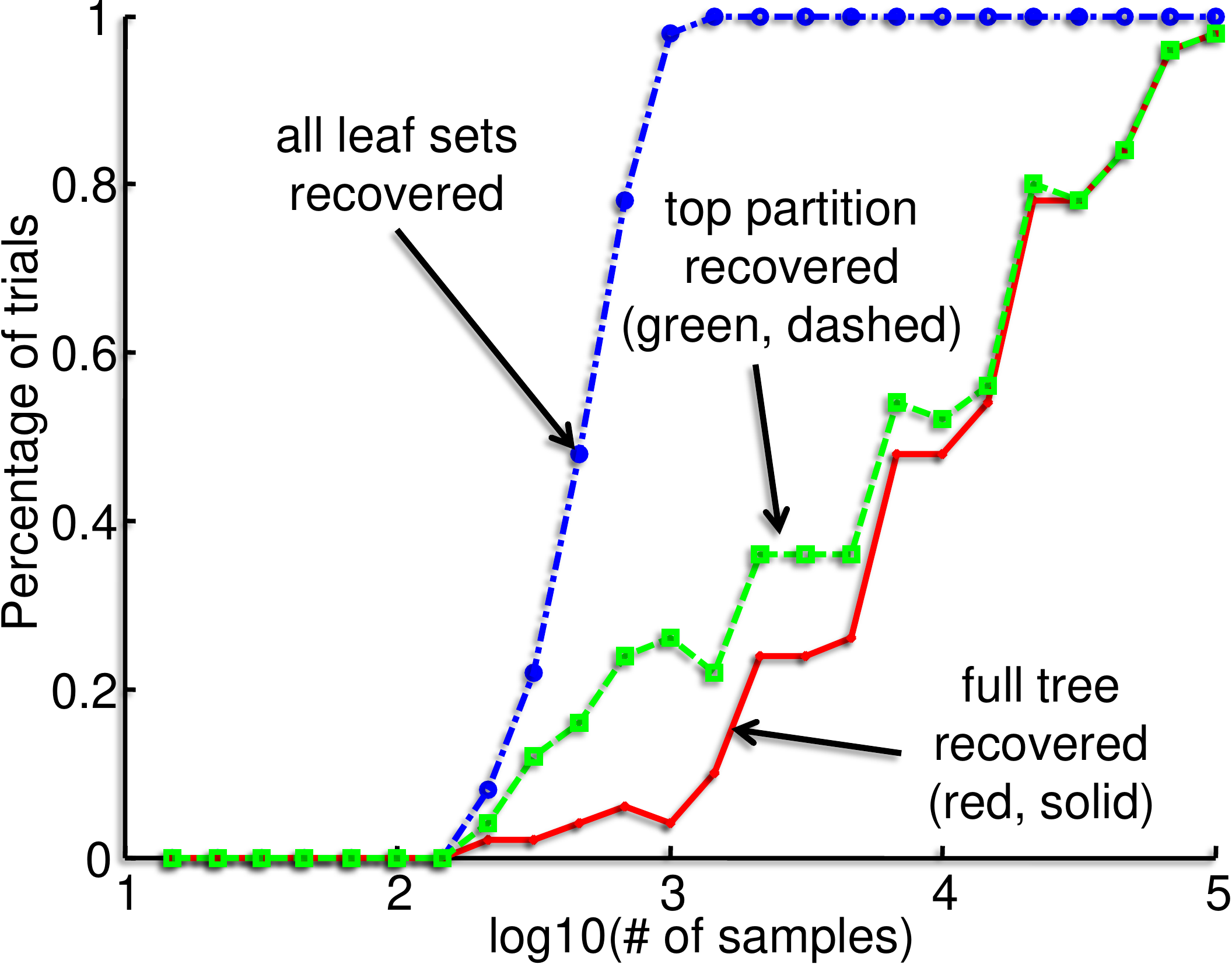}
   \label{fig:successrate16}
}\qquad\qquad
\subfigure[Number of samples required for structure recovery vs. number of items $n$]{
        \includegraphics[width=.37\textwidth]{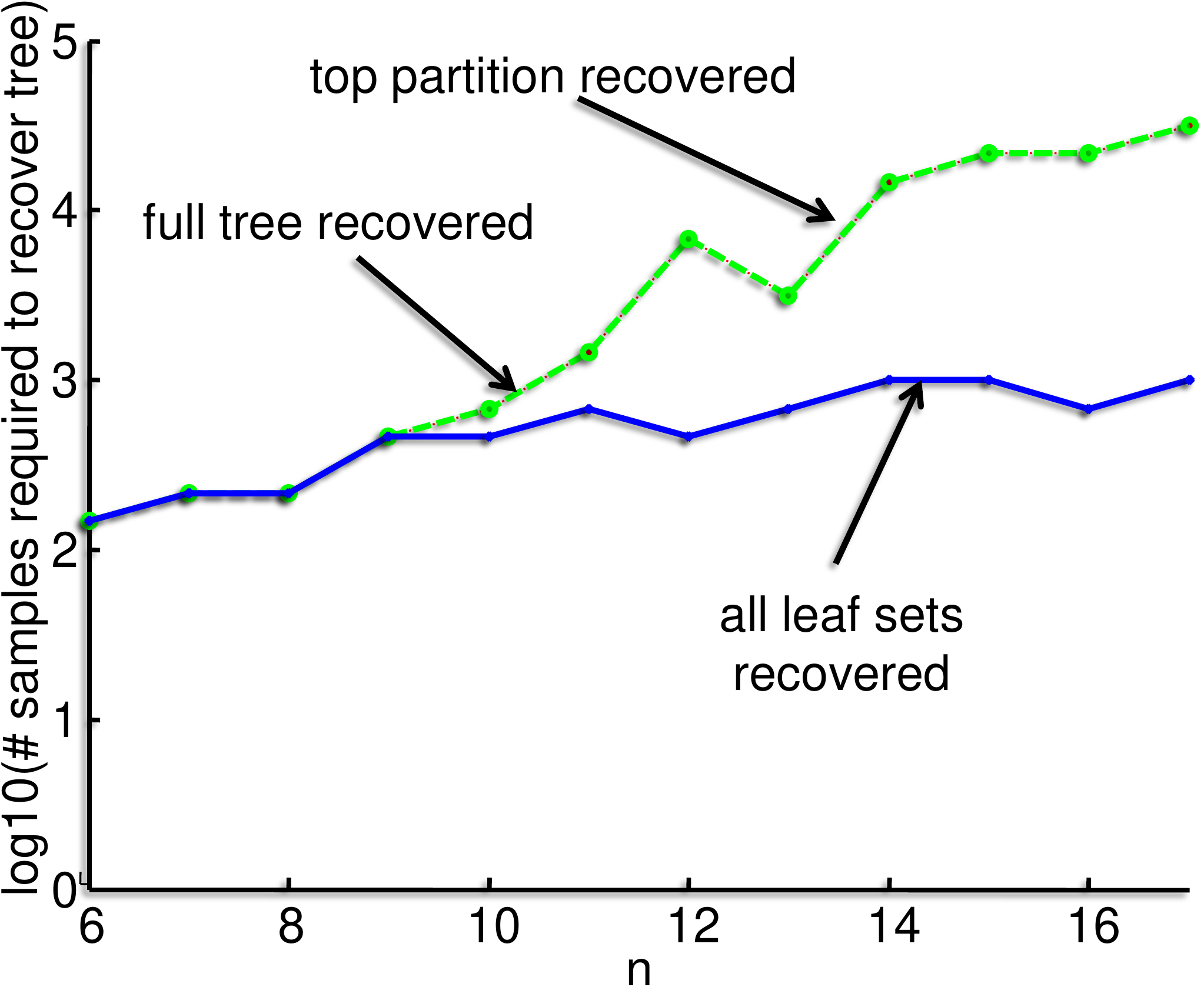}
   \label{fig:numsampsrequired}
}
\subfigure[Test set log-likelihood comparison]{
        \includegraphics[width=.37\textwidth]{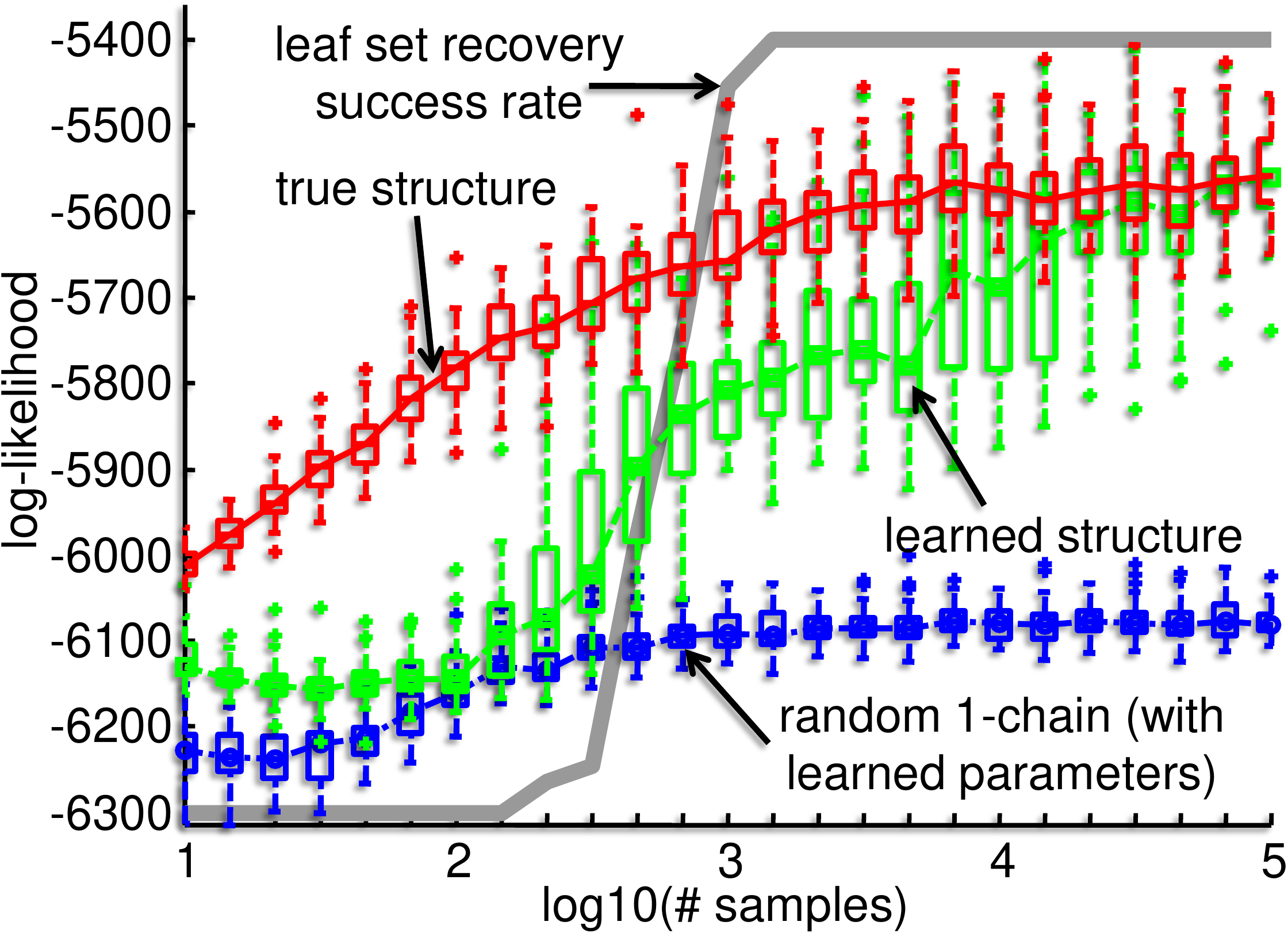}
   \label{fig:simloglike}
}\qquad\qquad
\subfigure[Anchors algorithm success rate $(n=16,\mbox{unknown } k)$]{
        \includegraphics[width=.37\textwidth]{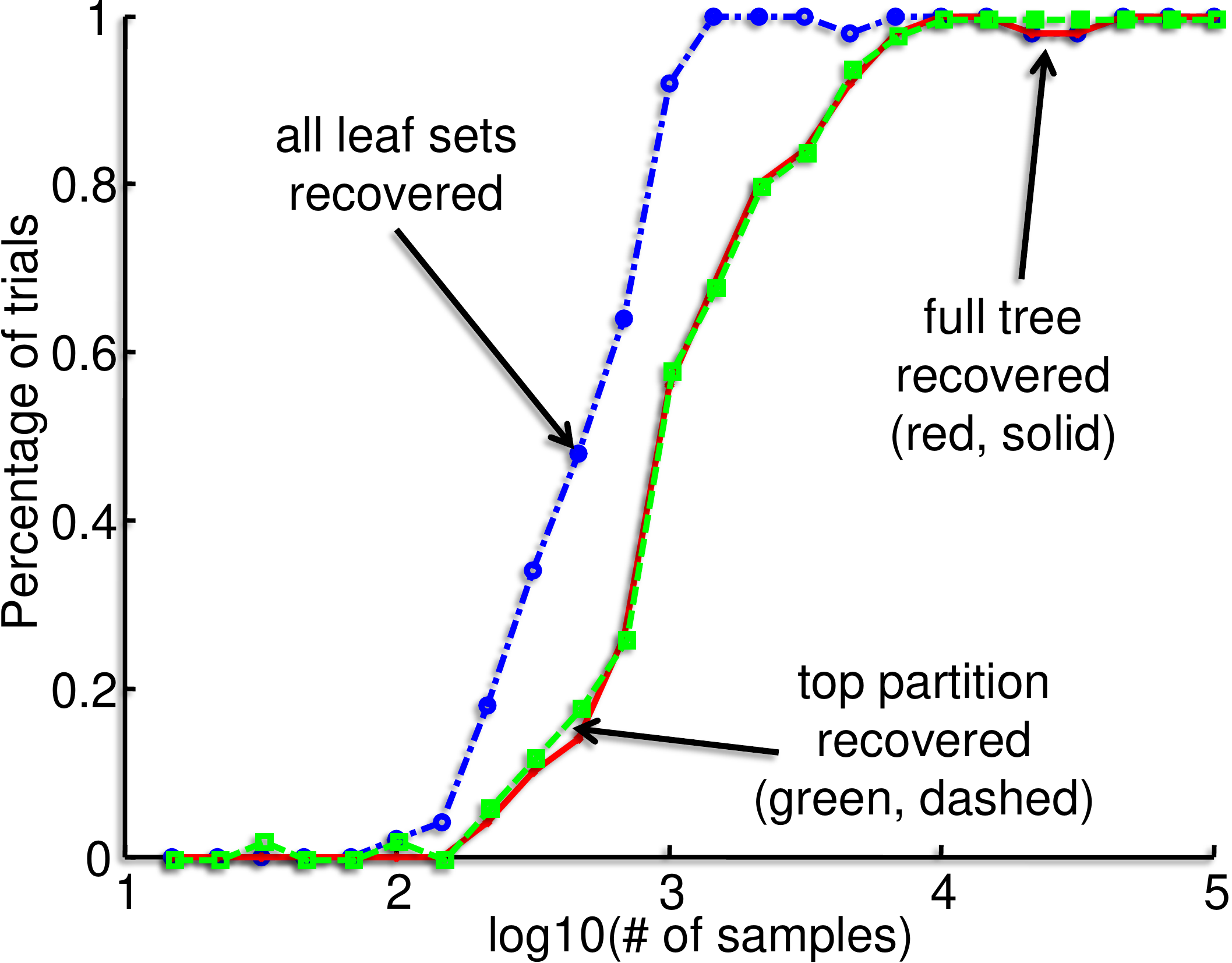}
   \label{fig:hierarchysuccessrate}
}
\caption{Structure discovery experiments on synthetic data}
\label{fig:experiments_structurelearning}
\end{center}
\end{figure*}
We next applied our methods to synthetic data to
show that, given enough samples, our algorithms
\emph{do} effectively
recover the optimal hierarchical structures which generated the original
datasets.  For various settings of $n$,
we simulated data drawn jointly from a $k$-thin chain
model (for $k=4$)
with a random parameter setting for each structure
and applied our exact method for learning thin chains to each sampled dataset.
First, we investigated the effect of varying sample size
on the proportion of trials (out of fifty)
for which our algorithms were able to (a) recover the full underlying tree
structure \emph{exactly}, (b) recover the topmost partition
correctly, or (c) recover all leaf sets correctly (but possibly
out of order).
Figure~\ref{fig:successrate16} shows the result for an itemset of size $n=16$.
Figure~\ref{fig:numsampsrequired}, shows, as a function of $n$,
the number of samples that were required in the same experiments to
(a) \emph{exactly} recover the full underlying
structure or (b) recover the correct leaf sets, for at least 90\% of the trials.
What we can observe from the plots is that, given enough samples,
reliable structure recovery \emph{is} indeed possible.  It is also interesting
to note that recovery of the correct leaf sets can be done with much fewer samples than
are required for recovering the full hierarchical structure of the model.

After learning a structure for each dataset, we learned model parameters and
evaluated the log-likelihood of each model on 200 test examples
drawn from the true distributions.  In Figure~\ref{fig:simloglike},
we compare log-likelihood performance when (a) the true
structure is given (but not parameters), (b) a $k$-thin chain is learned with known $k$,
and (c) when we use a random generated 1-chain structure.
As expected, knowing the true structure
results in the best performance, and the 1-chain is overconstrained.
However, our structure learning algorithm is eventually able to catch up to the
performance of the true structure given enough samples.  It is also interesting
to note that the jump in performance at the halfway point in the plot coincides
with the jump in the success rate of discovering all leaf sets correctly --- we
conjecture that performance
is sometimes less sensitive to the actual hierarchy used, as long as the leaf
sets have been correctly discovered.

To test the Anchors algorithm,
we ran the same simulation using Algorithm~\ref{alg:anchors}
on data drawn from hierarchical models with no fixed $k$.
We generated roughly balanced structures, meaning that item sets
were recursively partitioned into (almost) equally sized subsets
at each level of the hierarchy.
From Figure~\ref{fig:hierarchysuccessrate}, we see that the Anchors
algorithm can also discover the true structure given enough samples.
Interestingly, the difference in sample complexity for discovering leaf
sets versus discovering the full tree is not nearly as pronounced
as in Figure~\ref{fig:successrate16}.  We believe that this is due to the
fact that the balanced trees have less depth than the thin chains, leading to
fewer opportunities for our greedy top-down approach to commit errors.

\subsection{Data analysis: sushi preference data}
We now turn to analyzing real datasets.
For our first analysis, we examine a sushi preference ranking 
dataset~\citep{toshihiro03} consisting of 5000 full
rankings of ten types of sushi.  
The items are enumerated in Figure~\ref{tab:sushitypes}.
Note that, compared to the APA election data, the sushi dataset has twice
as many items, but fewer examples.
\begin{figure}[t]
\begin{center}
{\footnotesize
\begin{tabular}{|ccc|}
\hline
1. ebi (shrimp) &
2. anago (sea eel) &
3. maguro (tuna) \\
4. ika (squid) &
5. uni (sea urchin) &
6. sake (salmon roe) \\
7. tamago (egg) &
8. toro (fatty tuna) &
9. tekka-maki (tuna roll) \\ \, &
10. kappa-maki (cucumber roll)
& \, \\
\hline
\end{tabular}
}
\end{center}
\caption{List of sushi types in the~\cite{toshihiro03} dataset}
\label{tab:sushitypes}
\end{figure}

\begin{figure*}[t!]
\begin{center}
\subfigure[Average log-likelihood of held out 
test examples from the sushi dataset]{
\includegraphics[width=.35\textwidth]{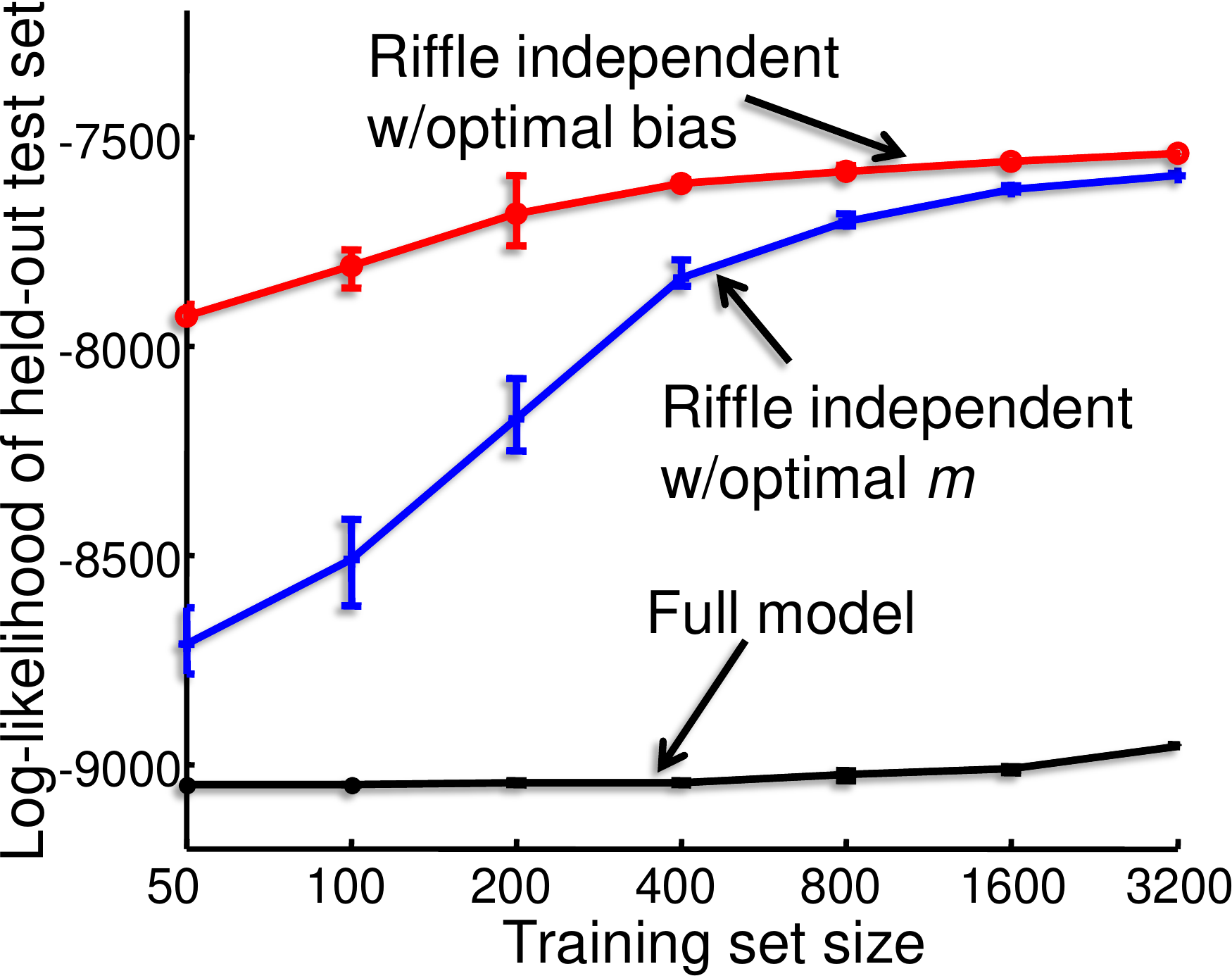}
   \label{fig:sushiloglike}
}\qquad\qquad
\subfigure[First-order probabilities of Uni (sea urchin) rankings]{
\includegraphics[width=.35\textwidth]{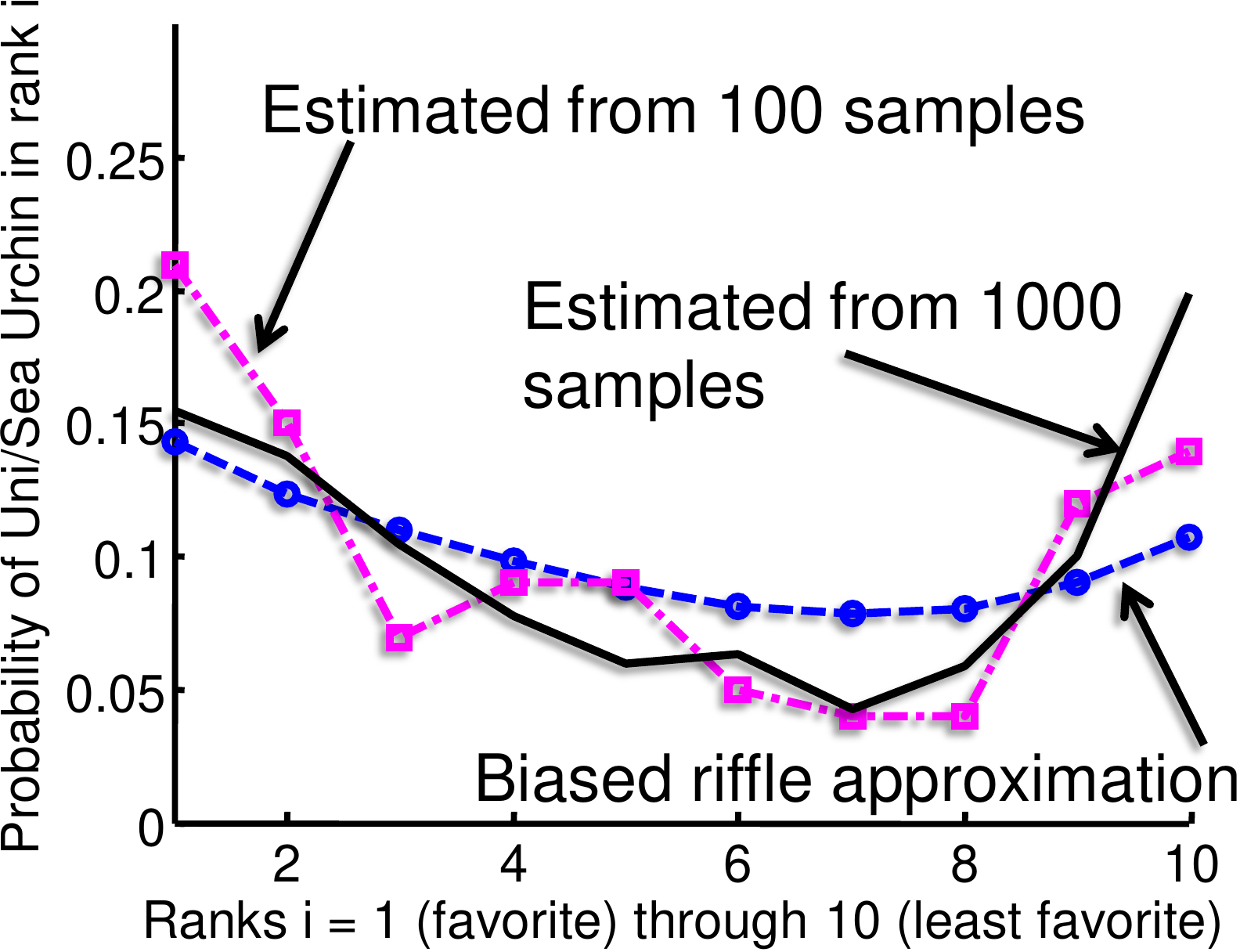}
   \label{fig:uniprobs}
}
\caption{Sushi preference ranking experiments}
\label{fig:experiments_sushi}
\end{center}
\end{figure*}

We begin by studying our methods in the case of a
single binary partitioning of the item set.
Unlike the APA dataset, there is no obvious way to 
naturally partition the types of sushi into two sets --- 
in our first set of experiments, we have arbitrarily
divided the item set into $A=\{1,\dots,5\}$ and $B=\{6,\dots,10\}$.

We divided the data into training
and test sets (with 500 examples)
and estimated the true distribution in three ways: (1) directly
from samples (with regularization),
(2) using a riffle
independent distribution (split evenly into two groups of five
and mentioned above) with the optimal shuffling distribution
$m$, and (3) with a biased riffle shuffle (and optimized bias $\alpha$).
Figure~\ref{fig:sushiloglike} plots testset log-likelihood as a function of
training set size --- we see that riffle independence assumptions can help
significantly to lower the sample complexity of learning.
Biased riffle shuffles, as can also be seen,
are a useful learning bias with very small samples.

As an illustration of the behavior of
biased riffle shuffles, see Figure~\ref{fig:uniprobs} which shows
the approximate first-order marginals of Uni (Sea Urchin) rankings, and the
biased riffle approximation.  The Uni marginals are interesting, because
while many people like Uni, thus providing high rankings,
many people also hate it, providing low rankings.
The first-order marginal estimates have significant variance at
low sample sizes, but with the biased riffle approximation, one can
achieve a reasonable approximation to the distribution even with few samples
at the cost of being somewhat oversmoothed.

\paragraph{Structure learning on the sushi dataset}
Figure~\ref{fig:sushitree} shows the hierarchical structure
that we learn using the entire sushi dataset.  Since the sushi
are not prepartitioned into distinct coalitions, it is
somewhat more difficult than with, say, the APA data, 
to interpret whether the estimated structure makes sense. 
However, parts of the tree certainly seem like reasonable groupings.
For example, all of the tuna related sushi types have been 
clustered together.  Tamago and kappa-maki (egg and cucumber rolls)
are ``safer'', typically more boring choices, while
uni and sake (sea urchin and salmon roe) are more daring.
Anago (sea eel), is the odd man out in the estimated hierarchy, being
partitioned away from the remaining items at the top of the tree.

To understand the behavior of our algorithm with smaller sample sizes,
we looked for features of the tree from Figure~\ref{fig:sushitree} 
which remained stable even when learning with smaller sample sizes.  
Figure~\ref{fig:stabletree_sushi} summarizes the results of our bootstrap
analysis for the sushi dataset, in which 
we resample from the original training set 200 times at each of
different sample sizes and plot the proportion of learned
hierarchies which, (a) recover `sea eel' as the topmost partition,
(b) recover all leaf sets correctly, (c), recover the entire tree correctly,
(d) recover the tuna-related sushi leaf set,
(e) recover the \{tamago, kappa-maki\} leaf set,
and (f) recover the \{uni, sake\} leaf set.

\begin{figure*}[t!]
\begin{center}
\subfigure[First-order marginals of Irish election data]{
\includegraphics[width=.3\textwidth]{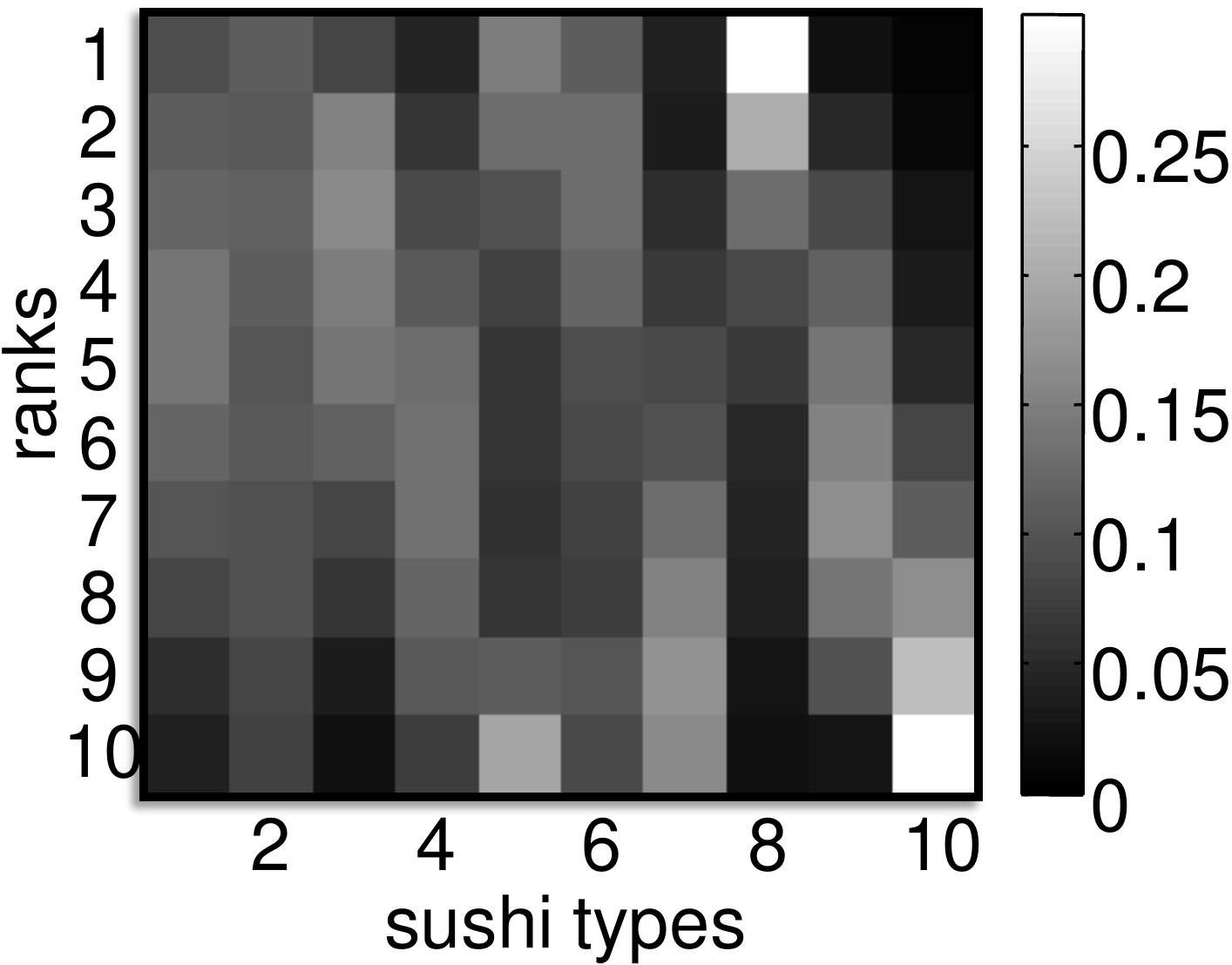}
\label{fig:sushifirstorder_exact}
}\qquad\qquad
\subfigure[Riffle independent approximation of first-order marginals
with learned hierarchy]{
\includegraphics[width=.3\textwidth]{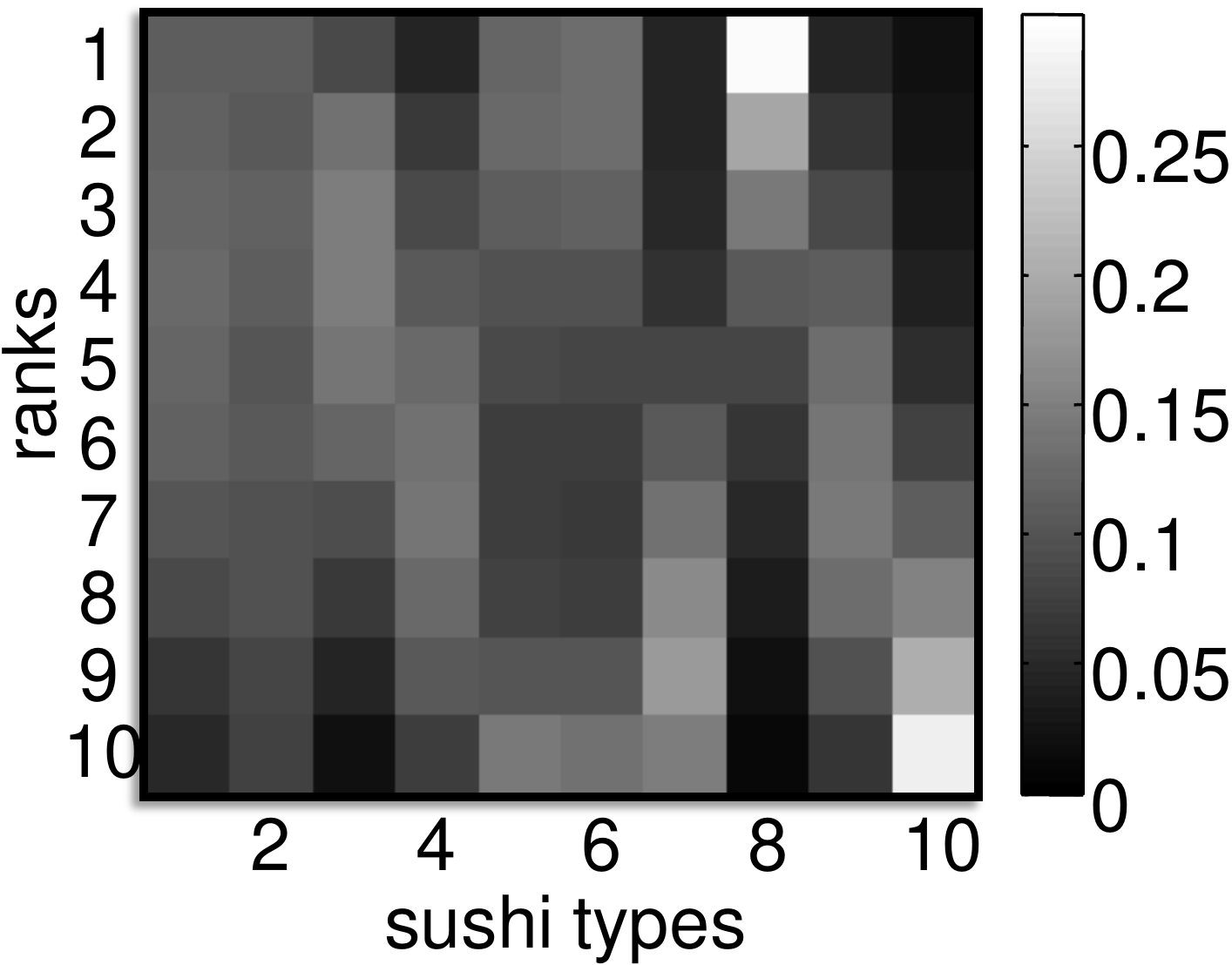}
\label{fig:sushifirstorder_approx}
}\,
\caption{Sushi preference dataset: exact first-order marginals
and riffle independent approximation}
\label{fig:sushifirstorder}
\end{center}
\end{figure*}

\begin{figure*}[t!]
\begin{center}
\subfigure[Stability of bootstrapped tree `features' of the sushi dataset]{
\includegraphics[width=.4\textwidth]{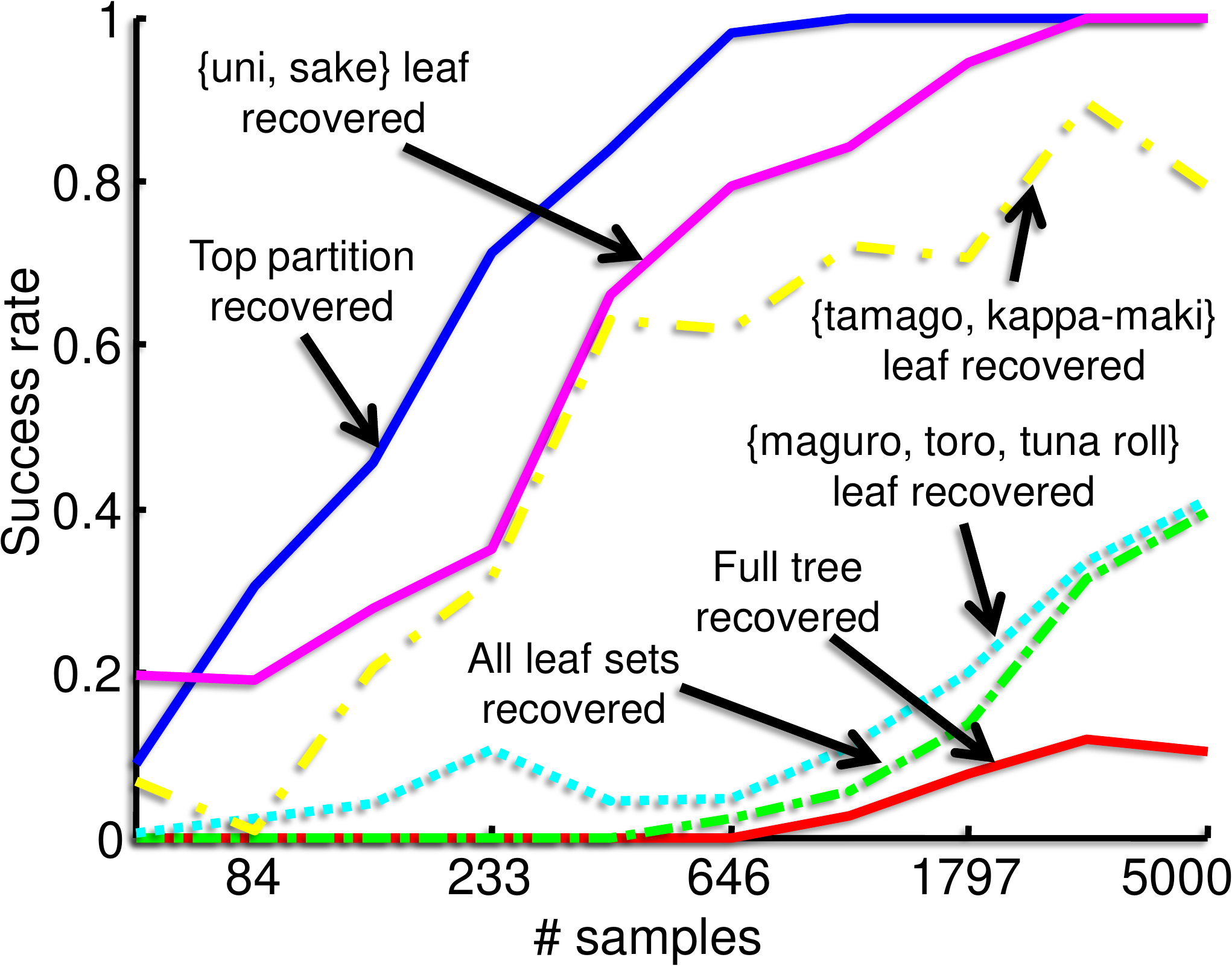}
\label{fig:stabletree_sushi}
}\;
\subfigure[Learned hierarchy for sushi dataset using all 5000 rankings]{
\includegraphics[width=.5\textwidth]{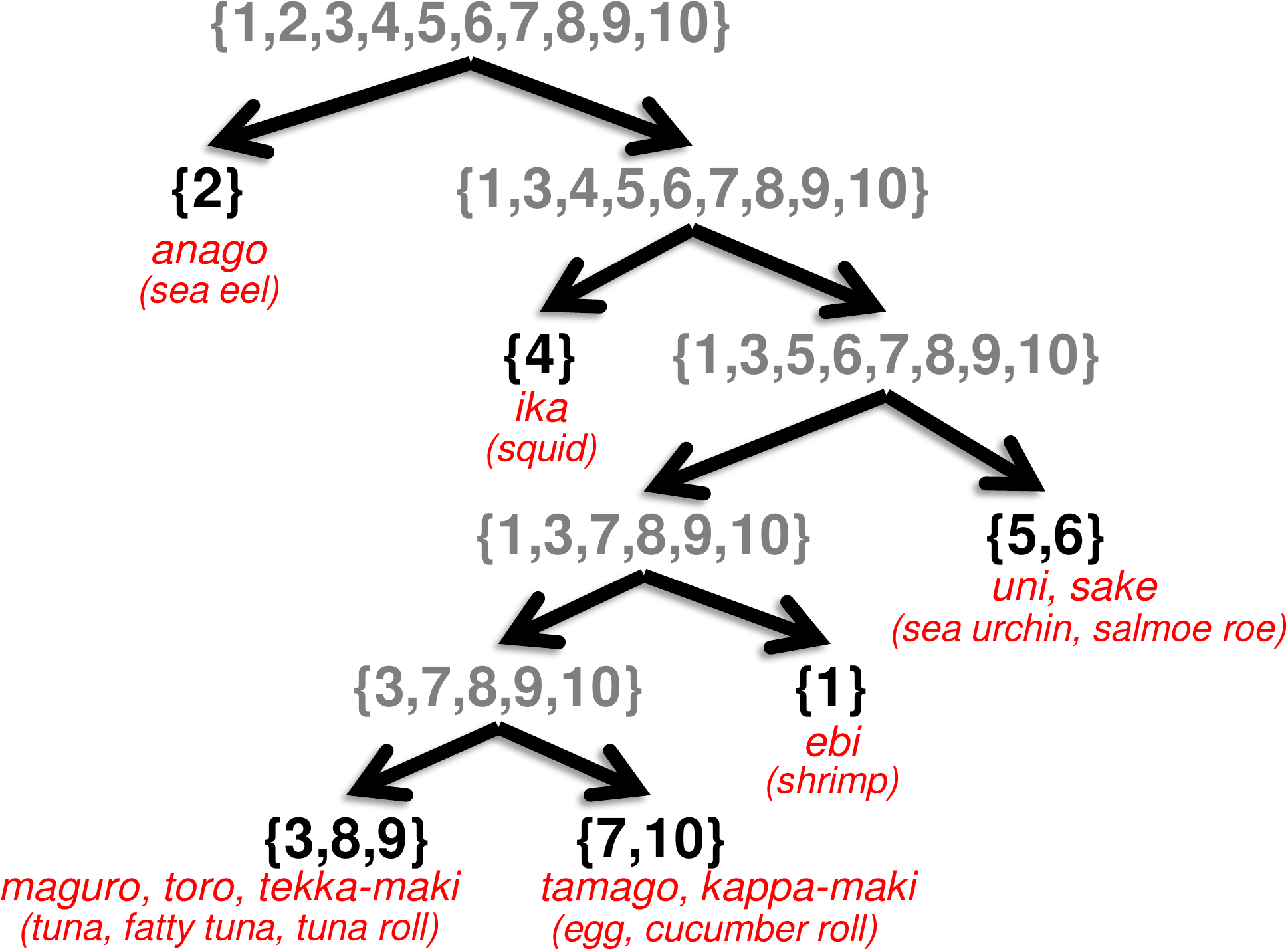}
\label{fig:sushitree}
}
\caption{Structure discovery experiments: Sushi preference dataset}
\label{fig:sushistructure}
\end{center}
\end{figure*}

\subsection{Data analysis: Irish election data}
We next applied our algorithms to a larger Irish House 
of Parliament (D\'ail \'Eireann)
election dataset from the Meath constituency in Ireland 
(Figure~\ref{fig:meathmap}).
The D\'ail \'Eireann uses the \emph{single transferable vote} (STV)
election system, in which voters rank a subset of candidates. 
In the Meath constituency, there were 
14 candidates in the 2002 election, running for five allotted seats.
The candidates identified with the two
major rival political parties, Fianna F\'ail and 
Fine Gael, as well as a number of smaller parties 
(Figure~\ref{fig:meathtable}).  
See~\cite{gormley06} for more election
details (including candidate names) as well as an alternative analysis.
In our experiments, we 
used a subset of roughly 2500 fully ranked ballots from the election.

\begin{figure}[t!]
\begin{center}
\subfigure[The Meath constituency in Ireland, shown in green,
was one of three constituencies to have electronic voting
in 2002. (Map from Wikipedia)]{
\raisebox{-60pt}{
\includegraphics[width=.30\textwidth]{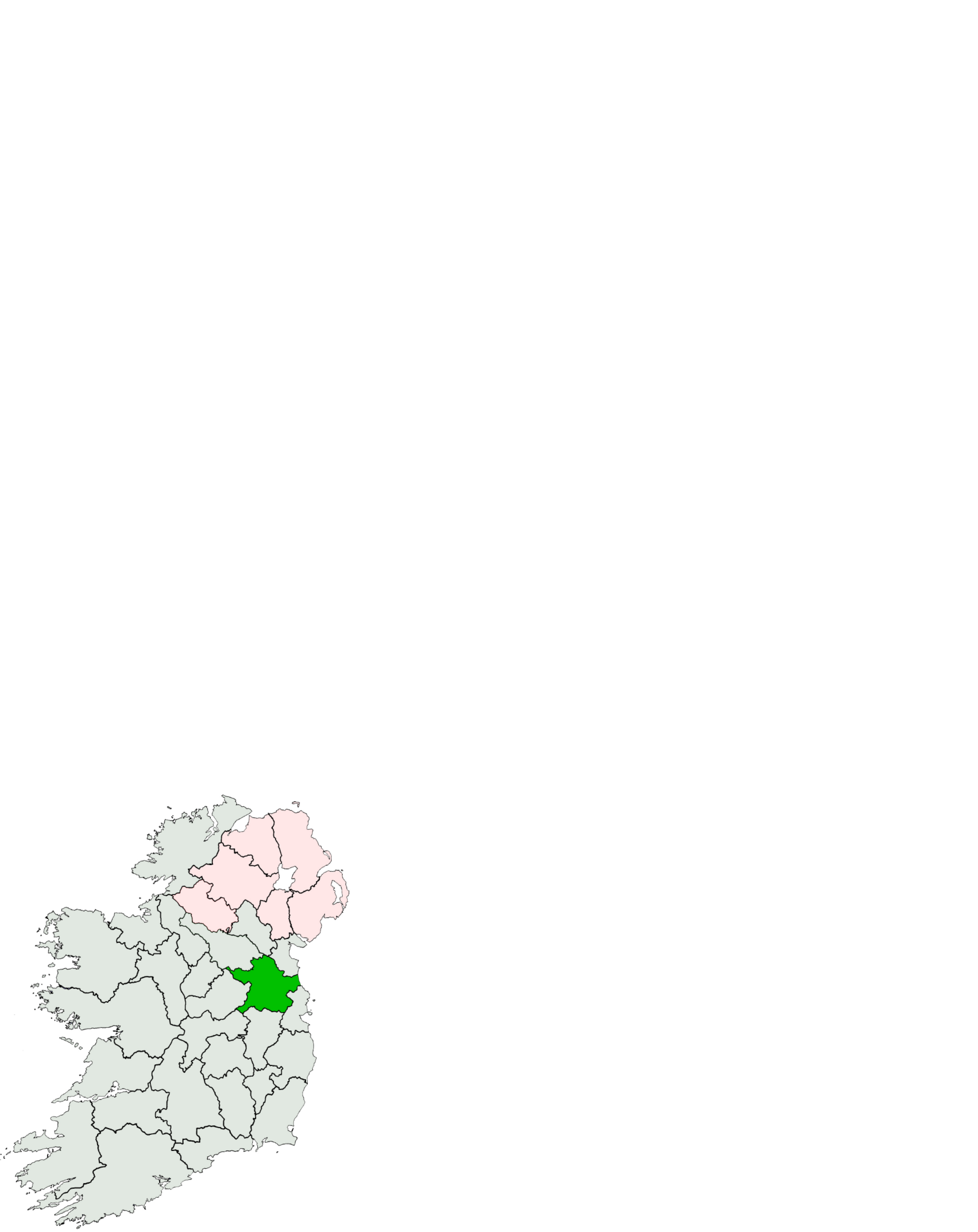}
\label{fig:meathmap}
}
}\qquad\qquad
\subfigure[List of candidates from the Meath
constituency election in 2002 for five seats in
the D\'ail \'Eireann (reproduced from~\cite{gormley06})]{
\raisebox{10pt}{\scriptsize
\begin{tabular}{|cc|c|}
\hline
& Candidate & Party \\
\hline
1 & Brady, J. & Fianna F\'ail \\
2 & Bruton, J. & Fine Gael \\
3 & Colwell, J. & Independent \\
4 & Dempsey, N. & Fianna F\'ail \\
5 & English, D. & Fine Gael \\
6 & Farrelly, J. & Fine Gael \\
7 & Fitzgerald, B. & Independent \\
8 & Kelly, T. & Independent \\
9 & O'Brien, P. & Independent \\
10 & O'Byrne, F. & Green Party \\
11 & Redmond, M. & Christian Solidarity \\
12 & Reilly, J. & Sinn F\'ein \\
13 & Wallace, M. & Fianna F\'ail \\
14 & Ward, P. & Labour \\
\hline
\end{tabular}\label{fig:meathtable}
}
}
\end{center}
\caption{Irish election dataset summary}
\end{figure}

\begin{figure*}[t!]
\begin{center}
\subfigure[First-order marginals of Irish election data]{
\includegraphics[width=.30\textwidth]{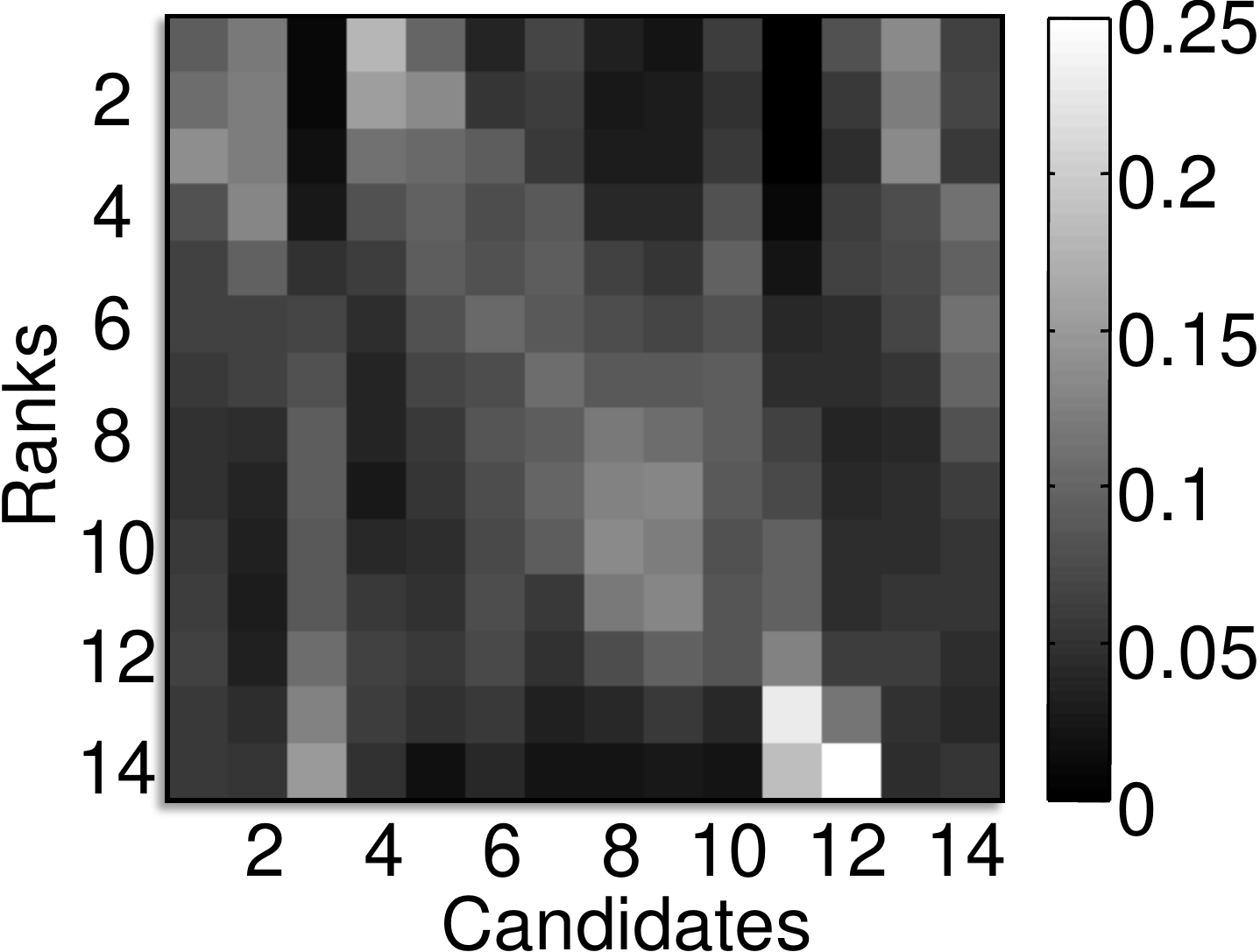}
\label{fig:meathfirstorder_exact}
}\,
\subfigure[Riffle independent approximation of first-order marginals
with $A=\{\mbox{Fianna F\'ail},\mbox{Fine Gael}\}$, and $B=\{\mbox{everything else}\}$]{
\includegraphics[width=.30\textwidth]{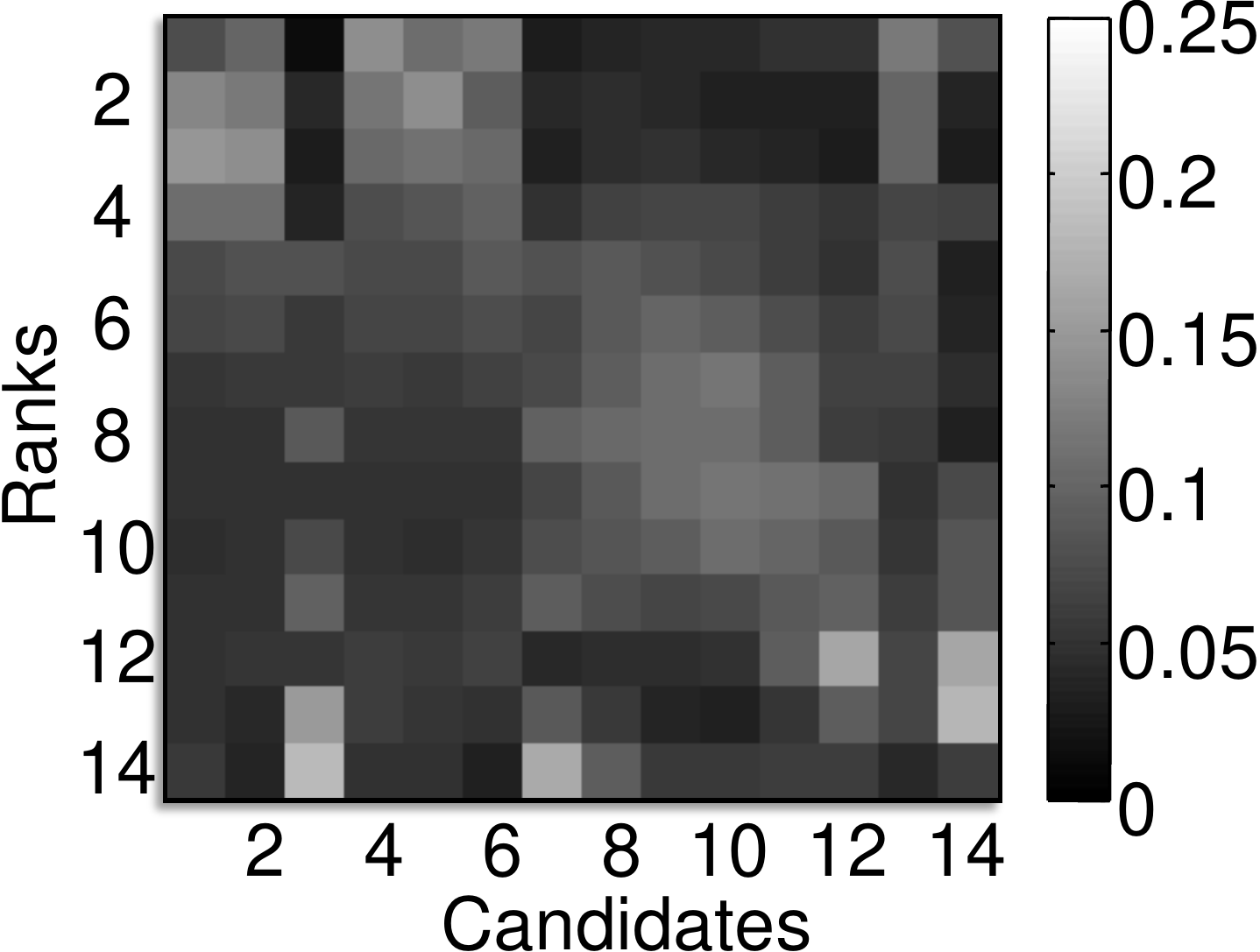}
\label{fig:meathfirstorder_approx1}
}\,
\subfigure[Riffle independent approximation of first-order marginals
with learned hierarchy in Figure~\ref{fig:meathtree}]{
\includegraphics[width=.30\textwidth]{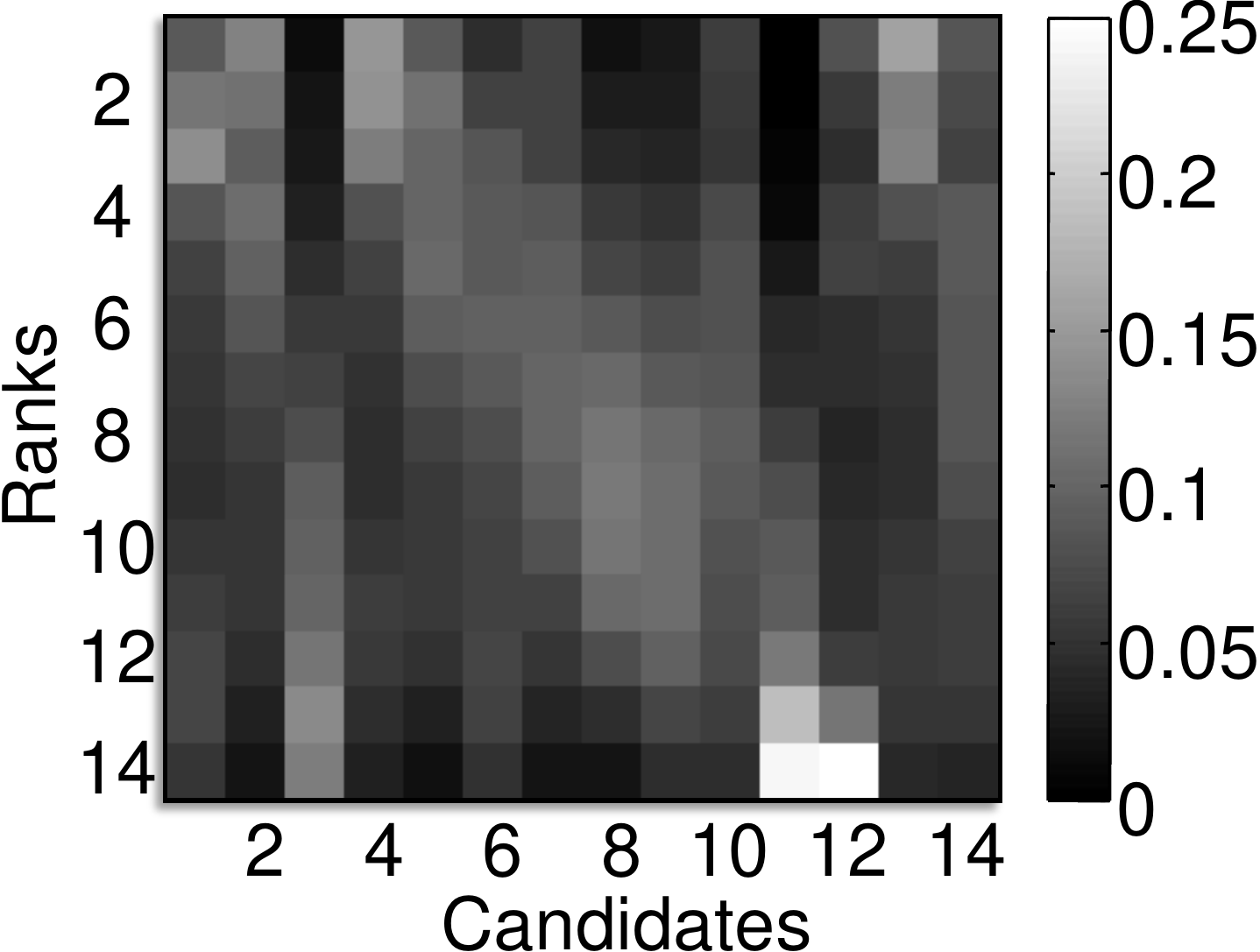}
\label{fig:meathfirstorder_approx2}
}
\caption{Irish Election dataset: exact first-order marginals 
and riffle independent approximations}
\label{fig:meathfirstorder}
\end{center}
\end{figure*}

\begin{figure*}[t!]
\begin{center}
\includegraphics[width=.5\textwidth]{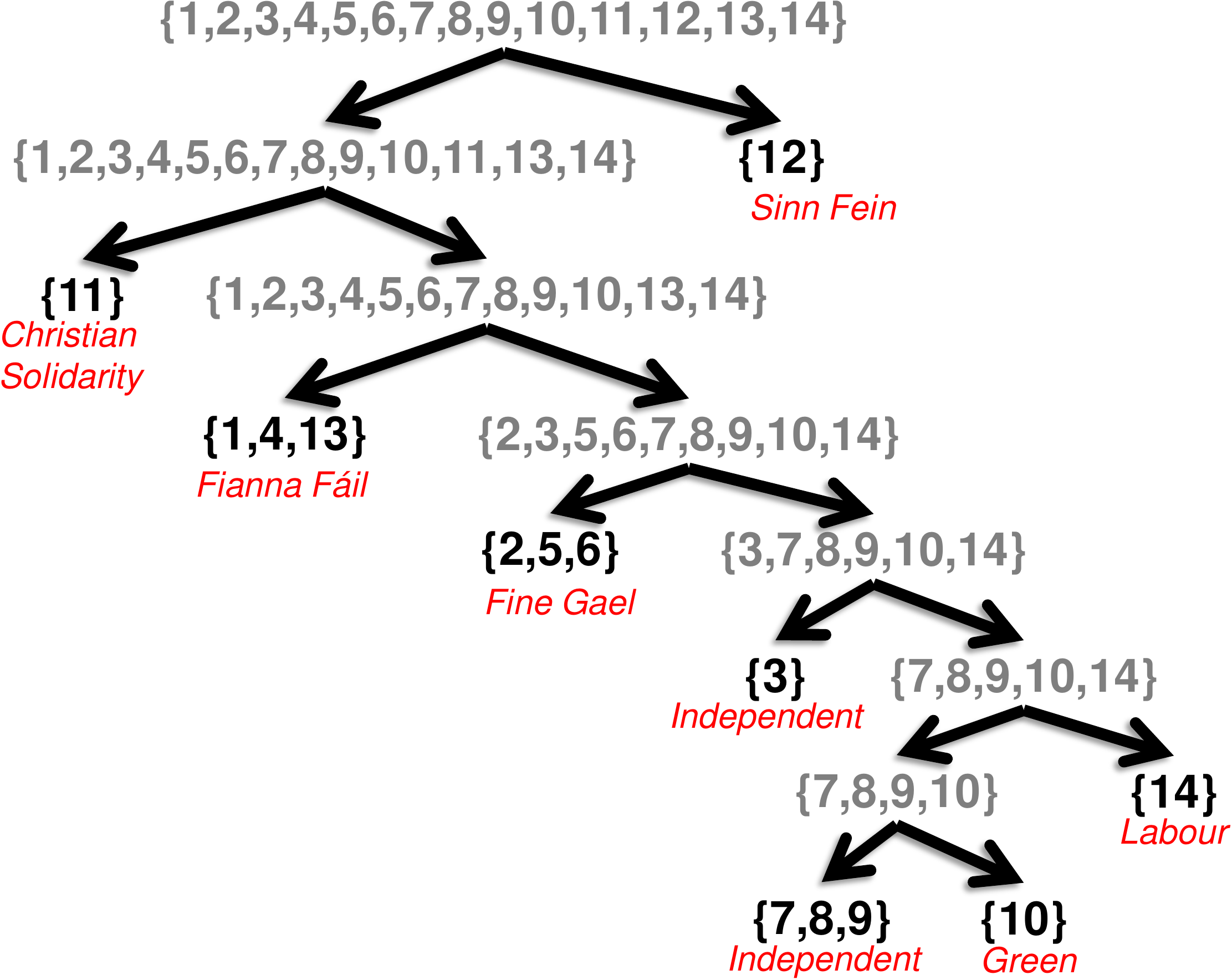}
\end{center}
\caption{Learned hierarchy for Irish Election dataset using
all 2500 ballots}
\label{fig:meathtree}
\end{figure*}

To summarize the dataset, Figure~\ref{fig:meathfirstorder_exact} 
shows the matrix of first-order marginals estimated from the dataset.
Candidates $\{1,2,4,5,6,13\}$ form the set of ``major'' party
candidates belonging to either Fianna F\'ail or Fine Gael, and
as shown in the figure, fared much better in the election than
the other seven minor party candidates.
Notably, candidates 11 and 12 (belonging to the Christian Solidary
Party and Sinn F\'ein, respectively) received on average, the lowest ranks
in the 2002 election.  One of the differences between the two candidates,
however, is that a significant portion of the electorate also
ranked the Sinn F\'ein candidate very high.

Though it may not necessarily be clear how one might partition 
the candidates, a natural idea might be to assume that the major party
candidates ($A$) are riffle independent of the minor party candidates ($B$).
In Figure~\ref{fig:meathfirstorder_approx1}, we show the first-order
marginals corresponding to an approximation in which $A$ and $B$
are assumed to be riffle independent.
Visually, the approximate first-order marginals can be seen to be roughly
similar to the exact first-order marginals, however there are significant
features of the matrix which are not captured by the approximation ---
for example, the columns belonging to candidates 11 and 12 are not well
approximated.  
In Figure~\ref{fig:meathfirstorder_approx2}, we plot a more principled
approximation corresponding to a learned hierarchy, which we discuss next.
As can be seen, the first-order marginals obtained via structure learning 
is visually much closer to the exact marginals.

\paragraph{Structure discovery on the Irish election data}
As with the APA data, both the exhaustive optimization of $\FFF$ and the Anchors algorithm
returned the same tree, with running times of
69.7 seconds and 2.1 seconds respectively (not including the
3.1 seconds required for precomputing mutual informations).
The resulting tree,
with candidates enumerated alphabetically from 1 through 14,
is shown (only up to depth 4), in Figure~\ref{fig:meathtree}.
As expected, the candidates belonging to the two major
parties, Fianna F\'ail and Fine Gael, are neatly partitioned into
their own leaf sets.  The topmost leaf is the Sinn Fein candidate,
indicating that voters tended to insert him into the ranking independently
of all of the other 13 candidates.

To understand the behavior of our algorithm with smaller sample sizes,
we looked for features of the tree from Figure~\ref{fig:meathtree} which
remained stable even when learning with smaller sample sizes.  In Figure~\ref{fig:stabletree},
we resampled from the original training set 200 times at 
different sample sizes and plot the proportion of learned 
hierarchies which, (a) recover the Sinn Fein candidate as 
the topmost leaf, (b) partition the two major parties into leaf sets,
and (c) agree with the original tree on all leaf sets,
and (d) recover the entire tree.
Note that while the dataset is insufficient to support
the entire tree structure, even with about 100 training examples, candidates 
belonging to the major parties
are consistently grouped together indicating strong party 
influence in voting behavior.

\begin{figure*}[t!]
\begin{center}
\subfigure[Stability of bootstrapped tree `features'
of the Irish dataset]{
\includegraphics[width=.36\textwidth]{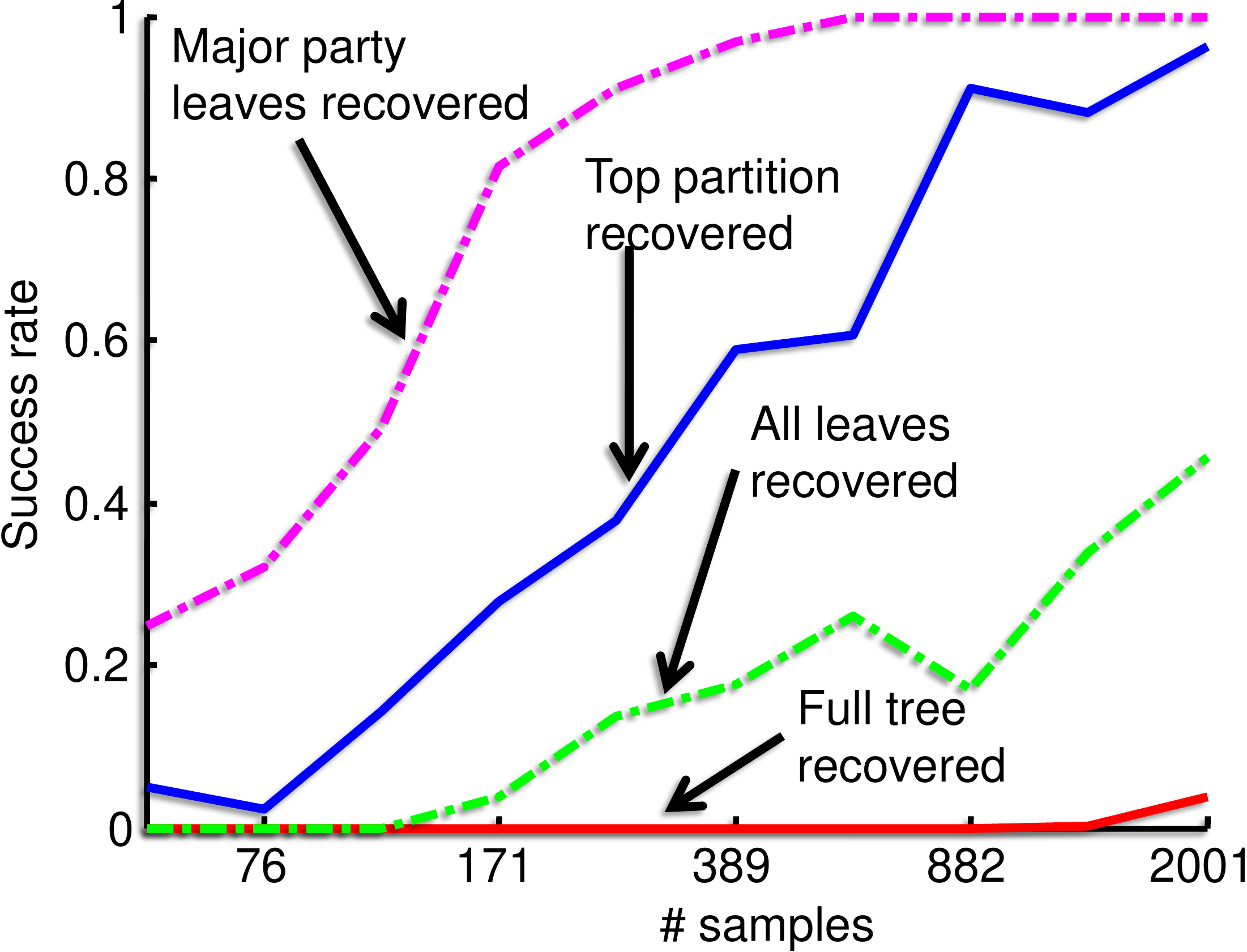}
\label{fig:stabletree}
}\qquad\qquad
\subfigure[Log-likelihood of held out test examples
from the Irish dataset using optimized structures]{
\includegraphics[width=.4\textwidth]{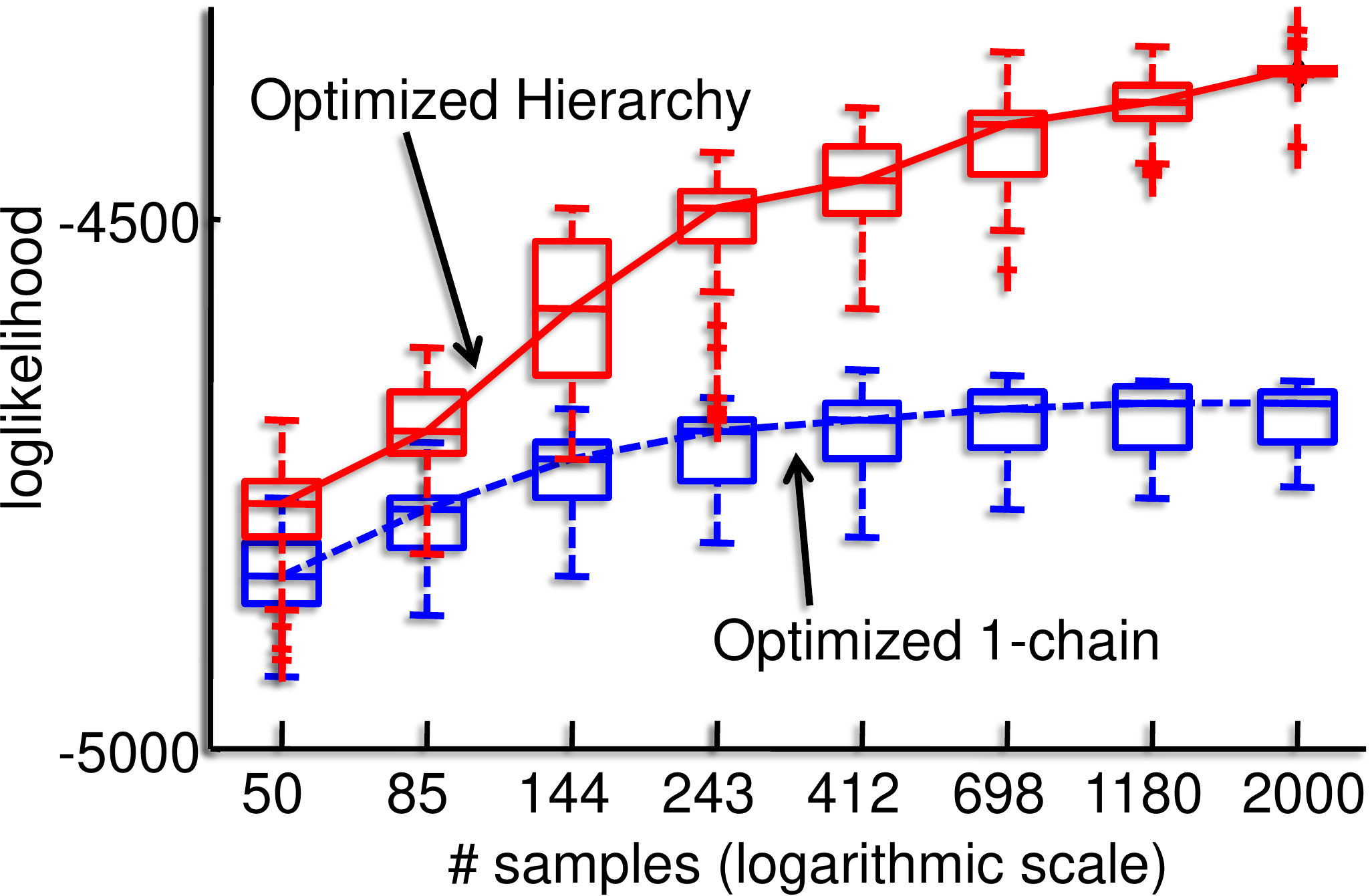}
\label{fig:meathloglike}
}
\caption{Structure Discovery Experiments: Irish Election dataset}
\label{fig:experiments_irish}
\end{center}
\end{figure*}
We compared the results between learning a
general hierarchy (without fixed $k$)
and learning a 1-thin chain model
on the Irish data.  Figure~\ref{fig:meathloglike} shows the
log-likelihoods achieved by both models on a held-out
test set as the training set size increases.
For each training set size, we subsampled the Irish
dataset 100 times to produce confidence intervals.
Again, even with small sample sizes, the hierarchy outperforms the 1-chain
and continually improves with more and more training data.
One might think that the hierarchical models, which use more parameters
are prone to overfitting, but in practice, the models learned by our
algorithm devote most of the extra parameters towards modeling the correlations
among the two major parties.  As our results suggest, such
intraparty ranking correlations are crucial for
achieving good modeling performance.

Finally, we ran our structure learning algorithm on 
two similar but smaller election datasets from the other
constituencies in the 2002 election which supported electronic
voting, the Dublin North and West constituencies.  
Figure~\ref{fig:dublintrees} shows the resulting hierarchies learned 
from each dataset.  
As with the Meath constituency, the Fianna F\'ail and Fine Gael
are consistently grouped together in leaf sets in the Dublin 
datasets.  Interestingly, the Sinn F\'ein and Socialist 
parties are also consistently grouped in the Dublin datasets,
potentially indicating some latent similarities between the two
parties.

\begin{figure*}[t!]
\begin{center}
\subfigure[Dublin North]{
\includegraphics[width=.34\textwidth]{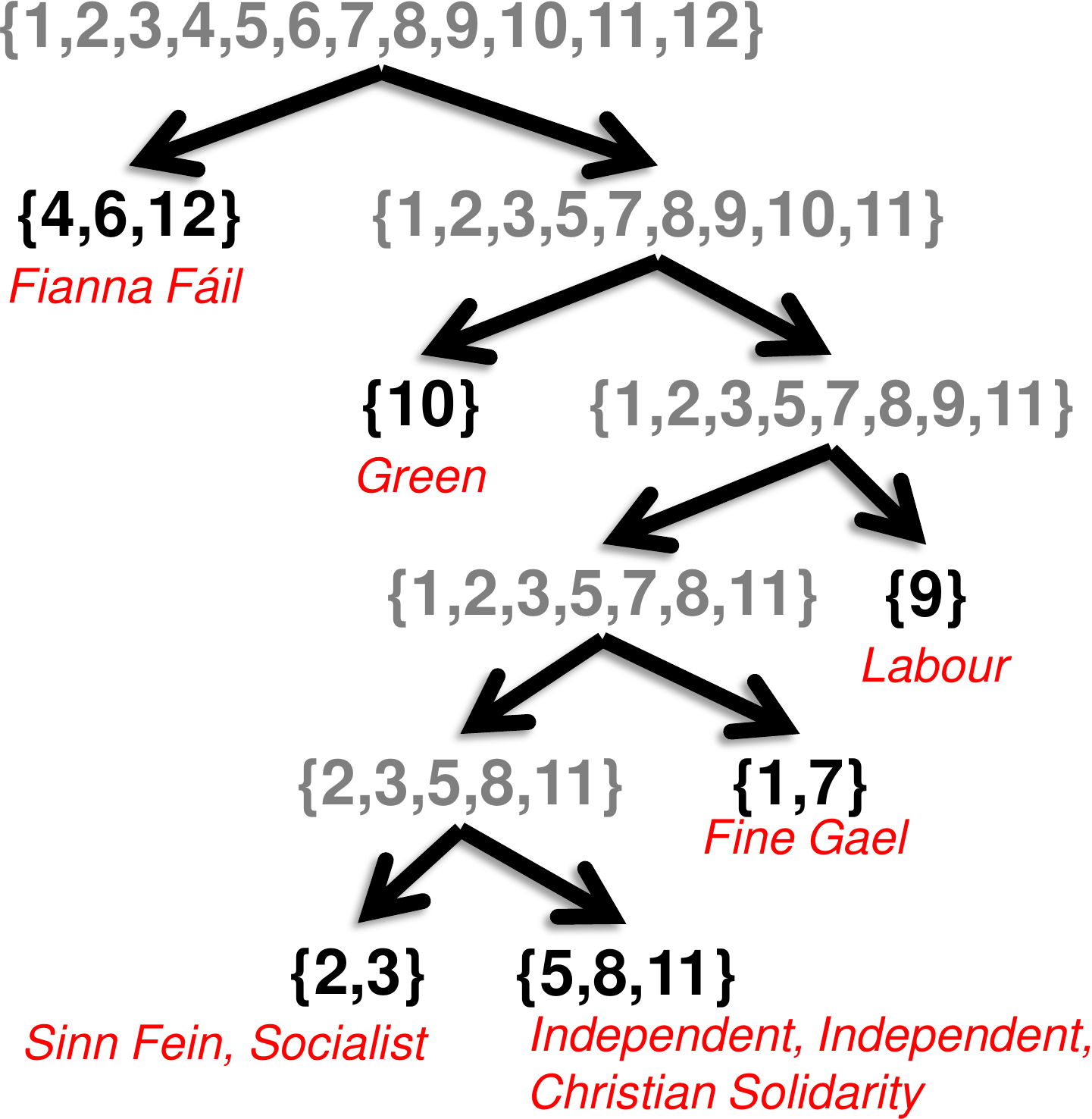}
\label{fig:dublinnorth}
}\qquad
\subfigure[Dublin West]{
\includegraphics[width=.40\textwidth]{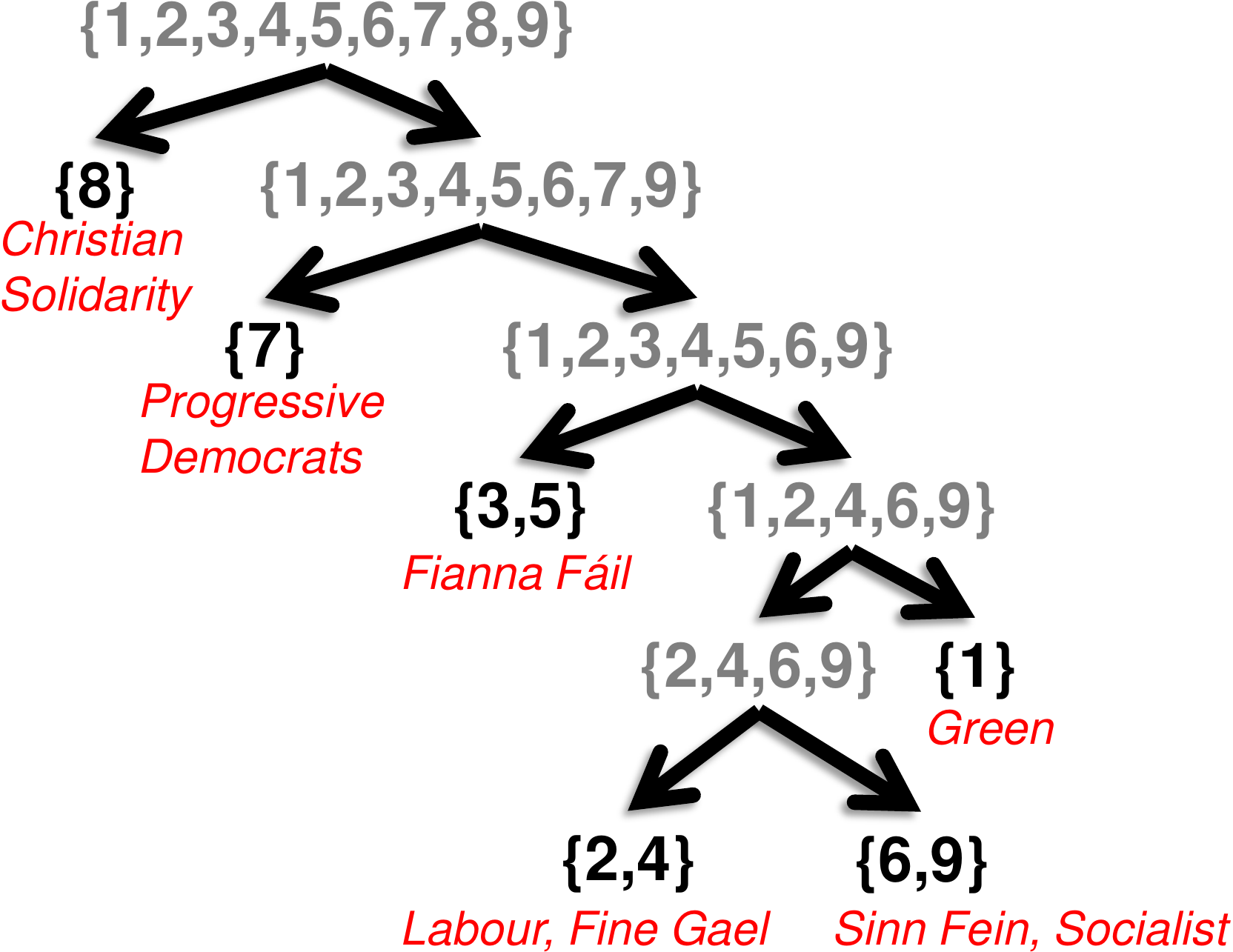}
\label{fig:dublinwest}
}
\caption{Learned hierarchies for Irish election data from Dublin (north)
and Dublin (west) constituencies}
\label{fig:dublintrees}
\end{center}
\end{figure*}

\section{Conclusions}
Exploiting independence structure for efficient inference
and low sample complexity is a simple yet powerful idea,
pervasive throughout the machine learning literature, showing up
in the form of Bayesian networks, Markov random fields, and more.
For rankings, independence can be problematic due to mutual exlusivity
constraints, and we began our paper by indicating a need for a
useful generalization of independence.

The main contribution of our paper is the definition
of such a generalized notion, namely, riffled independence.
There are a number of natural questions that immediately
follow any such definition, such as:
\begin{itemize}\denselist
\item Does the generalization retain any of the computational advantages
of probabilistic independence?
\item Can we find evidence that such
 generalized independence relations hold (or approximately hold)
in real datasets?
\item If subsets of items in a ranking dataset indeed
satisfy the generalized independence assumption,
or approximately so, how could we algorithmically 
determine what these subsets should be from samples?
\end{itemize}
We have shown that for riffled independence, the 
answer to each of the above questions lies in the affirmative.
We next explored hierarchical riffle independent decompositions.
Our model, in which riffle independent subsets are recursively chained
together, leads to a simple, interpretable model whose structure
we can estimate from data, and we have successfully applied our learning
algorithms to several real datasets.

Currently, the success of our structure learning methods
depends on the existence of a 
fairly sizeable dataset of full rankings.
However, ranking datasets are more typically composed
of partial or incomplete rankings, which are often far
easier to elicit from a multitude of users.
For example, top-$k$ type rankings, or even rating data (in
which a user/judge provides a rating of an item between, say, 1 and 5)
are common.
Extending our parameter and structure learning algorithms
for handling such partially ranked data would be
a valuable and practical extension of our work.
For structure learning, our tripletwise mutual information
measures can already potentially be estimated within a
top-$k$ ranking setting.  It would be interesting to also 
develop methods for estimating these mutual information measures
from other forms of partial rankings.
Additionally, the effect of using
partial rankings on structure learning sample complexity
is not yet understood, and the field would benefit 
from a careful analysis.

Many other possible extensions are possible.
In our paper, we have developed algorithms for 
estimating maximum likelihood parameters.  For small
training set sizes, a Bayesian approach would be
more appropriate, where a prior is placed on the parameter
space.  However, if the prior distribution ties parameters
together (i.e., if the prior does not factor across parameters),
then the structure learning problem can be considerably
more complicated, since we would not be able to simply identify
independence relations.

Riffled independence is a new tool for analyzing ranked data
and as we have shown, has the potential to give new insights
into ranking datasets.  We strongly believe that it will be crucial
in developing fast and efficient inference and 
learning procedures for ranking data, and perhaps other
forms of permutation data.

\appendix
\section*{Acknowledgements}
This work is supported in part by the
ONR under MURI N000140710747, and
the Young Investigator Program grant N00014-08-1-0752.
We thank Khalid El-Arini for feedback on initial drafts,
and Brendan Murphy and Claire Gormley for providing
the Irish voting datasets.  Discussions with Marina
Meila provided valuable initial ideas upon which this 
work is based.

\section{More proofs and discussion}\label{app:proofs}
\subsection{Proof of Lemma~\ref{lem:tau}}

\begin{proof}
Let $i$ be an item in $A$ (with $1\leq i\leq p$).  
Since $\phi_A(\sigma)\in S_p$, $[\phi_A(\sigma)](i)$ is 
some number between $1$ and $p$.
By definition, for any $j\in\{1,\dots,p\}$ the interleaving map
$[\tau_{A,B}(\sigma)](j)$ returns the $j^{th}$ largest rank in $\sigma(A)$.
Thus, $[\tau_{A,B}(\sigma)](\phi_A(i))$ is the $\phi_A(i)$-th largest
rank in $\sigma(A)$, which is simply the absolute rank of item $i$.
Therefore, we conclude that $\sigma(i)=[\tau_{A,B}(\sigma)](\phi_A(i))$.
Similarly, if $p+1\leq i\leq n$, we have
$\sigma(i)=[\tau_{A,B}(\sigma)](\phi_B(i)+p)$ (the added $p$ is necessary
since the indices of $B$ are offset by $p$ in $\sigma$),
and we can conclude that
$\sigma=\tau_{A,B}(\sigma) [\phi_A(\sigma)\,\phi_B(\sigma)]$.
\end{proof}

\subsection{Log-likelihood interpretations}\label{sec:likelihood}
If we examine the KL divergence objective introduced in 
Section~\ref{sec:structure1},
it is a standard fact that minimizing Equation~\ref{eqn:mainproblem}
is equivalent to find the structure which maximizes the log-likelihood
of the training data. \vspace{-3mm}

{\footnotesize
\begin{align*}
\mathcal{F}[A,B] &= D_{KL}(\hat{h}(\sigma) \,||\, 
m(\tau_{A,B}(\sigma)) f(\phi_A(\sigma)) g(\phi_B(\sigma))
),  \\
	&= \sum_{\sigma \in S_n} \hat{h}(\sigma) \log\left(\frac{\hat{h}(\sigma)}{m(\tau_{A,B}(\sigma)) f(\phi_A(\sigma)) g(\phi_B(\sigma))}\right), \\
	&= \mbox{const.} - \sum_{\sigma\in S_n} \hat{h}(\sigma) \log\left(m(\tau_{A,B}(\sigma)) f(\phi_A(\sigma)) g(\phi_B(\sigma))\right), \\	
	&= \mbox{const.} - \log\left(\prod_{i=1}^{m} 
	m(\tau_{A,B}(\sigma^{(i)})) f(\phi_A(\sigma^{(i)})) g(\phi_B(\sigma^{(i)})) \right).
\end{align*}\vspace{-3mm}
}

\noindent In the above (with some abuse of notation), 
$m$, $f$ and $g$ are estimated using counts 
from the training data.
The equivalence is significant
because it justifies structure learning for data which is not 
necessarily generated from a distribution which factors into riffle
independent components.   Using a similar manipulation, we can rewrite our 
objective function (Equation~\ref{eqn:micriterion}) as:\vspace{-3mm}

{\footnotesize
\begin{align*}
\mathcal{F}[A,B] &= I(\sigma(A)\;;\;\phi_B(\sigma)) 
        +I(\sigma(B)\;;\;\phi_A(\sigma)), \\
	&= \mbox{const.} - \log\left(\prod_{i=1}^{m} \psi_A(\phi_A(\sigma^{(i)}),\tau_{A,B}(\sigma^{(i)}))g(\phi_B(\sigma^{(i)})) \right) \\
	&\qquad\qquad - \log\left(\prod_{i=1}^{m} f(\phi_A(\sigma^{(i)}))\psi_B(\phi_B(\sigma^{(i)}),\tau_{A,B}(\sigma^{(i)}))\right),
\end{align*}\vspace{-3mm}
}

\noindent which we can see to be a ``composite'' of two likelihood functions.
We are evaluating our data log-likelihood first under a model in which the 
absolute ranks of items in $A$ are independent of relative ranks of items
in $B$, and secondly under a model in which the absolute ranks of items in 
$B$ are independent of relative ranks of items in $A$.
Here again, $\psi_A$ and $\psi_B$ are estimated using counts
of the training data.
We see that if these distributions, $\psi_A$, $\psi_B$ factor
along their inputs, then optimizing the objective function is
equivalent to optimizing the likelihood under the riffle independent model.
Thus, if the data is already riffle independent (or nearly riffle 
independent), then the structure learning objective can 
indeed to be interpreted
as maximizing the log-likelihood of the data, but otherwise there
does not seem to be a clear equivalence between the two objective functions.

\subsection{Why testing for independence of relative ranks is insufficient}
Why can we not just check to see that the relative ranks of $A$ are 
independent of the relative ranks of $B$?  
Another natural objective function for detecting riffle independent
subsets is:
\begin{equation}\label{eqn:insufficient}
\mathcal{F}[A,B] \equiv I(\phi_A(\sigma)\,;\,\phi_B(\sigma)).
\end{equation}
Equation~\ref{eqn:insufficient} is certainly a necessary condition
for subsets $A$ and $B$ to be riffle independent 
but why would it not be sufficient?
It is easy to construct a counterexample ---
simply find a distribution in which the interleaving
depends on either of the relative rankings.

\begin{example}
In this example, we will consider a distribution on $S_4$.
Let $A=\{1,2\}$ and $B=\{3,4\}$.  To generate rankings $\sigma\in S_4$, 
we will draw independent relative rankings, $\sigma_A$ and $\sigma_B$, 
with uniform probability for each of $A$ and $B$. 
Then set the interleaving as follows:
\[
\tau = \left\{\begin{array}{cc}
\llbracket AABB \rrbracket & \mbox{if $\sigma_A=(1,2)$} \\
\llbracket BBAA \rrbracket & \mbox{otherwise} 
\end{array}\right..
\]
Finally set $\sigma = \tau\cdot[\sigma_A,\sigma_B]$.

Since the relative rankings are independent, $\mathcal{F}[A,B]=0$.
But since the interleaving depends on the relative ranking of items
in $A$, we see that $A$ and $B$ are not riffle independent in this example.
\end{example}


\section{Manipulating interleaving distributions in the Fourier domain}\label{sec:fourierinterleavings}

\section{Proof of recurrence}

\begin{theorem}
Algorithm 1 returns a uniformly distributed $(p,q)$-interleaving.
\end{theorem}
\begin{proof}
The proof is by induction on $n=p+q$.  The base case (when $n=1$)
is obvious since the algorithm can only return a single permutation.

Next, we assume for the sake of induction that for any $m<n$, the algorithm
returns a uniformly distributed interleaving and we want to show this
to also be the case for $n$.

Let $\tau$ be any interleaving in $\Omega_{p,q}$.  We will show that 
$m(\tau)=1/{n \choose p}$.
Consider $\tau^-=\tau^{-1}(1:n-1)$.  There are two cases: $\tau^{-1}$
is either a $(p,q-1)$-interleaving (in which case $\tau(n)=n$),
or a $(p-1,q)$-interleaving (in which case $\tau(p)=n$).

We will just consider the first case since the second is similar.
$\tau^{-1}$ is uniformly distributed by the inductive hypothesis
and therefore has probability $1/{n-1\choose p}$.

$\tau^{-1}(n)$ is set to $n$ independently with probability $q/n$, so we compute the 
probability of the interleaving resulting from the algorithm as:
\[
\frac{n-p}{n} \cdot \frac{1}{{n-1 \choose p}} = \frac{n-p}{n} \cdot \frac{p!(n-1-p)!}{(n-1)!}
	=\frac{p!(n-p)!}{n!} = \frac{1}{{n \choose p}}.
\]

\end{proof}

\paragraph{Fourier transforming the biased riffle shuffle}
We describe the recurrence satisfied by $m_{p,q}^{unif}$,
allowing one to write $m_{p,q}^{unif}$, a distribution on $S_n$,
in terms of $m^{unif}_{p,q-1}$ and $m^{unif}_{p-1,q}$,
distributions over $S_{n-1}$ (see Algorithm $1$ in main paper).
Given a function $f:S_{n-1}\to\reals$, we will define the
\emph{embedded} function $f\uparrow_{n-1}^{n}:S_n\to\reals$
by $f\uparrow_{n-1}^{n}(\sigma)=f(\sigma_1,\dots,\sigma_{n-1})$ if
$\sigma(n)=n$, and $0$ otherwise.
Algorithm~1 can be then rephrased as a 
recurrence relation as follows.
\begin{proposition}\label{prop:recurrence}
The uniform riffle shuffling distribution $m_{p,q}^{unif}$
obeys the recurrence relation:\vspace{-3mm}

{\footnotesize
\begin{equation}\label{eqn:recurrence}
m^{unif}_{p,q} = \left[\left(\frac{p}{p+q}\right)\cdot m^{unif}_{p-1,q}\uparrow_{n-1}^n
*\delta_{(p+1,\dots,n)}\right]+
\left[\left(\frac{q}{p+q} \right)\cdot m^{unif}_{p,q-1}\uparrow_{n-1}^n\right],
\end{equation}\vspace{-3mm}
}

\noindent with base cases:
$m^{unif}_{0,n} = m^{unif}_{n,0} = \delta_{\epsilon}$,
where $\delta_\epsilon$ is the delta function at the identity permutation.
\vspace{-2mm}
\end{proposition}

Note that by taking the support sizes of each of the functions in the
above recurrence, we recover the following well known recurrence
for binomial coefficients:\vspace{-3mm}

{\footnotesize
\begin{equation}\label{eqn:binomial}
\binom{n}{p} = \binom{n-1}{p-1} + \binom{n-1}{p},
\;\mbox{with base case}\;
\binom{n}{0}=\binom{n}{n} = 1.
\end{equation}\vspace{-3mm}
}

\noindent The \emph{biased riffle shuffle} is defined by:\vspace{-3mm}

{\footnotesize
\begin{equation}\label{eqn:biasedriffle}
m^\alpha_{p,q} \propto 
\left[\left(\frac{\alpha p}{p+q}\right)\cdot m^{unif}_{p-1,q}\uparrow_{n-1}^n
*\delta_{(p+1,\dots,n)}\right]+
\left[\left(\frac{(1-\alpha)q}{p+q} \right)\cdot m^{unif}_{p,q-1}\uparrow_{n-1}^n\right]
\end{equation}
}

\begin{algorithm2e}[t!]
{\scriptsize
{\sc RiffleHat($p,q$)} \xspace \\
\ $n\leftarrow p+q$ \;
\ Initialize $\widehat{m}^{prev},\widehat{m}^{curr}$
as arrays of $p+1$ Fourier transform data structures \;
\ \For{$i=1,2,...,n$}{
        \ \For{$j=\max(0,p-n+i),\dots,\min(i,p)$}{
                \ \If{$j==0$ or $j==i$}{
                        $\widehat{m}^{curr}[j] \leftarrow
                                \widehat{\delta}_{\epsilon\in S_i}$ \;
                }\Else{
                        $\widehat{m}^{curr}[j] \leftarrow
                                \left(\frac{i-j}{i}\right)\mbox{\sc Embed}
                                        (\widehat{m}^{prev}[j],i-1,i)$ \\
                        \qquad\qquad\qquad
                $ +\left(\frac{i}{j}\right)\mbox{\sc Convolve}(\mbox{\sc Embed}
                                        (\widehat{m}^{prev}[j-1],i-1,i),
                                        \widehat{\delta}_{(i,i-1,\dots,j)})$\;
                }
                $\widehat{m}^{prev}\leftarrow\widehat{m}^{curr}$ \;
        }
    }
\ \Return $\widehat{m}^{curr}[p]$\;
\caption{\scriptsize Pseudocode for computing the Fourier
transform of the uniform riffle shuffling distribution using
dynamic programming.}
\label{alg:rifflehat}
}
\end{algorithm2e}

\begin{figure*}[t!]
\begin{center}
\includegraphics[width=.65\textwidth]{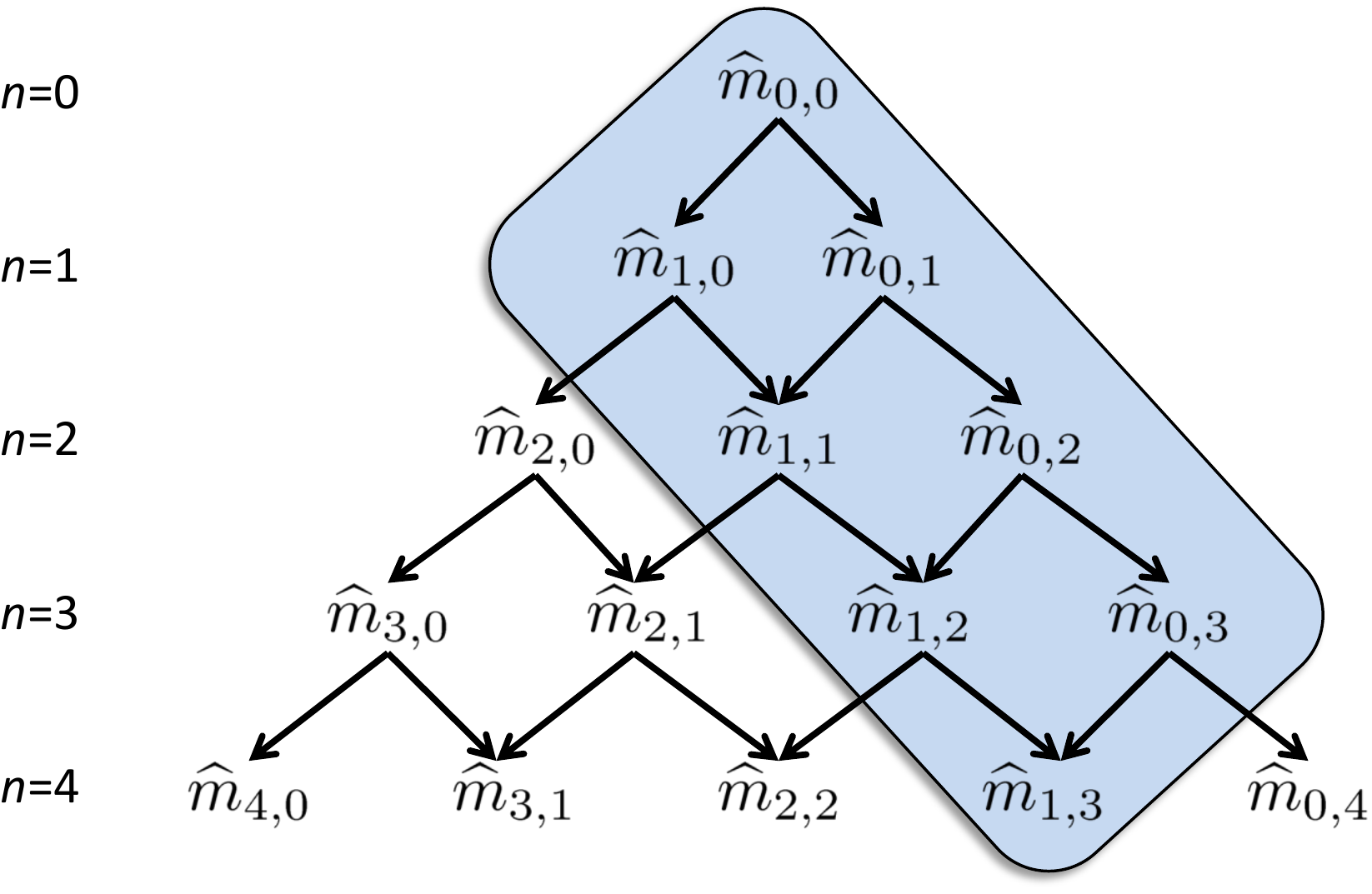}
\caption{The flow of information in Algorithm~\ref{alg:rifflehat} bears
much resemblance to Pascal's triangle for computing binomial coefficients.
The arrows in this diagram indicate the Fourier transforms that must be precomputed
before computing the Fourier transform of a larger interleaving distribution.
For example, to compute $\widehat{m}_{1,2}$, one must first compute 
$\widehat{m}_{1,1}$ and $\widehat{m}_{0,2}$.  In blue, we have outlined
the collection of Fourier transforms that are computed by Algorithm
~\ref{alg:rifflehat} while computing $\widehat{m}_{1,3}$}
\label{fig:pascal}
\end{center}
\end{figure*}

Writing the recursion in the form of Equation~\ref{eqn:recurrence}
provides a construction of the uniform riffle shuffle as a
sequence of operations on smaller distribution which can
be performed completely with respect to Fourier coefficients.
In particular, given the Fourier coefficients of a
function $f:S_{n-1}\to\reals$, one can construct the Fourier coefficients
of the embedding $f\uparrow_{n-1}^n$ by applying the branching rule
(see~\cite{sagan01,kondor08} for details).
Using the linearity property, the Convolution Theorem~\ref{prop:fourierprops}
and the fact that embeddings can be performed in the Fourier domain,
we arrive at the equivalent Fourier-theoretic recurrence for each
frequency level $i$.\vspace{-3mm}

{\footnotesize
\begin{equation}\label{eqn:fourierrecurrence}
\left[\widehat{m^{unif}_{p,q}}\right]_i = 
        \left(\frac{p}{p+q}\right)\cdot
        \left[\widehat{m^{unif}_{p-1,q}}\uparrow_{n-1}^n\right]_i
                \cdot \rho_i(p+1,\dots,n)
        +\left(\frac{q}{p+q}\right)\cdot
        \left[\widehat{m^{unif}_{p,q-1}}\right]_i
\end{equation}
}

where $\rho_i$ is the $i^{th}$ irreducible representation matrix
evaluated at the cycle $(p+1,\dots,n)$ (see~\cite{huangetal09b}
for details on irreducible representations).
Implementing the recurrence (Equation~\ref{eqn:fourierrecurrence})
in code can naively result in
an exponential time algorithm if one is not careful.
It is necessary to use dynamic programming to
be sure not to recompute things that were already computed.
In Algorithm~\ref{alg:rifflehat}, we present pseudocode
of such a dynamic programming approach, which builds a
`Pascal's triangle' similar to that which might be constructed
to compute a table of binomial coefficients.  The pseudocode
assumes the existence of Fourier domain algorithms
for convolving distributions and for embedding a distribution
over $S_{n-1}$ into $S_n$.  See Figure~\ref{fig:pascal} 
for a graphical illustration of the algorithm.


\subsection{Sample complexity analysis}
\begin{lemma}[adapted from~\cite{hoffgen93}]\label{lem:hoffgen_app}
The entropy of a discrete random variable with arity $R$
can be estimated to within accuracy $\Delta$ with probability
$1-\beta$ using
$O\left(\frac{R^2}{\Delta^2}\log^2\frac{R}{\Delta}\log\frac{R}{\beta}\right)$
i.i.d samples and the same time.
\end{lemma}
\begin{lemma}\label{lem:allmi}
The collection of mutual informations
$I_{i;j,k}$ can be
estimated to within accuracy $\Delta$ for all triplets $(i,j,k)$ with probability at least
$1-\gamma$ using
$S(\Delta,\gamma)\equiv 
O\left(\frac{n^2}{\Delta^2}\log^2\frac{n}{\Delta}\log\frac{n^4}{\gamma}\right)$
i.i.d. samples and the same amount of time.
\end{lemma}
\begin{proof}
Fix a $0<\gamma\leq 1$ and $\Delta$.
For any fixed triplet $(i,j,k)$, Hoffgen's result (Lemma~\ref{lem:hoffgen_app})
implies that $H(\sigma_i;\sigma_j<\sigma_k)$ can be estimated with accuracy $\Delta$
with probability at least $1-\gamma/n^3$ using
$O\left(\frac{n^2}{\Delta^2}\log^2\frac{n}{\Delta}\log\frac{n^4}{\gamma}\right)$ i.i.d. samples
since the variable $(\sigma_i,\sigma_j<\sigma_k)$ has arity $2n$ and setting
$\beta\equiv\frac{\gamma}{n^3}$.

Estimating the mutual information for the same triplet therefore
requires the same sample complexity by the expansion:
$I_{i;j,k}=H(\sigma_i)+H(\sigma_j<\sigma_k)-H(\sigma_i;\sigma_j<\sigma_k)$.
Now we use a simple union bound to bound the probability that the collection of mutual
informations over all triplets is estimated to within $\Delta$ accuracy.
Define $\Delta_{i,j,k}\equiv I_{i;j,k}-\hat{I}_{i;j,k}$.\vspace{-3mm}

{\footnotesize
\[
P(|\Delta_{i,j,k}|<\Delta,\;\forall (i,j,k)) \geq 1-\sum_{i,j,k} P(|\Delta_{i,j,k}|\geq\Delta) 
        \geq 1-n^3\cdot \frac{\gamma}{n^3}
        \geq 1-\gamma.\vspace{-3mm}
\]
}
\end{proof}
\begin{lemma}\label{lem:edges}
Fix $k\leq n/2$.
and let $A$ be a $k$-subset of $\{1,\dots,n\}$ with $A$ riffle independent of its complement $B$.
Let $A'$ be a $k$-subset with $A'\neq A$ or $B$.
If $A$ and $B$ are each $\epsilon$-third order strongly connected,
we have $\FF(A')=\FF(B')> \psi(n,k)\cdot \epsilon$,
where $\psi(n,k)\equiv (n-k)(n-2k)$.
\end{lemma}
\renewcommand{\Xedge}{\Omega^{cross}_{A',B'}}
\renewcommand{\ABedge}{\Omega^{cross}_{A,B}}
\renewcommand{\Aedge}{\Omega^{int}_{A}}
\renewcommand{\Bedge}{\Omega^{int}_{B}}
\begin{proof}
Let us first establish some notation.
Given a subset $X\subset\{1,\dots,n\}$, define
\[
\Omega^{int}_X \equiv \{(x;y,z)\,:\,x,y,z\in X\}.
\]
Thus $\Aedge$ and $\Bedge$ are the sets of triplets whose indices are all internal to $A$
or internal to $B$ respectively.
We define $\Xedge$ to be the set of triplets which ``cross'' between the sets $A$ and $B$:
\[
\Xedge \equiv \{(x;y,z)\,:\,x\in A, y,z\in B,\,\mbox{or}\,x\in B, y,z\in A \}.
\]

The goal of this proof is to use the strong connectivity assumptions
to lower bound $\FF(A')$.  In particular, due to strong connectivity,
each triplet inside $\Xedge$ that also lies in either $\Aedge$ or $\Bedge$
must contribute at least $\epsilon$ to the objective function $\FF(A')$.
It therefore suffices to lower bound the number of triplets
which cross between $A'$ and $B'$, but are internal to either $A$ or $B$
(i.e., $|\Xedge\cap(\Aedge\cup\Bedge)|$).  Define $\ell\equiv |A\cap A'|$
and note that $0\leq \ell<k$.
It is straightforward to check that:
$|A\cap B'|=k-\ell$, $|B\cap A'| = k-\ell$, and $|B\cap B'|=(n-k)-(k-\ell)=n+\ell-2k$.\vspace{-3mm}

{\footnotesize
\begin{align*}
|\Xedge\cap(\Aedge\cup\Bedge)| &= |\Xedge\cap\Aedge|+|\Xedge\cap\Bedge|, \\
        &\geq \ell (k-\ell)^2+\ell^2(k-\ell)
                + (k-\ell)(n+\ell-2k)^2 
                +(n+\ell-2k)(k-\ell)^2, \\
        &\geq (k-\ell)\left((n-k)(n-2k)+\ell n\right), \\
        &\geq k\left((n-k)(n-2k)+kn\right).
\end{align*}\vspace{-3mm}
}

We do want the bound above to depend on $\ell$.
Intuitively, for a fixed $k$ and $n$, the above expression
is minimized when either $\ell=0$ or $k-1$ (a more formal argument
is shown below in the proof of Lemma~\ref{lem:claim}).
Plugging $\ell=0$ and $k-1$ and bounding from below yields:\vspace{-3mm}

{\footnotesize
\begin{align*}
|\Xedge\cap(\Aedge\cup\Bedge)|  &\geq \min\left(k(n-k)(n-2k),
			(n-k)(n-2k)+n(k-1)]\right),\\
        &\geq (n-k)(n-2k).
\end{align*}\vspace{-3mm}
}

\noindent Finally due to strong connectivity, we know that for each triplet in $\Aedge\cup\Bedge$,
we have $I_{x;y,z}>\epsilon$, thus each edge in
$\Xedge\cap(\Aedge\cup\Bedge)$ contributes at least $\epsilon$ to $\FF(A')$, establishing
the desired result.
\end{proof}
\begin{lemma}\label{lem:claim}
Under the same assumptions as
Lemma~\ref{lem:edges}, $p(n,k,\ell)=(k-\ell)\left((n-k)(n-2k)+\ell n\right)$
is minimized at either $\ell=0$ or $k-1$.
\end{lemma}
\begin{proof}
Let $\alpha=(n-k)(n-2k)$.  We know that $\alpha\geq 0$ since $k\leq n/2$ by
assumption (and equals zero only when $k=n/2$).
We want to find the $\ell\in\{0,\dots,k-1\}$ which minimizes
the concave quadratic function $p(\ell)=(k-\ell)(\alpha+\ell n)$,
the roots of which are $\ell=k$ and $\ell=-\alpha/n$ (note that $-\alpha/n\leq 0$.
The minimizer is thus the element of $\{0,\dots,k-1\}$ which is closest
to either of the roots.
\end{proof}

\begin{theorem}
Let $A$ be a $k$-subset of $\{1,\dots,n\}$ with $A$ riffle independent of its complement $B$.
If $A$ and $B$ are each $\epsilon$-third order strongly connected,
then given
$S(\Delta,\epsilon)\equiv 
O\left(\frac{n^4}{\epsilon^2}\log^2\frac{n^2}{\epsilon}\log\frac{n^4}{\gamma}\right)$
i.i.d. samples,
the minimum of $\FFF$ (evaluated over all $k$-subsets of $\{1,\dots,n\}$)
is achieved at exactly the subsets $A$ and $B$
with probability at least $1-\gamma$.
\end{theorem}
\begin{proof}
Let $A'$ be a $k$-subset with $A'\neq A$ or $B$.
Our goal is to show that $\FFF(A')>\FFF(A)$.

Denote the error between estimated mutual information
and true mutual information by $\Delta_{i;j,k}\equiv \hat{I}_{i;j,k}-I_{i;j,k}$.
We have:\vspace{-3mm}

{\footnotesize
\begin{align*}
\FFF(A')-\FFF(A) &=  \left(\sum_{(i,j,k)\in \Xedge} \hat{I}_{i;j,k}\right) 
                -\left(\sum_{(i,j,k)\in \ABedge} \hat{I}_{i;j,k}\right) \\
        &= \FF(A') - \FF(A) + \sum_{(i,j,k)\in \Xedge} \Delta_{i;j,k} 
			-\sum_{(i,j,k)\in\ABedge} \Delta_{i;j,k} \\
        &\geq \psi(n,k)\cdot\epsilon + \sum_{(i,j,k)\in \Xedge} \Delta_{i;j,k} 
		-\sum_{(i,j,k)\in\ABedge} \Delta_{i;j,k} \\
        &\qquad\mbox{(by Lemma~\ref{lem:edges} and $\FF(A)=0$)} \\
\end{align*}\vspace{-3mm}
}

\noindent Now assume that all of the estimation errors $\Delta$ are uniformly bounded as:\vspace{-3mm}

{\footnotesize
\begin{equation}\label{eqn:errorbound}
|\Delta_{i;j,k}|\leq \frac{\epsilon}{4}\left(\frac{\psi(n,k)}{n^2k-k^2n}\right).
\end{equation}\vspace{-3mm}
}

\noindent And note that $|\Xedge|=|\ABedge|=k^2(n-k)+k(n-k)^2=n^2k-k^2n$.
We have:\vspace{-3mm}

{\footnotesize
\begin{align*}
 \sum_{(i,j,k)\in \Xedge} |\Delta_{i;j,k}| -\sum_{(i,j,k)\in\ABedge} |\Delta_{i;j,k} |
        &\leq 2\cdot (n^2k-k^2n)
	\cdot\frac{\epsilon}{4}\left(\frac{\psi(n,k)}{n^2k-k^2n}\right) \\
        & \leq \frac{\epsilon\psi(n,k)}{2} \\
        &\leq \epsilon\cdot\psi(n,k)
\end{align*}\vspace{-3mm}
}

Combining this bound on the estimation errors with the bound on $\FFF(A')-\FFF(A)$
yields:\vspace{-3mm}

{\footnotesize
\begin{align*}
\FFF(A')-\FFF(A) &\geq \epsilon\psi(n,k) 
	-\left(\sum_{(i,j,k)\in \Xedge} |\Delta_{i;j,k}| -\sum_{(i,j,k)\in\ABedge} |\Delta_{i;j,k} |\right) \\
        &\geq \frac{\epsilon\psi(n,k)}{2} \\
        & > 0,
\end{align*}\vspace{-3mm}
}

\noindent which is almost what we want to show.
How many samples do we require to achieve the
bound assumed in Equation~\ref{eqn:errorbound} with high probability?
Observe that the bound simplifies as,\vspace{-3mm}

{\footnotesize
\[
\frac{\epsilon}{4}\left(\frac{\psi(n,k)}{n^2k-k^2n}\right)
        = \frac{\epsilon}{4}\left( \frac{(n-k)(n-2k)}{nk(n-k)}\right)       
        = \frac{\epsilon}{4}\left(\frac{n-2k}{nk}\right),\vspace{-3mm}
\]
}

\noindent which behaves like $O\left(\epsilon\right)$ when $k$ is $O(1)$, but like
$O\left(\frac{\epsilon}{n}\right)$ when $k$ is $O(n)$.
Applying the sample complexity result of Lemma~\ref{lem:allmi} with $\Delta=O(\epsilon/n)$,
we see that given
$O\left(\frac{n^4}{\epsilon^2}\log^2\frac{n^2}{\epsilon}\log\frac{n^4}{\gamma}\right)$
i.i.d. samples, the bound in Equation~\ref{eqn:errorbound} holds with
probability $1-\gamma$, concluding the proof.
\end{proof}

\bibliographystyle{imsart-nameyear}
\bibliography{references}
\end{document}